\pdfoutput=1
\documentclass[letterpaper]{article}
\usepackage{ijcai17}
\usepackage[pass]{geometry}
\usepackage{times}
\usepackage{latexsym}
\usepackage{amsmath}
\usepackage{amssymb}
\usepackage{amsthm}
\usepackage{paralist}
\usepackage{thmtools}
\usepackage{thm-restate}
\usepackage{algorithm}
\usepackage[noend]{algpseudocode}
\usepackage{macros}
\usepackage{chngcntr}
\usepackage[draft]{ifdraft}
\allowdisplaybreaks
    
\newtheorem{theorem}{Theorem}
\newtheorem{proposition}[theorem]{Proposition}

\newtheorem{lemma}[theorem]{Lemma}
\newtheorem{definition}[theorem]{Definition}

\newtheorem{example}[theorem]{Example}
 
\title{
    Foundations of Declarative Data Analysis Using Limit Datalog Programs
}
\author{
    Mark Kaminski,~~Bernardo Cuenca Grau,~~Egor V. Kostylev,~~Boris Motik\and Ian Horrocks\\
    Department of Computer Science, University of Oxford, UK\\
    \{mark.kaminski,~bernardo.cuenca.grau,~egor.kostylev,~boris.motik,~ian.horrocks\}@cs.ox.ac.uk
}
 
\begin{document}

\maketitle

\begin{abstract}
Motivated by applications in declarative data analysis, we study $\DLog$---an
extension of positive Datalog with arithmetic functions over integers. This
language is known to be undecidable, so we propose two fragments. In
\emph{limit $\DLog$} predicates are axiomatised to keep minimal/maximal numeric
values, allowing us to show that fact entailment is
$\textsc{coNExpTime}$-complete in combined, and $\textsc{coNP}$-complete in
data complexity. Moreover, an additional \emph{stability} requirement causes
the complexity to drop to $\textsc{ExpTime}$ and $\textsc{PTime}$,
respectively. Finally, we show that stable $\DLog$ can express many useful data
analysis tasks, and so our results provide a sound foundation for the
development of advanced information systems.

\end{abstract}

\section{Introduction}\label{sec:introduction}

Analysing complex datasets is currently a hot topic in information systems. The
term `data analysis' covers a broad range of techniques that often involve
tasks such as data aggregation, property verification, or query answering. Such
tasks are currently often solved imperatively (e.g., using Java or Scala) by
specifying \emph{how} to manipulate the data, and this is undesirable because
the objective of the analysis is often obscured by evaluation concerns. It has
recently been argued that data analysis should be \emph{declarative}
\cite{DBLP:conf/eurosys/AlvaroCCEHS10,DBLP:journals/pvldb/Markl14,DBLP:journals/tkde/SeoGL15,DBLP:conf/sigmod/ShkapskyYICCZ16}:
users should describe \emph{what} the desired output is, rather than how to
compute it. For example, instead of computing shortest paths in a graph by a
concrete algorithm, one should (i)~describe what a path length is, and
(ii)~select only paths of minimum length. Such a specification is independent
of evaluation details, allowing analysts to focus on the task at hand. An
evaluation strategy can be chosen later, and general parallel and/or
incremental evaluation algorithms can be reused `for free'.

An essential ingredient of declarative data analysis is an efficient language
that can capture the relevant tasks, and Datalog is a prime candidate since it
supports recursion. Apart from recursion, however, data analysis usually also
requires integer arithmetic to capture quantitative aspects of data (e.g., the
length of a shortest path). Research on combining the two dates back to the '90s
\cite{DBLP:conf/vldb/MumickPR90,DBLP:conf/slp/KempS91,DBLP:journals/jlp/BeeriNST91,DBLP:conf/pods/Gelder92,DBLP:journals/tcs/ConsensM93,DBLP:journals/jcss/GangulyGZ95,RossS97},
and is currently experiencing a revival
\cite{DBLP:journals/ai/FaberPL11,DBLP:journals/vldb/MazuranSZ13}. This
extensive body of work, however, focuses primarily on integrating recursion and
arithmetic with \emph{aggregate functions} in a coherent semantic framework,
where technical difficulties arise due to nonmonotonicity of aggregates.
Surprisingly little is known about the computational properties of integrating
recursion with arithmetic, apart from that a straightforward combination is
undecidable \cite{DBLP:journals/csur/DantsinEGV01}. Undecidability also carries
over to the above formalisms and practical Datalog-based systems such as BOOM
\cite{DBLP:conf/eurosys/AlvaroCCEHS10}, DeALS
\cite{DBLP:conf/sigmod/ShkapskyYICCZ16}, Myria
\cite{DBLP:journals/pvldb/WangBH15}, SociaLite
\cite{DBLP:journals/tkde/SeoGL15}, Overlog
\cite{DBLP:journals/cacm/LooCGGHMRRS09}, Dyna
\cite{DBLP:conf/datalog/EisnerF10}, and Yedalog
\cite{DBLP:conf/snapl/ChinDEHMOOP15}.

To develop a sound foundation for Datalog-based declarative data analysis, we
study $\DLog$---negation-free Datalog with integer arithmetic and comparisons.
Our main contribution is a new \emph{limit $\DLog$} fragment that, like the
existing data analysis languages, is powerful and flexible enough to naturally
capture many important analysis tasks. However, unlike $\DLog$ and the existing
languages, reasoning with limit programs is decidable, and it becomes tractable
in data complexity under an additional \emph{stability} restriction.

In limit $\DLog$, all intensional predicates with a numeric argument are
\emph{limit predicates}. Instead of keeping all numeric values for a given
tuple of objects, such predicates keep only the minimal ($\tmin$) or only the
maximal ($\tmax$) bounds of numeric values entailed for the tuple. For example,
if we encode a weighted directed graph using a ternary predicate
$\mathit{edge}$, then rules \eqref{eq:sp1} and \eqref{eq:sp2}, where
$\mathit{sp}$ is a $\tmin$ limit predicate, compute the cost of a shortest path
from a given source node $v_0$ to every other node.
\begin{align}
                                                & \rightarrow \mathit{sp}(v_0,0)    \label{eq:sp1} \\
    \mathit{sp}(x,m) \land \mathit{edge}(x,y,n) & \to \mathit{sp}(y,m + n)          \label{eq:sp2}
\end{align}
If these rules and a dataset entail a fact $\mathit{sp}(v,k)$,
then the cost of a shortest path from $v_0$ to $v$ is at most $k$; hence,
$\mathit{sp}(v,k')$  holds for each ${k' \geq k}$ since the cost of
a shortest path is also at most $k'$. Rule \eqref{eq:sp2} intuitively says
that, if $x$ is reachable from $v_0$ with cost at most $m$ and ${\langle x,y
\rangle}$ is an edge of cost $n$, then $v'$ is reachable from $v_0$ with cost
at most ${m + n}$. This is different from $\DLog$, where there is no implicit
semantic connection between $\mathit{sp}(v,k)$ and $\mathit{sp}(v,k')$, and
such semantic connections allow us to prove decidability of limit $\DLog$. We
provide a direct semantics for limit predicates based on Herbrand
interpretations, but we also show that this semantics can be axiomatised in
standard $\DLog$. Our formalism can thus be seen as a fragment of $\DLog$, from
which it inherits well-understood properties such as monotonicity and existence
of a least fixpoint model \cite{DBLP:journals/csur/DantsinEGV01}.

Our contributions are as follows. First, we introduce limit $\DLog$ programs
and argue that they can naturally capture many relevant data analysis tasks. We
prove that fact entailment in limit $\DLog$ is undecidable, but, after
restricting the use of multiplication, it becomes \textsc{coNExpTime}- and
\textsc{coNP}-complete in combined and data complexity, respectively. To
achieve tractability in data complexity (which is very important for robust
behaviour on large datasets), we additionally introduce a \emph{stability}
restriction and show that this does not prevent expressing the relevant
analysis tasks.

The proofs of all results are given in \ifdraft{the appendix of this paper}{an
extended version of this paper \cite{extended-version}}.

%%% Local Variables:
%%% mode: latex
%%% TeX-master: "paper"
%%% End:

\section{Preliminaries}\label{sec:preliminaries}

In this section, we recapitulate the well-known definitions of Datalog with
integers, which we call $\DLog$.

\myparagraph{Syntax}
A vocabulary consists of \emph{predicates}, \emph{objects}, \emph{object
variables}, and \emph{numeric variables}. Each predicate has an integer
\emph{arity} $n$, and each position ${1 \leq i \leq n}$ is of either
\emph{object} or \emph{numeric sort}. An \emph{object term} is an object or an
object variable. A \emph{numeric term} is an integer, a numeric variable, or of
the form ${s_1 + s_2}$, ${s_1 - s_2}$, or ${s_1 \times s_2}$ where $s_1$ and
$s_2$ are numeric terms, and $+$, $-$, and $\times$ are the standard
\emph{arithmetic functions}. A \emph{constant} is an object or an integer. The
\emph{magnitude} of an integer is its absolute value. A \emph{standard atom} is
of the form ${B(t_1, \dots, t_n)}$, where $B$ is a predicate of arity $n$ and
each $t_i$ is a term whose type matches the sort of position $i$ of $B$. A
\emph{comparison atom} is of the form ${(s_1 < s_2)}$ or ${(s_1 \leq s_2)}$,
where $<$ and $\leq$ are the standard \emph{comparison predicates}, and $s_1$
and $s_2$ are numeric terms. A \emph{rule} $r$ is of the form
${\bigwedge\nolimits_i \alpha_i \wedge \bigwedge\nolimits_j \beta_j \to
\alpha}$, where $\alpha_{(i)}$ are standard atoms, $\beta_j$ are comparison
atoms, and each variable in $r$ occurs in some $\alpha_i$. Atom ${\head{r} =
\alpha}$ is the \emph{head} of $r$; ${\sbody{r} = \bigwedge\nolimits_i
\alpha_i}$ is the \emph{standard body} of $r$; ${\cbody{r} =
\bigwedge\nolimits_j \beta_j}$ is the \emph{comparison body} of $r$; and
${\body{r} = \sbody{r} \wedge \cbody{r}}$ is the \emph{body} of $r$. A
\emph{ground instance of} $r$ is obtained from $r$ by substituting all
variables by constants. A ($\DLog$) \emph{program} $\Prog$ is a finite set of
rules. Predicate $B$ is \emph{intensional} (\emph{IDB}) in $\Prog$ if $B$
occurs in $\Prog$ in the head of a rule whose body is not empty; otherwise, $B$
is \emph{extensional} (\emph{EDB}) in $\Prog$. A term, atom, rule, or program
is \emph{ground} if it contains no variables. A \emph{fact} is a ground,
function-free, standard atom. Program $\Prog$ is a \emph{dataset} if $\head{r}$
is a fact and ${\body{r} = \emptyset}$ for each ${r \in \Prog}$. We often say
that $\Prog$ contains a fact $\alpha$ and write ${\alpha \in \Prog}$, which
actually means ${\to \alpha \in \Prog}$. We write a tuple of terms as ${\mathbf
t}$, and we often treat conjunctions and tuples as sets and write, say,
${\alpha_i \in \sbody{r}}$, ${|\mathbf t|}$, and ${t_i \in \mathbf t}$.

\myparagraph{Semantics}  
A \emph{(Herbrand) interpretation} $I$ is a (not necessarily finite) set of
facts. Such $I$ \emph{satisfies} a ground atom $\alpha$, written ${I \models
\alpha}$, if (i)~$\alpha$ is a standard atom and evaluating the arithmetic
functions in $\alpha$ produces a fact in $I$, or (ii)~$\alpha$ is a comparison
atom and evaluating the arithmetic functions and comparisons produces $\true$.
The notion of satisfaction is extended to conjunctions of ground atoms, rules,
and programs as in first-order logic, where each rule is universally
quantified. If ${I \models \Prog}$, then $I$ is a \emph{model} of program
$\Prog$; and $\Prog$ \emph{entails} a fact $\alpha$, written ${\Prog \models
\alpha}$, if ${I \models \alpha}$ holds whenever ${I \models \Prog}$.

\myparagraph{Complexity}
In this paper we study the computational properties of checking ${\Prog \models
\alpha}$. \emph{Combined complexity} assumes that both $\Prog$ and $\alpha$ are
part of the input. In contrast, \emph{data complexity} assumes that $\Prog$ is
given as ${\Prog' \cup \Dat}$ for $\Prog'$ a program and $\Dat$ a dataset, and
that only $\Dat$ and $\alpha$ are part of the input while $\Prog'$ is fixed.
Unless otherwise stated, all numbers in the input are coded in binary, and the
\emph{size} $\ssize{\Prog}$ of $\Prog$ is the size of its representation.
Checking ${\Prog \models \alpha}$ is undecidable even if the only arithmetic
function in $\Prog$ is $+$ \cite{DBLP:journals/csur/DantsinEGV01}.

\myparagraph{Presburger arithmetic}
is first-order logic with constants $0$ and $1$, functions $+$ and $-$,
equality, and the comparison predicates $<$ and $\le$, interpreted over all
integers $\mathbb{Z}$. The complexity of checking sentence validity (i.e.,
whether the sentence is true in all models of Presburger arithmetic) is known
when the number of quantifier alternations and/or the number of variables in
each quantifier block are fixed
\cite{DBLP:journals/tcs/Berman80,DBLP:journals/tcs/Gradel88,DBLP:journals/mst/Schoning97,DBLP:conf/csl/Haase14}.

%%% Local Variables:
%%% mode: latex
%%% TeX-master: "paper"
%%% End:

\section{Limit Programs}\label{sec:limit-programs}

Towards introducing a decidable fragment of $\DLog$ for data analysis, we first
note that the undecidability proof of (plain) $\DLog$ outlined by
\citeA{DBLP:journals/csur/DantsinEGV01} uses atoms with at least two numeric
terms. Thus, to motivate introducing our fragment, we first prove that
undecidability holds even if atoms contain at most one numeric term. The proof
uses a reduction from the halting problem for deterministic Turing machines. To
ensure that each standard atom in $\Prog$ has at most one numeric term,
combinations of a time point and a tape position are encoded using a single
integer.

\begin{restatable}{theorem}{factentundecidable}\label{thm:fact-ent-undecidable}
    For $\Prog$ a $\DLog$ program and $\alpha$ a fact, checking ${\Prog \models
    \alpha}$ is undecidable even if\/ $\Prog$ contains no $\times$ or $-$ and
    each standard atom in $\Prog$ has at most one numeric term.
\end{restatable}

We next introduce \emph{limit $\DLog$}, where \emph{limit predicates} keep
bounds on numeric values. This language can be seen as either a semantic or a
syntactic restriction of $\DLog$.

\begin{definition}\label{def:limit-programs}
    In \emph{limit $\DLog$}, a predicate is either an \emph{object predicate}
    with no numeric positions, or a \emph{numeric predicate} where only the
    last position is numeric. A numeric predicate is either an \emph{ordinary
    numeric predicate} or a \emph{limit predicate}, and the latter is either a
    $\tmin$ or a $\tmax$ predicate. Atoms with object predicates are
    \emph{object atoms}, and analogously for other types of atoms/predicates. A
    $\DLog$ rule $r$ is a \emph{limit ($\DLog$) rule} if (i)~${\body{r} =
    \emptyset}$, or (ii)~each atom in $\sbody{r}$ is an object, ordinary
    numeric, or limit atom, and $\head{r}$ is an object or a limit atom. A
    \emph{limit ($\DLog$) program} $\Prog$ is a program containing only limit
    rules; and $\Prog$ is \emph{homogeneous} if it does not contain both
    $\tmin$ and $\tmax$ predicates.
\end{definition}

In the rest of this paper we make three simplifying assumptions. First, numeric
atoms occurring in a rule body are function-free (but comparison atoms and the
head can contain arithmetic functions). Second, each numeric variable in a rule
occurs in at most one standard body atom. Third, distinct rules in a program
use different variables. The third assumption is clearly w.l.o.g.\ because all
variables are universally quantified so their names are immaterial. Moreover,
the first two assumptions are w.l.o.g.\ as well since, for each rule, there
exists a logically equivalent rule that satisfies these assumptions. In
particular, we can replace an atom such as ${A(\mathbf t,m_1 + m_2)}$ with
conjunction
\begin{displaymath}
\begin{array}{@{}l@{}}
    A(\mathbf t,m) \wedge I(m_1) \wedge I(m_2) \; \wedge \\
    \qquad\qquad\qquad (m \leq m_1 + m_2) \wedge (m_1 + m_2 \leq m) \\
\end{array}
\end{displaymath}
where $m$ is a fresh variable and $I$ is a fresh predicate axiomatised to hold
on all integers as follows:
\begin{align*}
    \to I(0) \qquad I(m) \to I(m+1) \qquad I(m) \to I(m-1)
\end{align*}
Also, we can replace atoms ${A_1(\mathbf t_1,m) \wedge A_2(\mathbf t_2,m)}$
with conjunction ${A_1(\mathbf t_1,m) \wedge A_2(\mathbf t_2,m') \wedge (m \leq
m') \wedge (m' \leq m)}$, where $m'$ is a fresh variable.

Intuitively, a limit fact $B(\mathbf a,k)$ says that the value of $B$ for a
tuple of objects $\mathbf{a}$ is at least $k$ (if $B$ is $\tmax$) or at most
$k$ (if $B$ is $\tmin$). For example, a fact $sp(v,k)$ in our shortest path
example from Section~\ref{sec:introduction} says that node $v$ is reachable
from $v_0$ via a path with cost at most $k$. To capture this intended meaning,
we require interpretations $I$ to be \emph{closed} for limit predicates---that
is, whenever $I$ contains a limit fact $\alpha$, it also contains all facts
implied by $\alpha$ according to the predicate type. In our example, this
captures the observation that the existence of a path from $v_0$ to $v$ of cost
at most $k$ implies the existence of such a path of cost at most $k'$ for each
${k' \geq k}$.

\begin{definition}\label{def:limit-closed-int}
    An interpretation $I$ is \emph{limit-closed} if, for each limit fact
    ${B(\mathbf a,k) \in I}$ where $B$ is a $\tmin$ (resp.\ $\tmax$) predicate,
    ${B(\mathbf a,k') \in I}$ holds for each integer $k'$ with ${k \leq k'}$
    (resp.\ ${k' \leq k}$). An interpretation $I$ is a \emph{model} of a limit
    program $\Prog$ if ${I \models \Prog}$ and $I$ is limit-closed. The notion
    of entailment is modified to take into account only limit-closed models.
\end{definition}

The semantics of limit predicates in a limit $\DLog$ program $\Prog$ can be
axiomatised explicitly by extending $\Prog$ with the following rules, where $Z$
is a fresh predicate. Thus, limit $\DLog$ can be seen as a syntactic fragment
of $\DLog$.
\begin{displaymath}
\begin{array}{@{}c@{}}
    \to Z(0) \qquad Z(m) \to Z(m+1) \qquad Z(m) \to Z(m-1) \\[0.5ex]
    B(\mathbf{x},m) \wedge Z(n) \wedge (m \leq n) \to B(\mathbf{x},n) \\
    \hspace{3.7cm} \text{for each } \tmin \text{ predicate } B\text{ in } \Prog \\[0.5ex]
    B(\mathbf{x},m) \wedge Z(n) \wedge (n \leq m) \to B(\mathbf{x},n) \\
    \hspace{3.7cm} \text{for each } \tmax \text{ predicate } B \text{ in } \Prog \\
\end{array}
\end{displaymath}

Each limit program can be reduced to a homogeneous program; however, for the
sake of generality, in our technical results we do not require programs to be
homogeneous.

\begin{restatable}{proposition}{homogeneous}\label{prop:homogeneous}
   For each limit program $\Prog$ and fact $\alpha$, a homogeneous program
   $\Prog'$ and fact $\alpha'$ can be computed in linear time such that ${\Prog
   \models \alpha}$ if and only if ${\Prog' \models \alpha'}$.
\end{restatable}

Intuitively, program $\Prog'$ in Proposition~\ref{prop:homogeneous} is obtained
by replacing all $\tmin$ (or all $\tmax$) predicates in $\Prog$ by fresh
$\tmax$ (resp.\ $\tmin$) predicates and negating their numeric arguments.

In Section~\ref{sec:introduction} we have shown that limit $\DLog$ can compute
the cost of shortest paths in a graph. We next present further examples of data
analysis tasks that our formalism can handle. In all examples, we assume that
all objects in the input are arranged in an arbitrary linear order using facts
${\mathit{first}(a_1)}$, ${\mathit{next}(a_1,a_2)}$, $\dots$,
${\mathit{next}(a_{n-1},a_n)}$; we use this order to simulate aggregation by
means of recursion.

\begin{example}\label{ex:diffusion}
Consider a social network where agents are connected by the `follows' relation.
Agent $a_s$ introduces (tweets) a message, and each agent $a_i$ retweets the
message if at least $k_{a_i}$ agents that $a_i$ follows tweet the message,
where $k_{a_i}$ is a positive threshold uniquely associated with $a_i$. Our
goal is to determine which agents tweet the message eventually. To achieve this
using limit $\DLog$, we encode the network structure in a dataset
$\Dat_{\mathit{tw}}$ containing facts $\mathit{follows}(a_i,a_j)$ if $a_i$
follows $a_j$, and ordinary numeric facts $\mathit{th}(a_i,k_{a_i})$ if $a_i$'s
threshold is $k_{a_i}$. Program $\Prog_{\mathit{tw}}$, containing rules
\eqref{ex:diffusion:init}--\eqref{ex:diffusion:publish}, encodes message
propagation, where $\mathit{nt}$ is a $\tmax$ predicate.
\begin{align}
                                                                        & \to \mathit{tw}(a_s)      \label{ex:diffusion:init} \\
    \mathit{follows}(x,y') \land \mathit{first}(y)                      & \to \mathit{nt}(x,y,0)    \label{ex:diffusion:first:1} \\
    \mathit{follows}(x,y) \land \mathit{first}(y) \land \mathit{tw}(y) & \to \mathit{nt}(x,y,1)     \label{ex:diffusion:first:2} \\
    \mathit{nt}(x,y',m) \land \mathit{next}(y',y)                       & \to \mathit{nt}(x,y,m)    \label{ex:diffusion:next:1} \\
    \begin{array}{@{}r@{}}
        \mathit{nt}(x,y',m) \land \mathit{next}(y',y) \\
        \mathit{follows}(x,y) \land \mathit{tw}(y) \\
    \end{array}                                                         &
    \begin{array}{@{\;}l@{}}
        \land \\
        \to \mathit{nt}(x,y,m+1) \\
    \end{array}                                                                                     \label{ex:diffusion:next:2} \\
    \mathit{th}(x,m) \land \mathit{nt}(x,y,n) \land (m \le n)           & \to \mathit{tw}(x)        \label{ex:diffusion:publish}
\end{align}
Specifically, ${\Prog_{\mathit{tw}} \cup \Dat_{\mathit{tw}} \models
\mathit{tw}(a_i)}$ iff $a_i$ tweets the message. Intuitively,
$\mathit{nt}(a_i,a_j,m)$ is true if, out of agents ${\{ a_1, \dots, a_j \}}$
(according to the order), at least $m$ agents that $a_i$ follows tweet the
message. Rules \eqref{ex:diffusion:first:1} and \eqref{ex:diffusion:first:2}
initialise $\mathit{nt}$ for the first agent in the order; $\mathit{nt}$ is a
$\tmax$ predicate, so if the first agent tweets the message, rule
\eqref{ex:diffusion:first:2} `overrides' rule \eqref{ex:diffusion:first:1}.
Rules \eqref{ex:diffusion:next:1} and \eqref{ex:diffusion:next:2} recurse over
the order to compute $\mathit{nt}$ as stated above.
\end{example}

\begin{example}\label{ex:counting-paths}
Limit $\DLog$ can also solve the problem of counting paths between pairs of
nodes in a directed acyclic graph. We encode the graph in the obvious way as a
dataset $\Dat_{\mathit{cp}}$ that uses object predicates $\mathit{node}$ and
$\mathit{edge}$. Program $\Prog_{\mathit{cp}}$, consisting of rules
\eqref{eq:cp1}--\eqref{eq:cp6} where $\mathit{np}$ and $\mathit{np'}$ are
$\tmax$ predicates, then counts the paths.
\begin{align} 
    \mathit{node}(x)                                                    & \to \mathit{np}(x,x,1)    \label{eq:cp1} \\
    \mathit{node}(x) \land \mathit{node}(y) \land \mathit{first}(z)     & \to \mathit{np'}(x,y,z,0) \label{eq:cp2} \\
    \begin{array}{@{}r@{}}
        \mathit{edge}(x,z) \\
        \mathit{np}(z,y,m) \land \mathit{first}(z) \\
    \end{array}                                                         &
    \begin{array}{@{\;}l@{}}
        \land \\
        \to \mathit{np'}(x,y,z,m) \\
    \end{array}                                                                                     \label{eq:cp3} \\
    \mathit{np'}(x,y,z',m) \land \mathit{next}(z',z)                    & \to\mathit{np'}(x,y,z,m)  \label{eq:cp4} \\
    \begin{array}{@{}r@{}}
        \mathit{np'}(x,y,z',m) \land \mathit{next}(z',z) \\
        \mathit{edge}(x,z) \land \mathit{np}(z,y,n)
    \end{array}                                                         & 
    \begin{array}{@{\;}l@{}}
        \land \\
        \to \mathit{np'}(x,y,z,m+n) \\
    \end{array}                                                                                     \label{eq:cp5} \\
    \mathit{np'}(x,y,z,m)                                               & \to \mathit{np}(x,y,m)    \label{eq:cp6}
\end{align}
Specifically, ${\Prog_{\mathit{cp}} \cup \Dat_{\mathit{cp}} \models
\mathit{np}(a_i,a_j,k)}$ iff at least $k$ paths exist from node $a_i$ to node
$a_j$. Intuitively, $\mathit{np'}(a_i,a_j,a_k,m)$ is true if $m$ is at least
the sum of the number of paths from each ${a' \in \{ a_1, \dots, a_k \}}$
(according to the order) to $a_j$ for which there exists an edge from $a_i$ to
$a'$. Rule \eqref{eq:cp1} says that each node has one path to itself. Rule
\eqref{eq:cp2} initialises aggregation by saying that, for the first node $z$,
there are zero paths from $x$ to $y$, and rule \eqref{eq:cp3} overrides this if
there exists an edge from $x$ to $z$. Finally, rule \eqref{eq:cp4} propagates
the sum for $x$ to the next $z$ in the order, and rule \eqref{eq:cp5} overrides
this if there is an edge from $x$ to $z$ by adding the number of paths from
$z'$ and $z$ to $y$.
\end{example}

\begin{example}\label{ex:bandwidth}
Assume that, in the graph from Example~\ref{ex:counting-paths}, each node $a_i$
is associated with a bandwidth $b_{a_i}$ limiting the number of paths going
through $a_i$ to at most $b_{a_i}$. To count the paths compliant with the
bandwidth requirements, we extend $\Dat_{\mathit{cp}}$ to dataset
$\Dat_{\mathit{bcp}}$ that additionally contains an ordinary numeric fact
$\mathit{bw}(a_i,b_{a_i})$ for each node $a_i$, and we define
$\Prog_{\mathit{bcp}}$ by replacing rule \eqref{eq:cp6} in
$\Prog_{\mathit{cp}}$ with the following rule.
\begin{align*}
    \mathit{np'}(x,y,z,m) \land \mathit{bw}(z,n) \land (m \le n) \to \mathit{np}(x,y,m)
\end{align*}
Then, ${\Prog_{\mathit{bcp}} \cup \Dat_{\mathit{bcp}} \models
\mathit{np}(a_i,a_j,k)}$ iff there exist at least $k$ paths from node $a_i$ to
node $a_j$, where the bandwidth requirement is satisfied for all nodes on each
such path.
\end{example} 

%%% Local Variables:
%%% mode: latex
%%% TeX-master: "paper"
%%% End:

\section{Fixpoint Characterisation of Entailment}\label{sec:fixpoints}

Programs are often grounded to eliminate variables and thus simplify the
presentation. In limit $\DLog$, however, numeric variables range over integers,
so a grounding can be infinite. Thus, we first specialise the notion of a
grounding.

\begin{definition}\label{def:semi-grounding}
    A rule $r$ is \emph{semi-ground} if each variable in $r$ is a numeric
    variable that occurs in $r$ in a limit body atom. A limit program $\Prog$
    is \emph{semi-ground} if all of its rules are semi-ground. The
    \emph{semi-grounding} of\/ $\Prog$ contains, for each ${r \in \Prog}$, each
    rule obtained from $r$ by replacing each variable not occurring in $r$ in a
    numeric argument of a limit atom with a constant of\/ $\Prog$.
\end{definition}

Obviously, ${\Prog \models \alpha}$ if and only if ${\Prog' \models \alpha}$
for $\Prog'$ the semi-grounding of $\Prog$. We next characterise entailment of
limit programs by \emph{pseudo-interpretations}, which compactly represent
limit-closed interpretations. If a limit-closed interpretation $I$ contains
$B(\mathbf b,k)$ where $B$ is a $\tmin$ predicate, then either the \emph{limit}
value ${\ell \leq k}$ exists such that ${B(\mathbf b,\ell) \in I}$ and
${B(\mathbf b,k') \not\in I}$ for ${k' < \ell}$, or ${B(\mathbf b,k') \in I}$
holds for all ${k' \leq k}$, and dually for $B$ a $\tmax$ predicate. Thus, to
characterise the value of $B$ on a tuple of objects ${\mathbf b}$ in $I$, we
just need the limit value, or information that no such value exists.

\begin{definition}\label{def:pseudo-interpretation}
    A \emph{pseudo-interpretation} $J$ is a set of facts over integers extended
    with a special symbol $\infty$ such that ${k = k'}$ holds for all limit
    facts $B(\mathbf b,k)$ and $B(\mathbf b,k')$ in $I$.
\end{definition}

Limit-closed interpretations correspond naturally and one-to-one to
pseudo-interpretations, so we can recast the notions of satisfaction and model
using pseudo-interpretations. Unlike for interpretations, the number of facts
in a pseudo-model of a semi-ground limit program $\Prog$ can be bounded by
$|\Prog|$.

\begin{definition}\label{def:pseudo-model}
    A limit-closed interpretation $I$ \emph{corresponds} to a
    pseudo-interpretation $J$ if $I$ contains exactly all object and ordinary
    numeric facts of $J$, and, for each limit predicate $B$, each tuple of
    objects ${\mathbf b}$, and each integer $\ell$, (i)~${B(\mathbf b,k) \in
    I}$ for all $k$ if and only if ${B(\mathbf b,\infty) \in J}$, and
    (ii)~${B(\mathbf b,\ell) \in I}$ and ${B(\mathbf b,k) \not\in I}$ for all
    ${k < \ell}$ (resp.\ ${\ell < k}$) and $B$ is a $\tmin$ (resp.\ $\tmax$)
    predicate if and only if ${B(\mathbf b,\ell) \in J}$.

    Let $J$ and $J'$ be pseudo-interpretations corresponding to interpretations
    $I$ and $I'$. Then, $J$ \emph{satisfies} a ground atom $\alpha$, written
    ${J \models \alpha}$, if ${I \models \alpha}$; $J$ is a \emph{pseudo-model}
    of a program $\Prog$, written ${J \models \Prog}$, if ${I \models \Prog}$;
    finally, ${J \sqsubseteq J'}$ holds if ${I \subseteq I'}$.
\end{definition}

\begin{example}
Let $I$ be the interpretation consisting of $A(1)$, $A(2)$, $B(a,k)$ for ${k
\leq 5}$, and $B(b,k)$ for ${k \in \mathbb Z}$, where $A$ is an ordinary
numeric predicate, $B$ is a $\tmax$ predicate, and $a$ and $b$ are objects.
Then, $\set{A(1),A(2),B(a,5),B(b,\infty)}$ is the pseudo-interpretation
corresponding to $I$.
\end{example}

We next introduce the \emph{immediate consequence} operator
$\ILFPStepOp{\Prog}{}$ of a limit program $\Prog$ on pseudo-interpretations. We
assume for simplicity that $\Prog$ is semi-ground. To apply a rule ${r \in
\Prog}$ to a pseudo-interpretation $J$ while correctly handling limit atoms,
operator $\ILFPStepOp{\Prog}{}$ converts $r$ into a linear integer constraint
$\constraints{r}{J}$ that captures all ground instances of $r$ applicable to
the limit-closed interpretation $I$ corresponding to $J$. If
$\constraints{r}{J}$ has no solution, $r$ is not applicable to $J$. Otherwise,
$\head{r}$ is added to $J$ if it is not a limit atom; and if $\head{r}$ is a
$\tmin$ ($\tmax$) atom ${B(\mathbf b,m)}$, then the minimal (maximal) solution
$\ell$ for $m$ in $\constraints{r}{J}$ is computed, and $J$ is updated such
that the limit value of $B$ on ${\mathbf b}$ is at least (at most)
$\ell$---that is, the application of $r$ to $J$ keeps only the `best' limit
value.

\begin{definition}\label{def:rule-app}
    For $\Prog$ a semi-ground limit program, ${r \in \Prog}$, and $J$ a
    pseudo-interpretation, $\constraints{r}{J}$ is the conjunction of
    comparison atoms containing (i)~$\cbody{r}$; (ii)~${(0 < 0)}$ if an object
    or ordinary numeric atom ${\alpha \in \sbody{r}}$ exists with ${\alpha
    \not\in J}$, or a limit atom ${B(\mathbf b,s) \in \sbody{r}}$ exists with
    ${B(\mathbf b,\ell) \not\in J}$ for each $\ell$; and (iii)~${(\ell \leq
    s)}$ (resp.\ ${(s \leq \ell)}$) for each $\tmin$ (resp.\ $\tmax$) atom
    ${B(\mathbf b,s) \in \sbody{r}}$ with ${B(\mathbf b,\ell) \in J}$ and
    ${\ell \neq \infty}$. Rule $r$ is \emph{applicable} to $J$ if
    $\constraints{r}{J}$ has an integer solution.
    
    Assume $r$ is applicable to $J$. If $\head{r}$ is an object or ordinary
    numeric atom, let ${\mathsf{hd}(r,J) = \head{r}}$. If ${\head{r} =
    B(\mathbf b, s)}$ is a $\tmin$ (resp.\ $\tmax$) atom, the \emph{optimum
    value} $\mathsf{opt}(r,J)$ is the smallest (resp.\ largest) value of $s$ in
    all solutions to $\constraints{r}{J}$, or $\infty$ if no such bound on the
    value of $s$ in the solutions to $\constraints{r}{J}$ exists; moreover,
    ${\mathsf{hd}(r,J) = B(\mathbf b, \mathsf{opt}(r,J))}$.
    
    Operator $\ILFPStep{\Prog}{J}$ maps $J$ to the smallest (w.r.t.\
    $\sqsubseteq$) pseudo-interpretation satisfying $\mathsf{hd}(r,J)$ for each
    ${r \in \Prog}$ applicable to $J$. Finally, ${\ILFPStepOp{\Prog}{0} =
    \emptyset}$, and ${\ILFPStepOp{\Prog}{n} =
    \ILFPStep{\Prog}{\ILFPStepOp{\Prog}{n-1}}}$ for ${n > 0}$.
\end{definition}

\begin{example}
Let $r$ be ${A(x) \land (2 \leq x) \to B(x+1)}$ with $A$ and $B$ $\tmax$
predicates. Then, ${\constraints{r}{\emptyset}=(2 \leq x) \land (0<0)}$ does
not have a solution, and therefore rule $r$ is not applicable to the empty
pseudo-interpretation. Moreover, for ${J = \set{A(3)}}$, conjunction
${\constraints{r}{J} = (2 \leq x) \land (x \leq 3)}$ has two
solutions---$\set{x \mapsto 2}$ and $\set{x \mapsto 3}$---and therefore rule
$r$ is applicable to $J$. Finally, $B$ is a $\tmax$ predicate, and so
${\mathsf{opt}(r,J) = \max\set{2+1,3+1} = 4}$, and ${\mathsf{hd}(r,J) = B(4)}$.
Consequently, ${\ILFPStep{\set{r}}{J} = \set{B(4)}}$.
\end{example}

\begin{restatable}{lemma}{pseudomodelfixedpoint}\label{lem:imm-cons-op-props}
    For each semi-ground limit program $\Prog$, operator $\ILFPStepOp{\Prog}{}$
    is monotonic w.r.t.\ $\sqsubseteq$. Moreover, ${J \models \Prog}$ if and
    only if ${\ILFPStep{\Prog}{J} \sqsubseteq J}$ for each
    pseudo-interpretation $J$.
\end{restatable}

Monotonicity ensures existence of the \emph{closure}
$\ILFPStepOp{\Prog}{\infty}$ of $\Prog$---the least pseudo-interpretation such
that ${\ILFPStepOp{\Prog}{n} \sqsubseteq \ILFPStepOp{\Prog}{\infty}}$ for each
${n \geq 0}$. The following theorem characterises entailment and provides a
bound on the number of facts in the closure.

\begin{restatable}{theorem}{fixpoint}\label{thm:pseudo-mat-entails}
    For $\Prog$ a semi-ground limit program and $\alpha$ a fact, ${\Prog
    \models \alpha}$ if and only if ${\ILFPStepOp{\Prog}{\infty} \models
    \alpha}$; also, ${|\ILFPStepOp{\Prog}{\infty}| \leq |\Prog|}$; and ${J
    \models \Prog}$ implies ${\ILFPStepOp{\Prog}{\infty} \sqsubseteq J}$ for
    each pseudo-interpretation~$J$.
\end{restatable}

The proofs for the first and the third claim of
Theorem~\ref{thm:pseudo-mat-entails} use the monotonicity of
$\ILFPStepOp{\Prog}{}$ analogously to plain Datalog. The second claim holds
since, for each ${n \geq 0}$, each pair of distinct facts in
$\ILFPStepOp{\Prog}{n}$ must be derived by distinct rules in $\Prog$.

%%% Local Variables:
%%% mode: latex
%%% TeX-master: "paper"
%%% End:

\section{Decidability of Entailment: Limit-Linearity}\label{sec:decidability}

We now start our investigation of the computational properties of limit
$\DLog$. Theorem \ref{thm:pseudo-mat-entails} bounds the cardinality of the
closure of a semi-ground program, but it does not bound the magnitude of the
integers occurring in limit facts; in fact, integers can be arbitrarily large.
Moreover, due to multiplication, checking rule applicability requires solving
nonlinear inequalities over integers, which is undecidable.

\begin{restatable}{theorem}{typedfactentailmentundecidable}\label{thm:typed-ent-undecidable}
    For $\Prog$ a semi-ground limit program and $\alpha$ a fact, checking
    ${\Prog \models \alpha}$ and checking applicability of a rule of\/ $\Prog$
    to a pseudo-interpretation are both undecidable.
\end{restatable}

The proof of Theorem~\ref{thm:typed-ent-undecidable} uses a straightforward
reduction from Hilbert's tenth problem.

Checking rule applicability is undecidable due to products of variables in
inequalities; however, for \emph{linear inequalities} that prohibit multiplying
variables, the problem can be solved in $\textsc{NP}$, and in polynomial time
if we bound the number of variables. Thus, to ensure decidability, we next
restrict limit programs so that their semi-groundings contain only linear
numeric terms. All our examples satisfy this restriction.

\begin{definition}\label{def:limit-linear}
    A limit rule $r$ is \emph{limit-linear} if each numeric term in $r$ is of
    the form ${s_0 + \sum_{i=1}^n s_i \times m_i}$, where (i)~each $m_i$ is a
    distinct numeric variable occurring in a limit body atom of $r$, (ii)~term
    $s_0$ contains no variable occurring in a limit body atom of $r$, and
    (iii)~each $s_i$ with ${i \geq 1}$ is a term constructed using
    multiplication $\times$, integers, and variables not occurring in limit
    body atoms of $r$. A limit-linear program contains only limit-linear
    rules.\footnote{Note that each multiplication-free limit program can be
    normalised in polynomial time to a limit-linear program.}
\end{definition}

In the rest of this section, we show that entailment for limit-linear programs
is decidable and provide tight complexity bounds. Our upper bounds are obtained
via a reduction to the validity of Presburger formulas of a certain shape.

\begin{restatable}{lemma}{presburgerencoding}\label{lem:presburger-encoding}
    For $\Prog$ a semi-ground limit-linear program and $\alpha$ a fact, there
    exists a Presburger sentence ${\varphi = \forall \mathbf x \exists \mathbf
    y.\bigvee_{i=1}^n\psi_i}$ that is valid if and only if ${\Prog \models
    \alpha}$. Each $\psi_i$ is a conjunction of possibly negated atoms.
    Moreover, ${|\mathbf x| + |\mathbf y|}$ and each $\ssize{\psi_i}$ are
    bounded polynomially by ${\ssize{\Prog} + \ssize{\alpha}}$. Number $n$ is
    bounded polynomially by $|\Prog|$ and exponentially by $\max_{r\in \Prog}
    \ssize{r}$. Finally, the magnitude of each integer in $\varphi$ is bounded
    by the maximal magnitude of an integer in $\Prog$ and $\alpha$.
\end{restatable}

The reduction in Lemma \ref{lem:presburger-encoding} is based on three main
ideas. First, for each limit atom $B(\mathbf b,s)$ in a semi-ground program
$\Prog$, we use a Boolean variable $\defined{B}{\mathbf b}$ to indicate that an
atom of the form $B(\mathbf b,\ell)$ exists in a pseudo-model of $\Prog$, a
Boolean variable $\fin{B}{\mathbf b}$ to indicate whether the value of $\ell$
is finite, and an integer variable $\val{B}{\mathbf b}$ to capture $\ell$ if it
is finite. Second, each rule of $\Prog$ is encoded as a universally quantified
Presburger formula by replacing each standard atom with its encoding. Finally,
entailment of $\alpha$ from $\Prog$ is encoded as a sentence stating that, in
every pseudo-interpretation, either some rule in $\Prog$ is not satisfied, or
$\alpha$ holds; this requires universal quantifiers to quantify over all
models, and existential quantifiers to negate the (universally quantified)
program.
 
Lemma~\ref{lem:validity-presburger} bounds the magnitude of integers in models
of Presburger formulas from Lemma \ref{lem:presburger-encoding}. These bounds
follow from recent deep results on semi-linear sets and their connection to
Presburger arithmetic \cite{DBLP:conf/icalp/ChistikovH16}.\footnote{We thank
  Christoph Haase for providing a proof of this lemma.} %based on his results.}

\begin{restatable}{lemma}{presburgervalidity}\label{lem:validity-presburger}
    Let ${\varphi = \forall\mathbf x \exists\mathbf y.\bigvee_{i=1}^n\psi_i}$
    be a Presburger sentence where each $\psi_i$ is a conjunction of possibly
    negated atoms of size at most $k$ mentioning at most $\ell$ variables, $a$
    is the maximal magnitude of an integer in $\varphi$, and ${m = |\mathbf
    x|}$. Then, $\varphi$ is valid if and only if $\varphi$ is valid over
    models where each integer variable assumes a value whose magnitude is
    bounded by ${(2^{O(\ell\log\ell)} \cdot a^{k\ell})^{n2^\ell\cdot O(m^4)}}$.
\end{restatable}

Lemmas~\ref{lem:presburger-encoding} and~\ref{lem:validity-presburger} provide
us with bounds on the size of counter-pseudo-models for entailment.

\begin{restatable}{theorem}{counterinterpretationsbounded}\label{th:pseudo}
    For $\Prog$ a semi-ground limit-linear program, $\Dat$ a dataset, and
    $\alpha$ a fact, ${\Prog \cup \Dat \not\models \alpha}$ if and only if a
    pseudo-model $J$ of\/ ${\Prog \cup \Dat}$ exists where ${J \not\models
    \alpha}$, ${|J| \leq |\Prog \cup \Dat|}$, and the magnitude of each integer
    in $J$ is bounded polynomially by the largest magnitude of an integer in
    ${\Prog \cup \Dat}$, exponentially by $|\Prog|$, and double-exponentially
    by ${\max_{r \in \Prog} \ssize{r}}$.
\end{restatable}

By Theorem \ref{th:pseudo}, the following nondeterministic algorithm decides
${\Prog \not\models \alpha}$:
\begin{compactenum}
    \item compute the semi-grounding $\Prog'$ of $\Prog$;

    \item guess a pseudo-interpretation $J$ that satisfies the bounds given in
    Theorem~\ref{th:pseudo};
    
    \item if ${\ILFPStep{\Prog'}{J} \sqsubseteq J}$ (so ${J \models \Prog'}$)
    and ${J \not\models \alpha}$, return $\mathsf{true}$.
\end{compactenum}
Step~1 requires exponential (polynomial in data) time, and it does not increase
the maximal size of a rule. Hence, Step~2 is nondeterministic exponential
(polynomial in data), and Step~3 requires exponential (polynomial in data) time
to solve a system of linear inequalities. Theorem~\ref{thm:conp-conexptime}
proves that these bounds are both correct and tight.

\begin{restatable}{theorem}{admissiblefactentailmentconpconexptime}\label{thm:conp-conexptime}
    For $\Prog$ a limit-linear program and $\alpha$ a fact, deciding ${\Prog
    \models \alpha}$ is \textsc{coNExpTime}-complete in combined and
    \textsc{coNP}-complete in data complexity.
\end{restatable}

The upper bounds in Theorem~\ref{thm:conp-conexptime} follow from
Theorem~\ref{th:pseudo}, \textsc{coNP}-hardness in data complexity is shown by
a reduction from the square tiling problem, and \textsc{coNExpTime}-hardness in
combined complexity is shown by a similar reduction from the succinct version
of square tiling.

%%% Local Variables:
%%% mode: latex
%%% TeX-master: "paper"
%%% End:

\section{Tractability of Entailment: Stability}\label{sec:tractability}

Tractability in data complexity is important on large datasets, so we next
present an additional \emph{stability} condition that brings the complexity of
entailment down to $\textsc{ExpTime}$ in combined and $\textsc{PTime}$ in data
complexity, as in plain Datalog.

\subsection{Cyclic Dependencies in Limit Programs}

The fixpoint of a plain Datalog program can be computed in $\textsc{PTime}$ in
data complexity. However, for $\Prog$ a limit-linear program, a na{\"i}ve
computation of $\ILFPStepOp{\Prog}{\infty}$ may not terminate since repeated
application of $\ILFPStepOp{\Prog}{}$ can produce larger and larger numbers.
Thus, we need a way to identify when the numeric argument $\ell$ of a limit
fact ${A(\mathbf a,\ell)}$ \emph{diverges}---that is, grows or decreases
without a bound; moreover, to obtain a procedure tractable in data complexity,
divergence should be detected after polynomially many steps.
Example~\ref{ex:stability} illustrates that this can be achieved by analysing
cyclic dependencies.

\begin{example}\label{ex:stability}
Let $\Prog_c$ contain facts $A(0)$ and $B(0)$, and rules ${A(m) \to B(m)}$ and
${B(m) \to A(m+1)}$, where $A$ and $B$ are $\tmax$ predicates. Applying the
first rule copies the value of $A$ into $B$, and applying the second rule
increases the value of $A$; thus, both $A$ and $B$ diverge in
$\ILFPStepOp{\Prog_c}{\infty}$. The existence of a cyclic dependency between
$A$ and $B$, however, does not necessarily lead to divergence. Let program
$\Prog'_c$ be obtained from $\Prog_c$ by adding a $\tmax$ fact $C(5)$ and
replacing the first rule with ${A(m) \land C(n)\land (m\le n) \to B(m)}$. While
a cyclic dependency between $A$ and $B$ still exists, the increase in the
values of $A$ and $B$ is bounded by the value of $C$, which is independent of
$A$ or $B$; thus, neither $A$ nor $B$ diverge in
$\ILFPStepOp{\Prog'_c}{\infty}$.
\end{example}

In the rest of this section, we extend $<$, $+$, and $-$ to $\infty$ by
defining ${k < \infty}$, ${\infty + k = \infty}$, and ${\infty - k = \infty}$
for each integer $k$. We formalise cyclic dependencies as follows.

\begin{definition}\label{def:vp-graph}
    For each $n$-ary limit predicate $B$ and each tuple ${\mathbf b}$ of $n-1$
    objects, let $\pgNode{B}{\mathbf b}$ be a \emph{node} unique for $B$ and
    ${\mathbf b}$. The \emph{value propagation graph} of a semi-ground
    limit-linear program $\Prog$ and a pseudo-interpretation $J$ is the
    directed weighted graph ${G^J_\Prog = (\pgNodes,\pgEdges,\mu)}$ defined as
    follows.
    \begin{compactenum}
        \item For each limit fact ${B(\mathbf b,\ell)\in J}$, we have
        ${\pgNode{B}{\mathbf b}\in\pgNodes}$.

        \item For each rule ${r \in \Prog}$ applicable to $J$ with the head of
        the form ${A(\mathbf a,s)}$ where ${\pgNode{A}{\mathbf a} \in
        \pgNodes}$ and each body atom ${B(\mathbf b,m)}$ of\/ $r$ where
        ${\pgNode{B}{\mathbf b} \in \pgNodes}$ and variable $m$ occurs in term
        $s$, we have ${\pgEdge{B}{\mathbf b}{A}{\mathbf a} \in \pgEdges}$; such
        $r$ is said to \emph{produce} the edge ${\pgEdge{B}{\mathbf
        b}{A}{\mathbf a}}$ in $\pgEdges$.

        \item For each ${r \in \Prog}$ and each edge ${e = \pgEdge{B}{\mathbf
        b}{A}{\mathbf a} \in \pgEdges}$ produced by $r$, ${\delta_r^e(J) =
        \infty}$ if ${\mathsf{opt}(r,J) = \infty}$; otherwise,
        \begin{displaymath}
            \hskip-5mm\delta_r^e(J) =
            \begin{cases}
                \mathsf{opt}(r,J)  - \ell   & \text{if } B \text{ and } A \text{ are }\tmax, \ell \neq \infty\\
                - \mathsf{opt}(r,J)  - \ell & \text{if } B \text{ is } \tmax, A \text{ is } \tmin, \ell \neq \infty\\
                - \mathsf{opt}(r,J)  + \ell & \text{if } B \text{ and } A \text{ are } \tmin, \ell \neq \infty\\
                \mathsf{opt}(r,J) + \ell    & \text{if } B \text{ is } \tmin, A \text{ is } \tmax, \ell \neq \infty\\
                \mathsf{opt}(r,J)           & \text{if } \ell = \infty
            \end{cases}
        \end{displaymath}
        where $\ell$ is such that ${B(\mathbf b,\ell) \in J}$. The weight of
        each edge ${e \in \pgEdges}$ is then given by
        \begin{displaymath}
            \mu(e) = \max \qset{\delta_r^e(J)}{r \in \Prog \text{ produces } e }.
        \end{displaymath}
    \end{compactenum}
    A \emph{positive-weight cycle} in $G^J_\Prog$ is a cycle for which the sum
    of the weights of the contributing edges is greater than $0$.
\end{definition}

Intuitively, $G^J_\Prog$ describes how, for each limit predicate $B$ and
objects ${\mathbf b}$ such that ${B(\mathbf b,\ell) \in J}$, operator
$\ILFPStepOp{\Prog}{}$ propagates $\ell$ to other facts. The presence of a node
${\pgNode{B}{\mathbf b}}$ in $\pgNodes$ indicates that ${B(\mathbf b,\ell) \in
J}$ holds for some ${\ell \in \mathbb Z \cup \set{\infty}}$; this $\ell$ can be
uniquely identified given $\pgNode{B}{\mathbf b}$ and $J$. An edge ${e =
\pgEdge{B}{\mathbf b}{A}{\mathbf a} \in \pgEdges}$ indicates that at least one
rule ${r \in \Prog}$ is applicable to $J$ where ${\head{r} = A(\mathbf a,s)}$,
${B(\mathbf b,m) \in \sbody{r}}$, and $m$ occurs in $s$; moreover, applying $r$
to $J$ produces a fact ${A(\mathbf a,\ell')}$ where $\ell'$ satisfies ${\ell +
\mu(e) \leq \ell'}$ if both $A$ and $B$ are $\tmax$ predicates, and analogously
for the other types of $A$ and $B$. In other words, edge $e$ indicates that the
application of $\ILFPStepOp{\Prog}{}$ to $J$ will propagate the value of
$\pgNode{B}{\mathbf b}$ to $\pgNode{A}{\mathbf a}$ while increasing it by at
least $\mu(e)$. Thus, presence of a positive-weight cycle in $G^J_\Prog$
indicates that repeated rule applications might increment the values of all
nodes on the cycle.

\subsection{Stable Programs}

As Example~\ref{ex:stability} shows, the presence of a positive-weight cycle in
$G^J_\Prog$ does not imply the divergence of all atoms corresponding to the
nodes in the cycle. This is because the weight of such a cycle may decrease
after certain rule applications and so it is no longer positive.

This motivates the \emph{stability} condition, where edge weights in
$G^J_\Prog$ may only grow but never decrease with rule application. Hence, once
the weight of a cycle becomes positive, it will remain positive and thus
guarantee the divergence of all atoms corresponding to its nodes. Intuitively,
$\Prog$ is stable if, whenever a rule ${r \in \Prog}$ is applicable to some
$J$, rule $r$ is also applicable to each $J'$ with larger limit values, and
applying $r$ to such $J'$ further increases the value of the head.
Definition~\ref{def:stable-programs} defines stability as a condition on
$G^J_\Prog$. Please note that, for all pseudo-interpretations $J$ and $J'$ with
${J \sqsubseteq J'}$ and ${G^J_\Prog = (\pgNodes,\pgEdges,\mu)}$ and
${G^{J'}_\Prog = (\pgNodes',\pgEdges',\mu')}$ the corresponding value
propagation graphs, we have ${\pgEdges \subseteq \pgEdges'}$.

\begin{definition}\label{def:stable-programs}
    A semi-ground limit-linear program $\Prog$ is \emph{stable} if, for all
    pseudo-interpretations $J$ and $J'$ with ${J \sqsubseteq J'}$, ${G^J_\Prog
    = (\pgNodes,\pgEdges,\mu)}$, ${G^{J'}_\Prog = (\pgNodes',\pgEdges',\mu')}$,
    and each ${{e \in \pgEdges}}$,
    \begin{compactenum}
        \item ${\mu(e) \leq \mu'(e)}$, and

        \item ${e = \pgEdge{B}{\mathbf b}{A}{\mathbf a}}$ and ${B(\mathbf
        b,\infty)\in J}$ imply ${\mu(e) = \infty}$.
    \end{compactenum}
    A limit-linear program is \emph{stable} if its semi-grounding is stable.
\end{definition}

\begin{example}
Program $\Prog_c$ from Example~\ref{ex:stability} is stable, while $\Prog'_c$
is not: for ${J = \set{A(0),C(0)}}$ and ${J' = \set{A(1),C(0)}}$, we have ${J
\sqsubseteq J'}$, but ${\mu(\pgEdge{A}{}{B}{}) = 0}$ and
${\mu'(\pgEdge{A}{}{B}{}) = -1}$.
\end{example}

For each semi-ground program $\Prog$ and each integer $n$, we have
${\ILFPStepOp\Prog{n} \sqsubseteq \ILFPStepOp\Prog{n+1}}$, and stability
ensures that edge weights only grow after rule application. Thus, recursive
application of the rules producing edges involved in a positive-weight cycle
leads to divergence, as shown by the following lemma.

\begin{restatable}{lemma}{soundness}\label{lem:soundness}
    For each semi-ground stable program $\Prog$, each pseudo-interpretation $J$
    with ${J \sqsubseteq \ILFPStepOp{\Prog}{\infty}}$, and each node
    $\pgNode{A}{\mathbf a}$ on a positive-weight cycle in $G^J_\Prog$, we have
    ${A(\mathbf a,\infty) \in \ILFPStepOp{\Prog}{\infty}}$.
\end{restatable}

Algorithm~\ref{alg:lim-stab-fp} uses this observation to deterministically
compute the fixpoint of $\Prog$. The algorithm iteratively applies
$\ILFPStepOp{\Prog}{}$; however, after each step, it computes the corresponding
value propagation graph (line~\ref{step:compute-dep-graph}) and, for each
${A(\mathbf a,\ell)}$ where node $\pgNode{A}{\mathbf a}$ occurs on a
positive-weight cycle (line~\ref{step:loop2-begin2}), it replaces $\ell$ with
$\infty$ (line~\ref{step:infty-propagation}). By Lemma~\ref{lem:soundness},
this is sound. Moreover, since the algorithm repeatedly applies
$\ILFPStepOp{\Prog}{}$, it necessarily derives each fact from
$\ILFPStepOp{\Prog}{\infty}$ eventually. Finally, Lemma~\ref{lem:termination}
shows that the algorithm terminates in time polynomial in the number of rules
in a semi-ground program. Intuitively, the proof of the lemma shows that,
without introducing a new edge or a new positive weight cycle in the value
propagation graph, repeated application of $\ILFPStepOp{\Prog}{}$ necessarily
converges in $O(|\Prog|^2)$ steps; moreover, the number of edges in $G^J_\Prog$
is at most quadratic $|\Prog|$, and so a new edge or a new positive weight
cycle can be introduced at most $O(|\Prog|^2)$ many times.

\begin{algorithm}[t]
\begin{footnotesize}
\caption{Entailment for Semi-Ground Stable Programs}\label{alg:lim-stab-fp}
\textbf{Input:} semi-ground stable program $\Prog$, fact $\alpha$ \\
\textbf{Output:} $\mathsf{true}$ if $\Prog \models \alpha$
\begin{algorithmic}[1]
    \State $J' \defeq \emptyset$
    \Repeat                                                                                                 \label{step:loop1-begin}
        \State $J \defeq J'$                                                                                \label{step:copy-j}
        \State $G^{J}_{\Prog} \defeq (\pgNodes,\pgEdges,\mu)$                                               \label{step:compute-dep-graph}
        \For{\textbf{each} $\pgNode{A}{\mathbf a} \in \pgNodes$ in a positive-weight cycle in $G^J_\Prog$}  \label{step:loop2-begin2}
            \State replace $A(\mathbf a,\ell)$ in $J$ with $A(\mathbf a,\infty)$                            \label{step:infty-propagation}
        \EndFor                                                                                             \label{step:loop2-end}
        \State $J' \defeq \ILFPStep{\Prog}{J}$                                                              \label{step:compute-imm-cons}
        \Until{$J=J'$}                                                                                      \label{step:loop1-end}
    \State \Return $\mathsf{true}$ if $J \models \alpha$ and $\mathsf{false}$ otherwise
\end{algorithmic}
\end{footnotesize}
\end{algorithm}

\begin{restatable}{lemma}{termination}\label{lem:termination}
    When applied to a semi-ground stable program $\Prog$,
    Algorithm~\ref{alg:lim-stab-fp} terminates after at most $8|\Prog|^6$
    iterations of the loop in
    lines~\ref{step:loop1-begin}--\ref{step:loop1-end}.
\end{restatable}

Lemmas~\ref{lem:soundness} and~\ref{lem:termination} imply the following
theorem.

\begin{restatable}{theorem}{algorithmcorrect}\label{thm:correctness}
    For $\Prog$ a semi-ground stable program, $\Dat$ a dataset, and $\alpha$ a
    fact, Algorithm~\ref{alg:lim-stab-fp} decides ${\Prog \cup \Dat \models
    \alpha}$ in time polynomial in $\ssize{\Prog \cup \Dat}$ and exponential in
    ${\max_{r \in \Prog} \ssize{r}}$.
\end{restatable}

Since the running time is exponential in the maximal size of a rule, and
semi-grounding does not increase rule sizes, Algorithm~\ref{alg:lim-stab-fp}
combined with a semi-grounding preprocessing step provides an exponential time
decision procedure for stable, limit-linear programs. This upper bound is tight
since entailment in plain Datalog is already $\textsc{ExpTime}$-hard in
combined and $\textsc{PTime}$-hard in data complexity.

\begin{restatable}{theorem}{stablecomplexity}
    For $\Prog$ a stable program and $\alpha$ a fact, checking ${\Prog \models
    \alpha}$ is \textsc{ExpTime}-complete in combined and
    \textsc{PTime}-complete in data complexity.
\end{restatable}

\subsection{Type-Consistent Programs}

Unfortunately, the class of stable programs is not recognisable, which can
again be shown by a reduction from Hilbert's tenth problem.

\begin{restatable}{proposition}{stabilityundecidable}
    Checking stability of a limit-linear program $\Prog$ is undecidable.
\end{restatable}

We next provide a sufficient condition for stability that captures programs
such as those in Examples~\ref{ex:diffusion} and~\ref{ex:counting-paths}. Intuitively,
Definition~\ref{def:type-consistent} syntactically prevents certain harmful
interactions. In the second rule of program $\Prog'_c$ from
Example~\ref{ex:stability}, numeric variable $m$ occurs in a $\tmax$ atom and
on the left-hand side of a comparison atom ${(m \leq n)}$; thus, if the rule is
applicable for some value of $m$, it is not necessarily applicable for each
${m' \geq m}$, which breaks stability.

\begin{definition}\label{def:type-consistent}
    A semi-ground limit-linear rule $r$ is \emph{type-consistent} if
    \begin{compactitem}[--]
        \item each numeric term $t$ in $r$ is of the form ${k_0 + \sum_{i=1}^n
        k_i \times m_i}$ where $k_0$ is an integer and each $k_i$, ${1 \leq i
        \leq n}$, is a nonzero integer, called the \emph{coefficient of
        variable $m_i$ in $t$};

        \item if ${\head{r} = A(\mathbf a, s)}$ is a limit atom, then each
        variable occurring in $s$ with a positive (resp.\ negative) coefficient
        also occurs in a (unique) limit body atom or $r$ that is of the same
        (resp.\ different) type (i.e., $\tmin$ vs.\ $\tmax$) as ${\head{r}}$;
        and
        
        \item for each comparison ${(s_1 < s_2)}$ or ${(s_1 \leq s_2)}$ in $r$,
        each variable occurring in $s_1$ with a positive (resp.\ negative)
        coefficient also occurs in a (unique) $\tmin$ (resp.\ $\tmax$) body
        atom, and each variable occurring in $s_2$ with a positive (resp.\
        negative) coefficient also occurs in a (unique) $\tmax$ (resp.\
        $\tmin$) body atom of\/ $r$.
    \end{compactitem}
    A semi-ground limit-linear program is \emph{type-consistent} if all of its
    rules are type-consistent. Moreover, a limit-linear program $\Prog$ is
    \emph{type-consistent} if the program obtained by first semi-grounding
    $\Prog$ and then simplifying all numeric terms as much as possible is
    type-consistent.
\end{definition}

The first condition of Definition~\ref{def:type-consistent} ensures that each
variable occurring in a numeric term contributes to the value of the term. For
example, it disallows terms such as $0\cdot x$ and $x-x$, since a rule with
such a term in the head may violate the second condition. Moreover, the second
condition of Definition~\ref{def:type-consistent} ensures that, if the value of
a numeric variable $x$ occurring in the head `increases' w.r.t.\ the type of
the body atom introducing $x$ (i.e., $x$ increases if it occurs in a $\tmax$
body atom and decreases otherwise), then so does the value of the numeric term
in the head; this is essential for the first condition of stability (cf.\
Definition~\ref{def:stable-programs}). Finally, the third condition of
Definition~\ref{def:type-consistent} ensures that comparisons cannot be
invalidated by `increasing' the values of the variables involved, which is
required for both conditions of stability.

Type consistency is a purely syntactic condition that can be checked by looking
at one rule and one atom at a time. Hence, checking type consistency is
feasible in $\textsc{LogSpace}$.

\begin{restatable}{proposition}{typeconsistentprogramsstable}
    Each type-consistent limit-linear program is stable.
\end{restatable}

\begin{restatable}{proposition}{typeconsistentprogramschecking}
    Checking whether a limit-linear program is type-consistent can be
    accomplished in $\textsc{LogSpace}$.
\end{restatable}

%%% Local Variables:
%%% mode: latex
%%% TeX-master: "paper"
%%% End:

\section{Conclusion and Future Work}

We have introduced several decidable/tractable fragments of Datalog with
integer arithmetic, thus obtaining a sound theoretical foundation for
declarative data analysis. We see many challenges for future work. First, our
formalism should be extended with aggregate functions. While certain forms of
aggregation can be simulated by iterating over the object domain, as in our
examples in Section~\ref{sec:limit-programs}, such a solution may be too
cumbersome for practical use, and it relies on the existence of a linear order
over the object domain, which is a strong theoretical assumption. Explicit
support for aggregation would allow us to formulate tasks such as the ones in
Section~\ref{sec:limit-programs} more intuitively and without relying on the
ordering assumption. Second, it is unclear whether integer constraint solving
is strictly needed in Step~\ref{step:compute-imm-cons} of
Algorithm~\ref{alg:lim-stab-fp}: it may be possible to exploit stability of
$\Prog$ to compute $\ILFPStep\Prog J$ more efficiently. Third, we shall
implement our algorithm and apply it to practical data analysis problems.
Fourth, it would be interesting to establish connections between our results
and existing work on data-aware artefact
systems~\cite{DBLP:journals/tods/DamaggioDV12,DBLP:journals/jcss/KoutsosV17},
which faces similar undecidability issues in a different formal setting.

%%% Local Variables:
%%% mode: latex
%%% TeX-master: "paper"
%%% End:

\section*{Acknowledgments}

We thank Christoph Haase for explaining to us his results on Presburger
arithmetic and semi-linear sets as well as for providing a proof for
Lemma~\ref{lem:validity-presburger}. Our work has also benefited from
discussions with Michael Benedikt.  This research was supported by the Royal
Society and the EPSRC projects DBOnto, MaSI$^3$, and ED$^3$.

\bibliographystyle{named}
\bibliography{references}

\begin{thebibliography}{}

\bibitem[\protect\citeauthoryear{Alvaro \bgroup \em et al.\egroup
  }{2010}]{DBLP:conf/eurosys/AlvaroCCEHS10}
Peter Alvaro, Tyson Condie, Neil Conway, Khaled Elmeleegy, Joseph~M.
  Hellerstein, and Russell Sears.
\newblock {BOOM} analytics: exploring data-centric, declarative programming for
  the cloud.
\newblock In {\em EuroSys}. {ACM}, 2010.

\bibitem[\protect\citeauthoryear{Beeri \bgroup \em et al.\egroup
  }{1991}]{DBLP:journals/jlp/BeeriNST91}
Catriel Beeri, Shamim~A. Naqvi, Oded Shmueli, and Shalom Tsur.
\newblock Set constructors in a logic database language.
\newblock {\em J. Log. Program.}, 10(3{\&}4), 1991.

\bibitem[\protect\citeauthoryear{Berman}{1980}]{DBLP:journals/tcs/Berman80}
Leonard Berman.
\newblock The complexitiy of logical theories.
\newblock {\em Theor. Comput. Sci.}, 11, 1980.

\bibitem[\protect\citeauthoryear{Byrd \bgroup \em et al.\egroup
  }{1987}]{ByrdGH87}
Richard~H. Byrd, Alan~J. Goldman, and Miriam Heller.
\newblock Recognizing unbounded integer programs.
\newblock {\em Oper. Res.}, 35(1), 1987.

\bibitem[\protect\citeauthoryear{Chin \bgroup \em et al.\egroup
  }{2015}]{DBLP:conf/snapl/ChinDEHMOOP15}
Brian Chin, Daniel von Dincklage, Vuk Ercegovac, Peter Hawkins, Mark~S. Miller,
  Franz~Josef Och, Christopher Olston, and Fernando Pereira.
\newblock Yedalog: Exploring knowledge at scale.
\newblock In {\em {SNAPL}}, 2015.

\bibitem[\protect\citeauthoryear{Chistikov and
  Haase}{2016}]{DBLP:conf/icalp/ChistikovH16}
Dmitry Chistikov and Christoph Haase.
\newblock The taming of the semi-linear set.
\newblock In {\em {ICALP}}, 2016.

\bibitem[\protect\citeauthoryear{Consens and
  Mendelzon}{1993}]{DBLP:journals/tcs/ConsensM93}
Mariano~P. Consens and Alberto~O. Mendelzon.
\newblock Low complexity aggregation in {GraphLog} and {Datalog}.
\newblock {\em Theor. Comput. Sci.}, 116(1), 1993.

\bibitem[\protect\citeauthoryear{Damaggio \bgroup \em et al.\egroup
  }{2012}]{DBLP:journals/tods/DamaggioDV12}
Elio Damaggio, Alin Deutsch, and Victor Vianu.
\newblock Artifact systems with data dependencies and arithmetic.
\newblock {\em ACM Trans. Database Syst.}, 37(3):22:1--22:36, 2012.

\bibitem[\protect\citeauthoryear{Dantsin \bgroup \em et al.\egroup
  }{2001}]{DBLP:journals/csur/DantsinEGV01}
Evgeny Dantsin, Thomas Eiter, Georg Gottlob, and Andrei Voronkov.
\newblock Complexity and expressive power of logic programming.
\newblock {\em ACM Comput. Surv.}, 33(3), 2001.

\bibitem[\protect\citeauthoryear{Eisner and
  Filardo}{2011}]{DBLP:conf/datalog/EisnerF10}
Jason Eisner and Nathaniel~Wesley Filardo.
\newblock Dyna: Extending datalog for modern {AI}.
\newblock In {\em Datalog}, 2011.

\bibitem[\protect\citeauthoryear{Faber \bgroup \em et al.\egroup
  }{2011}]{DBLP:journals/ai/FaberPL11}
Wolfgang Faber, Gerald Pfeifer, and Nicola Leone.
\newblock Semantics and complexity of recursive aggregates in answer set
  programming.
\newblock {\em Artif. Intell.}, 175(1), 2011.

\bibitem[\protect\citeauthoryear{Ganguly \bgroup \em et al.\egroup
  }{1995}]{DBLP:journals/jcss/GangulyGZ95}
Sumit Ganguly, Sergio Greco, and Carlo Zaniolo.
\newblock Extrema predicates in deductive databases.
\newblock {\em J. Comput. Syst. Sci.}, 51(2), 1995.

\bibitem[\protect\citeauthoryear{Gr{\"{a}}del}{1988}]{DBLP:journals/tcs/Gradel88}
Erich Gr{\"{a}}del.
\newblock Subclasses of presburger arithmetic and the polynomial-time
  hierarchy.
\newblock {\em Theor. Comput. Sci.}, 56, 1988.

\bibitem[\protect\citeauthoryear{Haase}{2014}]{DBLP:conf/csl/Haase14}
Christoph Haase.
\newblock Subclasses of {Presburger} arithmetic and the weak {EXP} hierarchy.
\newblock In {\em {CSL-LICS}}, 2014.

\bibitem[\protect\citeauthoryear{Hougardy}{2010}]{DBLP:journals/ipl/Hougardy10}
Stefan Hougardy.
\newblock The {Floyd}-{Warshall} algorithm on graphs with negative cycles.
\newblock {\em Inf. Process. Lett.}, 110(8-9), 2010.

\bibitem[\protect\citeauthoryear{Kannan}{1987}]{DBLP:journals/mor/Kannan87}
Ravi Kannan.
\newblock Minkowski's convex body theorem and integer programming.
\newblock {\em Math. Oper. Res.}, 12(3):415--440, 1987.

\bibitem[\protect\citeauthoryear{Kemp and
  Stuckey}{1991}]{DBLP:conf/slp/KempS91}
David~B. Kemp and Peter~J. Stuckey.
\newblock Semantics of logic programs with aggregates.
\newblock In {\em {ISLP}}, 1991.

\bibitem[\protect\citeauthoryear{Koutsos and
  Vianu}{2017}]{DBLP:journals/jcss/KoutsosV17}
Adrien Koutsos and Victor Vianu.
\newblock Process-centric views of data-driven business artifacts.
\newblock {\em J. Comput. System Sci.}, 86:82--107, 2017.

\bibitem[\protect\citeauthoryear{Loo \bgroup \em et al.\egroup
  }{2009}]{DBLP:journals/cacm/LooCGGHMRRS09}
Boon~Thau Loo, Tyson Condie, Minos~N. Garofalakis, David~E. Gay, Joseph~M.
  Hellerstein, Petros Maniatis, Raghu Ramakrishnan, Timothy Roscoe, and Ion
  Stoica.
\newblock Declarative networking.
\newblock {\em Commun. {ACM}}, 52(11), 2009.

\bibitem[\protect\citeauthoryear{Markl}{2014}]{DBLP:journals/pvldb/Markl14}
Volker Markl.
\newblock Breaking the chains: On declarative data analysis and data
  independence in the big data era.
\newblock {\em {PVLDB}}, 7(13), 2014.

\bibitem[\protect\citeauthoryear{Mazuran \bgroup \em et al.\egroup
  }{2013}]{DBLP:journals/vldb/MazuranSZ13}
Mirjana Mazuran, Edoardo Serra, and Carlo Zaniolo.
\newblock Extending the power of datalog recursion.
\newblock {\em {VLDB} J.}, 22(4), 2013.

\bibitem[\protect\citeauthoryear{Mumick \bgroup \em et al.\egroup
  }{1990}]{DBLP:conf/vldb/MumickPR90}
Inderpal~Singh Mumick, Hamid Pirahesh, and Raghu Ramakrishnan.
\newblock The magic of duplicates and aggregates.
\newblock In {\em VLDB}, pages 264--277, 1990.

\bibitem[\protect\citeauthoryear{Papadimitriou}{1981}]{DBLP:journals/jacm/Papadimitriou81}
Christos~H. Papadimitriou.
\newblock On the complexity of integer programming.
\newblock {\em J. ACM}, 28(4), 1981.

\bibitem[\protect\citeauthoryear{Ross and Sagiv}{1997}]{RossS97}
Kenneth~A. Ross and Yehoshua Sagiv.
\newblock Monotonic aggregation in deductive databases.
\newblock {\em J. Comput. System Sci.}, 54(1), 1997.

\bibitem[\protect\citeauthoryear{Sch{\"{o}}ning}{1997}]{DBLP:journals/mst/Schoning97}
Uwe Sch{\"{o}}ning.
\newblock Complexity of presburger arithmetic with fixed quantifier dimension.
\newblock {\em Theory Comput. Syst.}, 30(4), 1997.

\bibitem[\protect\citeauthoryear{Seo \bgroup \em et al.\egroup
  }{2015}]{DBLP:journals/tkde/SeoGL15}
Jiwon Seo, Stephen Guo, and Monica~S. Lam.
\newblock {SociaLite}: An efficient graph query language based on datalog.
\newblock {\em {IEEE} Trans. Knowl. Data Eng.}, 27(7), 2015.

\bibitem[\protect\citeauthoryear{Shkapsky \bgroup \em et al.\egroup
  }{2016}]{DBLP:conf/sigmod/ShkapskyYICCZ16}
Alexander Shkapsky, Mohan Yang, Matteo Interlandi, Hsuan Chiu, Tyson Condie,
  and Carlo Zaniolo.
\newblock Big data analytics with datalog queries on {Spark}.
\newblock In {\em {SIGMOD}}. {ACM}, 2016.

\bibitem[\protect\citeauthoryear{{Van Gelder}}{1992}]{DBLP:conf/pods/Gelder92}
Allen {Van Gelder}.
\newblock The well-founded semantics of aggregation.
\newblock In {\em {PODS}}, 1992.

\bibitem[\protect\citeauthoryear{{von zur Gathen} and
  Sieveking}{1978}]{GathenSieveking78}
Joachim {von zur Gathen} and Malte Sieveking.
\newblock A bound on solutions of linear integer equalities and inequalities.
\newblock {\em Proc. {AMS}}, 72(1), 1978.

\bibitem[\protect\citeauthoryear{Wang \bgroup \em et al.\egroup
  }{2015}]{DBLP:journals/pvldb/WangBH15}
Jingjing Wang, Magdalena Balazinska, and Daniel Halperin.
\newblock Asynchronous and fault-tolerant recursive datalog evaluation in
  shared-nothing engines.
\newblock {\em {PVLDB}}, 8(12), 2015.

\end{thebibliography}

\clearpage
\onecolumn
\appendix
\counterwithin{theorem}{section}
\renewcommand{\theproposition}{\thesection.\arabic{theorem}}
\renewcommand{\thecorollary}{\thesection.\arabic{theorem}}
\renewcommand{\thelemma}{\thesection.\arabic{theorem}}
\renewcommand{\thedefinition}{\thesection.\arabic{theorem}}
\renewcommand{\theclaim}{\thesection.\arabic{claim}}
\renewcommand{\theexample}{\thesection.\arabic{theorem}}

\ifdraft{
    \section{Proofs for Section~\ref{sec:limit-programs}}\label{sec:proofs-opt-progs}

\factentundecidable*

\begin{proof}
We prove the claim by presenting a reduction of the halting problem for
deterministic Turing machines on the empty tape. Let $M$ be an arbitrary
deterministic Turing machine with finite alphabet $\Gamma$ containing the blank
symbol $\,\text{\textvisiblespace}\,$, the finite set of states $S$ containing
the initial state $s$ and the halting state $h$, and transition function
${\delta : S \times \Gamma \to S \times \Gamma \times \set{L,R}}$. We assume
that $M$ works on a tape that is infinite to the right, that it starts with the
empty tape and the head positioned on the leftmost cell, and that it never
moves the head off the left edge of the tape.

We encode each time point ${i \geq 0}$ using an integer $2^i$, and we index
tape positions using zero-based integers; thus, at time $i$, each position ${j
\geq 2^i}$ is necessarily empty, so we can encode a combination of a time point
$i$ and tape position $j$ with ${0 \leq j < 2^i}$ using a single integer ${2^i
+ j}$. We use this idea to encode the state of the execution of $M$ using the
following facts:
\begin{itemize}
    \item $\mathit{Num}(k)$ is true for each positive number $k$;
    
    \item $\mathit{Time}(k)$ is true if ${k = 2^i}$ and so $k$ encodes a time
    point $i$;
    
    \item $\mathit{Tape}(a,2^i+j)$ says that symbol $a$ occupies position $j$
    of the tape at time $i$, and it will be defined for each ${0 \leq j < 2^i}$;
    
    \item $\mathit{Pos}(2^i+j)$ says that the head points to position $j$ of
    the tape at time $i$;
    
    \item $\mathit{State}(q,2^i)$ says that the machine is in state $q$ at time
    $i$; and
    
    \item $\mathit{Halts}$ is a propositional variable saying that the machine
    has halted.
\end{itemize}

We next give a $\DLog$ program $\Prog_M$ that simulates the behaviour of $M$ on
the empty tape. We represent each alphabet symbol ${a \in \Gamma}$ using an
object constant $a$, and we represent each state ${q \in S}$ using an object
constant $q$. Furthermore, we abbreviate ${(s \leq t) \land (t < u)}$ as ${(s
\leq t < u)}$. Finally, we abbreviate conjunction ${(s \leq t) \land (t \leq
s)}$ as ${(s \doteq t)}$, and disjunction ${(s < t) \lor (t > s)}$ as ${(s
\not\doteq t)}$. Strictly speaking, disjunctions are not allowed in rule
bodies; however, each rule with a disjunction in the body of the form ${\varphi
\land (s \not\doteq t) \to \alpha}$ corresponds to rules ${\varphi \land (s <
t) \to \alpha}$ and ${\varphi \land (t < s) \to \alpha}$, so we use the former
form for the sake of clarity. With these considerations in mind, program
$\Prog_M$ contains rules
\eqref{eq:undec:Num:1}--\eqref{eq:undec:Tape:inertia:2}.
\begin{align}
                                                                                                                                    & \to \mathit{Num}(1)                               \label{eq:undec:Num:1}\\
    \mathit{Num}(x)                                                                                                                 & \to \mathit{Num}(x+1)                             \label{eq:undec:Num:plus1}\\
                                                                                                                                    & \to \mathit{Time}(1)                              \label{eq:undec:Time:1}\\
    \mathit{Time}(x)                                                                                                                & \to \mathit{Time}(x+x)                            \label{eq:undec:Time:power}\\
                                                                                                                                    & \to \mathit{Tape}(\,\text{\textvisiblespace}\,,1) \label{eq:undec:Tape:init}\\
                                                                                                                                    & \to \mathit{Pos}(1)                               \label{eq:undec:Pos:init}\\
                                                                                                                                    & \to \mathit{State}(s,1)                           \label{eq:undec:State:init}\\
    \mathit{State}(h,x)                                                                                                             & \to \mathit{Halts}                                \label{eq:undec:Tape:halting}\\
    \mathit{Time}(x) \land \mathit{Tape}(v,y) \land \mathit{Pos}(z) \land (x \le y<x+x) \land (x \le z<x+x) \land (y \not\doteq z)  & \to \mathit{Tape}(v,x+y)                          \label{eq:undec:Tape:inertia:1}\\
    \mathit{Time}(x) \land \mathit{Num}(u) \land (x+x+x \leq u+u<x+x+x+x)                                                           & \to \mathit{Tape}(\,\text{\textvisiblespace}\,,u) \label{eq:undec:Tape:inertia:2}
\end{align}
Moreover, for each alphabet symbol ${a \in \Gamma}$ and all states ${q,q' \in
S}$ such that ${\delta(q,a) = (q',a',D)}$ where ${D \in \set{L,R}}$ is a
direction, $\Prog_M$ contains rules
\eqref{eq:undec:Tape:change}--\eqref{eq:undec:Pos:R}.
\begin{align}
    \mathit{Time}(x) \land \mathit{State}(q,x) \land \mathit{Tape}(a,y) \land \mathit{Pos}(y) \land (x \le y<x+x)                   & \to \mathit{Tape}(a',x+y)                         \label{eq:undec:Tape:change}\\
    \mathit{Time}(x) \land \mathit{State}(q,x) \land \mathit{Tape}(a,y) \land \mathit{Pos}(y) \land (x \le y<x+x)                   & \to \mathit{State}(q',x+x)                        \label{eq:undec:State:change}\\
    \begin{array}{@{}r@{}}
        \mathit{Time}(x) \land \mathit{State}(q,x) \land \mathit{Tape}(a,y) \land \mathit{Pos}(y) \land \mathit{Num}(u) \\
        (x \le y<x+x) \land (x+y \doteq u+1) \\
    \end{array}                                                                                                                     &
    \begin{array}{@{\;}lr@{}}
        \land \\
        \to \mathit{Pos}(u) \qquad \text{if } D = L \\
    \end{array}                                                                                                                                                                         \label{eq:undec:Pos:L}\\
    \begin{array}{@{}r@{}}
        \mathit{Time}(x) \land \mathit{State}(q,x) \land \mathit{Tape}(a,y) \land \mathit{Pos}(y) \land \mathit{Num}(u) \\
        (x \le y<x+x) \land (x+y+1 \doteq u) \\
    \end{array}                                                                                                                     &
    \begin{array}{@{\;}l@{}}
        \land \\
        \to \mathit{Pos}(u) \qquad \text{if } D = R \\
    \end{array}                                                                                                                                                                         \label{eq:undec:Pos:R}
\end{align}

Rules \eqref{eq:undec:Num:1}--\eqref{eq:undec:Num:plus1} initialise
$\textit{Num}$ so that it holds of all positive integers, and rules
\eqref{eq:undec:Time:1}--\eqref{eq:undec:Time:power} initialise $\textit{Time}$
so that it holds for each integer ${k = 2^i}$. Rules
\eqref{eq:undec:Tape:init}--\eqref{eq:undec:State:init} initialise the state of
the $M$ at time ${i = 0}$. Rule \eqref{eq:undec:Tape:halting} derives
$\mathit{Halts}$ if at any point the Turing machine enters the halting state
$h$. The remaining rules encode the evolution of the state of $M$, and they are
based on the following idea: if variable $x$ encodes a time point $i$ using
value $2^i$, then variable $y$ encodes a position $j$ for time point $i$ if ${x
\leq y < x+x}$ holds; moreover, for such $y$, position $j$ at time point $i+1$
is encoded as ${2^{i+1} + j = 2^i + 2^i + j}$ and can be obtained as $x+y$, and
the encodings of positions $j-1$ and $j+1$ can be obtained as ${x+y-1}$ and
${x+y+1}$, respectively. Since our goal is to prove undecidability by just
using $+$, we simulate subtraction by looking for a value $u$ such that ${x+y =
u+1}$. With these observations in mind, one can see that rule
\eqref{eq:undec:Tape:inertia:1} copies the unaffected part of the tape from
time point $i$ to time point $i+1$. Moreover, rule
\eqref{eq:undec:Tape:inertia:2} pads the tape by filling each location $j$ with
${1.5 \times 2^i \leq j < 2 \times 2^i}$ with the blank symbol; since division
is not supported in our language, we express this condition as ${3 \times 2^i
\leq 2 \times j < 2 \times 2 \times 2^i}$. Finally, rule
\eqref{eq:undec:Tape:change} updates the tape at the position of the head; rule
\eqref{eq:undec:State:change} updates the state; and rules
\eqref{eq:undec:Pos:L} and \eqref{eq:undec:Pos:R} move the head left and right,
respectively. Consequently, we have ${\Prog_M \models \mathit{Halts}}$ if and
only if $M$ halts on the empty tape.
\end{proof}

\homogeneous*

\begin{proof}
Let $\Prog$ be an arbitrary limit program. Without loss of generality, we
construct a program $\Prog'$ containing only $\tmax$ predicates. For each
$\tmin$ predicate $A$, let $A'$ be a fresh $\tmax$ predicate uniquely
associated with $A$. We construct $\Prog'$ from $\Prog$ by modifying each rule
${r \in \Prog}$ as follows:
\begin{enumerate}
    \item if ${\head{r} = A(\mathbf t,s)}$ where $A$ is a $\tmin$ predicate,
    replace the head of $r$ with $A'(\mathbf t,-s)$;

    \item for each body atom ${A(\mathbf t,n) \in \sbody{r}}$ where $A$ is a
    $\tmin$ predicate and $n$ is a variable, replace the atom with ${A'(\mathbf
    t,m)}$ where $m$ is a fresh variable, and replace all other occurrences of
    $n$ in the rule with $-m$;

    \item for each body atom ${A(\mathbf t,k) \in \sbody{r}}$ where $A$ is a
    $\tmin$ predicate and $k$ is an integer, replace the atom with ${A'(\mathbf
    t,-k)}$.
\end{enumerate}
Finally, if ${\alpha = A(\mathbf a,k)}$ is a $\tmin$ fact, let ${\alpha' =
A(\mathbf a,-k)}$; otherwise, let ${\alpha' = \alpha}$. Now consider an
arbitrary interpretation $I$, and let $I'$ be the interpretation obtained from
$I$ by replacing each $\tmin$ fact ${A(\mathbf a,k)}$ with ${A'(\mathbf
a,-k)}$; it is straightforward to see that ${I \models \Prog}$ if and only if
${I' \models \Prog'}$, and that ${I \models \alpha}$ if and only if ${I'
\models \alpha'}$. Thus, ${\Prog \models \alpha}$ if and only if ${\Prog'
\models \alpha'}$.
\end{proof}

\section{Proofs for Section~\ref{sec:fixpoints}}\label{sec:proofs-pseudo-models}

We use the standard notion of substitutions---sort-compatible partial mappings
of variables to constants. For $\varphi$ a formula and $\sigma$ a substitution,
$\varphi\sigma$ is the formula obtained by replacing each free variable $x$ in
$\varphi$ on which $\sigma$ is defined with $\sigma(x)$.

\begin{proposition}\label{prop:applicability-characterisation}
    For each semi-ground rule ${r = \varphi \to \alpha}$, each
    pseudo-interpretation $J$, and each mapping $\sigma$ of the variables of
    $r$ to integers, $\sigma$ is an integer solution to $\constraints{r}{J}$ if
    and only if ${J \models \varphi\sigma}$.
\end{proposition}

\begin{proof}
($\Rightarrow$) Assume that $\sigma$ is an integer solution to
$\constraints{r}{J}$. We consider each atom ${\beta \in \varphi}$ and show that
${J \models \beta\sigma}$ holds.
\begin{itemize}
    \item If $\beta$ is a comparison atom, the claim is straightforward due to
    ${\beta \in \constraints{r}{J}}$.

    \item If $\beta$ is an object atom or an ordinary numeric atom, $\beta$ is
    ground and we have ${\beta \in J}$ and ${J \models \beta}$; otherwise, ${(0
    < 0) \in \constraints{r}{J}}$ would hold and so $\sigma$ could not be a
    solution to $\constraints{r}{J}$.

    \item If $\beta$ is a $\tmax$ atom $B(\mathbf b,s)$, since ${(0 < 0)
    \not\in \constraints{r}{J}}$, either ${B(\mathbf b,\ell) \in J}$ for some
    integer ${\ell \in \mathbb{Z}}$ and ${(s \le \ell) \in
    \constraints{r}{J}}$, or ${B(\mathbf b,\infty) \in J}$. In the former case,
    since $\sigma$ is a solution to $\constraints{r}{J}$, we have ${s\sigma \le
    \ell}$, and, since $B$ is a $\tmax$ predicate, ${J \models B(\mathbf
    b,s\sigma)}$ holds. In the latter case, ${J \models B(\mathbf b,s\sigma)}$
    holds due to ${\set{ B(\mathbf b,s\sigma) } \sqsubseteq \set{ B(\mathbf
    b,\infty) }}$.
    
    \item If $\beta$ is a $\tmin$ atom, the proof is analogous to the previous
    case.
\end{itemize}
The proof of the ($\Leftarrow$) direction is analogous and we omit it for the
sake of brevity.
\end{proof}

\begin{definition}
    Given a limit-closed interpretation $I$ and a program $\Prog$, let
    \begin{displaymath}
        \ICFPStep{\Prog}{I} = \qset{\gamma}{ \textstyle\bigwedge_i \alpha_i \to \beta \text{ is a ground instance of a rule in } \Prog \text{ such that } I \models \textstyle\bigwedge_i \alpha_i \text{ and } \gamma \text{ is a fact such that } \set{ \gamma } \sqsubseteq \set{ \beta } }.
    \end{displaymath}
    Let ${\ICFPStepOp{\Prog}{0} = \emptyset}$, let ${\ICFPStepOp{\Prog}{n} =
    \ICFPStep{\Prog}{\ICFPStepOp{\Prog}{n-1}}}$ for ${n > 0}$, and let
    ${\ICFPStepOp{\Prog}{\infty} = \bigcup_{n \geq 0} \ICFPStepOp{\Prog}{n}}$.
\end{definition}

\begin{lemma}\label{lemma:operator:I}
    For each program $\Prog$, operator $\ICFPStepOp{\Prog}{}$ is monotonic
    w.r.t.\ $\subseteq$; moreover, for $I$ an interpretation, ${I \models
    \Prog}$ if and only if ${\ICFPStep{\Prog}{I} = I}$, and ${I \models \Prog}$
    implies ${\ICFPStepOp{\Prog}{\infty} \subseteq I}$.
\end{lemma}

\begin{proof}
Operator $\ICFPStepOp{\Prog}{}$ is the standard immediate consequence operator
of Datalog, but applied to the program $\Prog'$ obtained by extending $\Prog$
with the rules from Section~\ref{sec:limit-programs} encoding the semantics of
limit predicates. Thus, all claims of this lemma hold in the usual way
\cite{DBLP:journals/csur/DantsinEGV01}.
\end{proof}

\begin{lemma}\label{lemma:icfp-ilfp-coincide}
    For each limit-closed interpretation $I$ and the corresponding
    pseudo-interpretation $J$, and for each semi-ground limit program $\Prog$,
    interpretation $\ICFPStep{\Prog}{I}$ corresponds to the
    pseudo-interpretation $\ILFPStep{\Prog}{J}$.
\end{lemma}

\begin{proof}
It suffices to show that, for each fact $\alpha$, the following claims hold:
\begin{enumerate}
    \item if $\alpha$ is an object fact or ordinary numeric fact, then ${\alpha
    \in \ICFPStep{\Prog}{I}}$ if and only if ${\alpha \in \ILFPStep{\Prog}{J}}$;

    \item if $\alpha$ is a limit fact of the form $A(\mathbf a,k)$ where $k$ is
    an integer, then ${\alpha \in \ICFPStep{\Prog}{I}}$ if and only if
    ${\set{\alpha} \sqsubseteq \ILFPStep{\Prog}{J}}$; and

    \item if $\alpha$ is a limit fact of the form ${A(\mathbf a,\infty)}$, then
    ${\qset{A(\mathbf a,k)}{k \in \mathbb Z} \subseteq \ICFPStep{\Prog}{I}}$ if
    and only if ${\alpha \in \ILFPStep{\Prog}{J}}$.
\end{enumerate}

\smallskip

(Claim~1) Consider an arbitrary object fact $\alpha$ of the form ${A(\mathbf
a)}$; the proof for ordinary numeric facts is analogous.
\begin{itemize}
    \item Assume ${\alpha \in \ICFPStep{\Prog}{I}}$. Then, a rule ${r = \varphi
    \to \alpha \in \Prog}$ and a grounding $\sigma$ of $r$ exist such that ${I
    \models \varphi\sigma}$; the head of $r$ must be $\alpha$ since $\Prog$ is
    semi-ground. But then, ${J \models \varphi\sigma}$ holds as well, so
    Proposition~\ref{prop:applicability-characterisation} ensures that $\sigma$
    is a solution to $\constraints{r}{J}$; moreover, ${\mathsf{hd}(r,J) =
    \alpha}$, and thus we have ${\alpha \in \ILFPStep{\Prog}{J}}$.

    \item Assume ${\alpha \in \ILFPStep{\Prog}{J}}$. Then, there exist a rule
    ${r = \varphi \to \alpha \in \Prog}$ and an integer solution $\sigma$ to
    $\constraints{r}{J}$. Proposition~\ref{prop:applicability-characterisation}
    then ensures ${J \models \varphi\sigma}$, and so ${I \models
    \varphi\sigma}$ holds as well. Thus, we have ${\alpha\sigma = \alpha \in
    \ICFPStep{\Prog}{I}}$.
\end{itemize}

\smallskip

(Claim~2) Consider an arbitrary $\tmax$ fact $\alpha$ of the form ${A(\mathbf
a,k)}$; the proof for a $\tmin$ fact is analogous.
\begin{itemize}
    \item Assume ${\alpha \in \ICFPStep{\Prog}{I}}$. Then, a rule ${r = \varphi
    \to A(\mathbf a,s) \in \Prog}$ and a grounding $\sigma$ of $r$ exist such
    that ${I \models \varphi\sigma}$ and ${\alpha = A(\mathbf a,s\sigma)}$. But
    then, ${J \models \varphi\sigma}$ holds as well, so
    Proposition~\ref{prop:applicability-characterisation} ensures that $\sigma$
    is a solution to $\constraints{r}{J}$; moreover, ${s\sigma \le
    \mathsf{opt}(r,J)}$, and therefore we have ${\set{\alpha} \sqsubseteq
    \set{\mathsf{hd}(r,J)} = \set{ A(\mathbf a,\mathsf{opt}(r,J)) } \sqsubseteq
    \ILFPStep{\Prog}{J}}$.

    \item Assume ${\set{\alpha} \sqsubseteq \ILFPStep{\Prog}{J}}$. Then, there
    exist a rule ${r = \varphi \to A(\mathbf a,s) \in \Prog}$ and an integer
    solution $\sigma$ to $\constraints{r}{J}$ such that ${\set{\alpha}
    \sqsubseteq \set{\mathsf{hd}(r,J)} = \set{ A(\mathbf a,s\sigma)}}$ where
    ${s\sigma = \mathsf{opt}(r,J)}$.
    Proposition~\ref{prop:applicability-characterisation} then ensures ${J
    \models \varphi\sigma}$, and so ${I \models \varphi\sigma}$ holds as well.
    Thus, ${\gamma \in \ICFPStep{\Prog}{I}}$ holds for each fact $\gamma$ with
    ${\set{ \gamma } \sqsubseteq \set{ A(\mathbf a,s\sigma) }}$, so we have
    ${\alpha \in \ICFPStep{\Prog}{I}}$.
\end{itemize}

\smallskip

(Claim~3) Consider an arbitrary $\tmax$ fact $\alpha$ of the form $A(\mathbf
a,\infty)$; the proof for a $\tmin$ fact is analogous. In the following, let
${S = \qset{A(\mathbf a,k)}{k\in\mathbb Z}}$.
\begin{itemize}
    \item Assume ${S \subseteq \ICFPStep{\Prog}{I}}$. Program $\Prog$ contains
    only finitely many rules, so the infinitely many facts of $S$ in
    $\ICFPStep{\Prog}{I}$ are produced by a rule ${r = \varphi \to A(\mathbf
    a,s) \in \Prog}$ and an infinite sequence $(\sigma_i)_{i \ge 0}$ of
    groundings of $r$ such that, for each $i$, we have ${I \models
    \varphi\sigma_i}$ and ${s\sigma_i < s\sigma_{i+1}}$. But then, ${J \models
    \varphi\sigma_i}$, so Proposition~\ref{prop:applicability-characterisation}
    ensures that $\sigma_i$ satisfies $\constraints{r}{J}$ for each ${i \geq
    0}$; therefore, ${\mathsf{opt}(r,J) = \infty}$ and ${\alpha \in
    \ILFPStep{\Prog}{J}}$ holds.

    \item Assume ${A(\mathbf a,\infty) \in \ILFPStep{\Prog}{J}}$. Then, a rule
    ${r = \varphi \to A(\mathbf a,s) \in \Prog}$ exists such that
    ${\mathsf{opt}(r,J) = \infty}$, so an infinite sequence $(\sigma_i)_{i\ge
    0}$ of solutions to $\constraints{r}{J}$ exists such that ${s\sigma_i <
    s\sigma_{i+1}}$ for each ${i \geq 0}$.
    Proposition~\ref{prop:applicability-characterisation} ensures ${J \models
    \varphi\sigma_i}$ for each ${i \geq 0}$, and so ${I \models
    \varphi\sigma_i}$ as well. Thus, for each ${k \in \mathbb Z}$, some ${i
    \geq 0}$ exists such that ${k \le s\sigma_i}$, and therefore we have
    ${\set{A(\mathbf a,k)} \sqsubseteq \set{A(\mathbf a,s\sigma_i)}}$;
    consequently, ${S \subseteq \ICFPStep{\Prog}{I}}$ holds. \qedhere
\end{itemize}
\end{proof}

\pseudomodelfixedpoint*

\begin{proof}
Immediate from Lemmas~\ref{lemma:operator:I} and~\ref{lemma:icfp-ilfp-coincide}.
\end{proof}

\fixpoint*

\begin{proof}
By inductively applying Lemma~\ref{lemma:icfp-ilfp-coincide}, for each ${n \geq
0}$, the limit-closed interpretation $\ICFPStepOp{\Prog}{n}$ clearly
corresponds to the pseudo-interpretation $\ILFPStepOp{\Prog}{n}$. Thus,
$\ICFPStepOp{\Prog}{\infty}$ and $\ILFPStepOp{\Prog}{\infty}$ also correspond
on all object and ordinary numeric facts. Now consider an arbitrary $n$-ary
$\max$ predicate $A$ and a tuple of $n-1$ objects ${\mathbf a}$, and for ${M =
\qset{k}{A(\mathbf a,k) \in \ICFPStepOp{\Prog}{\infty}}}$ consider the
following cases.
\begin{itemize}
    \item ${M = \emptyset}$. Then, for each ${n \geq 0}$ and each ${k \in
    \mathbb{Z}}$, we have ${A(\mathbf a,k) \not\in \ICFPStepOp{\Prog}{n}}$,
    which implies ${A(\mathbf a,k) \not\in \ILFPStepOp{\Prog}{n}}$ and
    ${A(\mathbf a,\infty) \not\in \ILFPStepOp{\Prog}{n}}$. Finally,
    $\ILFPStepOp{\Prog}{\infty}$ is the least (w.r.t.\ $\sqsubseteq$) fixpoint
    of $\ILFPStepOp{\Prog}{}$, so ${A(\mathbf a,k) \not\in
    \ILFPStepOp{\Prog}{\infty}}$ and ${A(\mathbf a,\infty) \not\in
    \ILFPStepOp{\Prog}{\infty}}$ holds as well.
    
    \item There exists ${\ell = \max M}$. Then, there exists ${n \geq 0}$ such
    that ${A(\mathbf a,\ell) \in \ICFPStepOp{\Prog}{n}}$, and ${A(\mathbf
    a,\ell') \not\in \ICFPStepOp{\Prog}{m}}$ for each ${\ell' > \ell}$ and ${m
    \geq 0}$; but then, ${A(\mathbf a,\ell) \in \ILFPStepOp{\Prog}{n}}$, and
    ${A(\mathbf a,\ell') \not\in \ILFPStepOp{\Prog}{m}}$ for each ${\ell' >
    \ell}$ and ${m \geq 0}$; finally, $\ILFPStepOp{\Prog}{\infty}$ is the least
    (w.r.t.\ $\sqsubseteq$) fixpoint of operator $\ILFPStepOp{\Prog}{}$, so
    ${A(\mathbf a,\ell) \in \ILFPStepOp{\Prog}{\infty}}$ holds.
    
    \item ${M = \mathbb{Z}}$. Then, for each ${k \in \mathbb{Z}}$, there exists
    ${n \geq 0}$ such that ${A(\mathbf a,k) \in \ICFPStepOp{\Prog}{n}}$, and so
    ${\ILFPStepOp{\Prog}{n} \models A(\mathbf a,k)}$ holds; but then,
    ${A(\mathbf a,\infty) \in \ILFPStepOp{\Prog}{n}}$ holds as well.
\end{itemize}
Analogous reasoning holds for $\tmin$ predicates, so
$\ICFPStepOp{\Prog}{\infty}$ corresponds to $\ILFPStepOp{\Prog}{\infty}$. But
then, the first and the third claim of this theorem follow straightforwardly
from Lemma~\ref{lemma:operator:I}. Moreover, pseudo-interpretations contain at
most one fact per combination of a limit predicate and a tuple of objects (of
corresponding arity), and program $\Prog$ is semi-ground so each rule in
$\Prog$ produces at most one fact in $\ILFPStepOp{\Prog}{\infty}$, which
implies the second claim of this theorem.
\end{proof}

\section{Proofs for Section~\ref{sec:decidability}}\label{sec:proofs-undecidability}

\typedfactentailmentundecidable*

\begin{proof}
We present a reduction from Hilbert's tenth problem, which is to determine
whether, given a polynomial ${P(x_1, \dots, x_n)}$ over variables ${x_1, \dots,
x_n}$, equation ${P(x_1, \dots, x_n) = 0}$ has integer solutions. It is well
known that the problem remains undecidable even if the solutions must be
nonnegative integers, so we use that variant in this proof. For each such
polynomial $P$, let $\Prog_P$ be the program containing
rules~\eqref{eq:factent:initA}--\eqref{eq:factent:diophantine} for $A$ a unary
$\tmin$ predicate and $B$ a nullary object predicate; it is obvious that
${\Prog_P \models B}$ if and only if ${P(x_1, \dots, x_n) = 0}$ has a
nonnegative integer solution.
\begin{align}
                                                                                                            & \to A(0)      \label{eq:factent:initA} \\
    \textstyle\bigwedge_{i=1}^n A(x_i) \land (P(x_1, \dots, x_n) \leq 0) \land (P(x_1, \dots, x_n) \geq 0)  & \to B         \label{eq:factent:diophantine}
\end{align}
Moreover, rule \eqref{eq:factent:diophantine} is applicable to ${J =
\set{A(0)}}$ if and only if ${P(x_1, \dots, x_n) = 0}$ has a nonnegative
integer solution.
\end{proof}

Although Presburger arithmetic does not have propositional variables, these can
clearly be axiomatised using numeric variables. Hence, in the rest of this
section we use propositional variables in Presburger formulas for the sake of
clarity.

\begin{definition}\label{def:Presburger-encoding}
    For each $n$-ary object predicate $A$, each $(n+1)$-ary ordinary numeric
    predicate $B$, each $(n+1)$-ary limit predicate $C$, each $n$-tuple of
    objects ${\mathbf a}$, and each integer $k$, let $\defined{A}{\mathbf a}$,
    $\defined{B}{\mathbf a k}$, $\defined{C}{\mathbf a}$ and $\fin{C}{\mathbf
    a}$ be distinct propositional variables, and let $\val{C}{\mathbf a}$ a
    distinct integer variable. Moreover, let $\preceq_C$ be $\le$ (resp.\
    $\ge$) if $C$ is a $\tmax$ (resp.\ $\tmin$) predicate.

    For $\Prog$ a semi-ground program, ${\mathsf{Pres}(\Prog) = \bigwedge_{r
    \in \Prog} \mathsf{Pres}(r)}$ is the Presburger formula where
    ${\mathsf{Pres}(r) = \forall \mathbf y.r'}$ for $\mathbf y$ all numeric
    variables in $r$ and $r'$ is obtained by replacing each atom $\alpha$ in
    $r$ with its encoding $\mathsf{Pres}(\alpha)$ defined as follows:
    \begin{itemize}
        \item ${\mathsf{Pres}(\alpha) = \alpha}$ if $\alpha$ is a comparison
        atom;

        \item ${\mathsf{Pres}(\alpha) = \defined{A}{\mathbf a}}$ if $\alpha$ is
        an object atom of the form $A(\mathbf a)$;

        \item ${\mathsf{Pres}(\alpha) = \defined{B}{\mathbf a k}}$ if $\alpha$
        is an ordinary numeric atom of the form $B(\mathbf a,k)$; and

        \item ${\mathsf{Pres}(\alpha) = \defined{C}{\mathbf a} \land (\neg
        \fin{C}{\mathbf a} \lor s \preceq_C \val{C}{\mathbf a})}$ if $\alpha$
        is a limit atom of the form ${C(\mathbf a,s)}$.
    \end{itemize}

    Let $J$ be a pseudo-interpretation, and let $\pass$ be an assignment of
    Boolean and integer variables. Then, $J$ corresponds to $\pass$ if all of
    the following conditions hold for all $A$, $B$, $C$, and ${\mathbf a}$ as
    specified above, for each integer ${k \in \mathbb{Z}}$:
    \begin{itemize}
        \item ${\pval{\defined{A}{\mathbf a}} = \true}$ if and only if
        ${A(\mathbf a) \in J}$;

        \item ${\pval{\defined{B}{\mathbf a k}} = \true}$ if and only if
        ${B(\mathbf a,k) \in J}$;

        \item ${\pval{\defined{C}{\mathbf a}} = \true}$ if and only if
        ${C(\mathbf a,\infty) \in J}$ or there exists ${\ell \in \mathbb{Z}}$
        such that ${C(\mathbf a,\ell) \in J}$;

        \item ${\pval{\fin{C}{\mathbf a}} = \true}$ and ${\pval{\val{C}{\mathbf
        a}} = k}$ if and only if ${C(\mathbf a,k) \in J}$.
    \end{itemize}
\end{definition}

Note that $k$ in Definition~\ref{def:Presburger-encoding} ranges over all
integers (which excludes $\infty$), ${\pval{\val{C}{\mathbf a}}}$ is an equal
to some integer $k$, and $J$ is a pseudo-interpretation and thus cannot contain
both ${C(\mathbf a,\infty)}$ and ${C(\mathbf a,k)}$; thus, ${C(\mathbf
a,\infty) \in J}$ implies ${\pval{\fin{C}{\mathbf a}} = \false}$.

Also note that each assignment $\pass$ corresponds to precisely one $J$;
however, each $J$ corresponds to infinitely many assignments $\pass$ since
Definition~\ref{def:Presburger-encoding} does not restrict the value of
variables other than $\defined{A}{\mathbf a}$, $\defined{B}{\mathbf a k}$,
$\defined{C}{\mathbf a}$, $\fin{C}{\mathbf a}$, and $\val{C}{\mathbf a}$.
Moreover, two assignments corresponding to the same pseudo-interpretation may
differ on the value of $\val{C}{\mathbf a}$ if $\fin{C}{\mathbf a}$ is set to
$\false$ in both assignments, and they can differ on the values of
$\fin{C}{\mathbf a}$ and $\val{C}{\mathbf a}$ if $\defined{C}{\mathbf a}$ is
set to $\false$ in both assignments.

\begin{lemma}\label{lemma:pseudo-presburger-correspondence}
    Let $J$ be a pseudo-interpretation and let $\pass$ be a variable assignment
    such that $J$ corresponds to $\pass$. Then,
    \begin{enumerate}
        \item ${J \models \alpha}$ if and only if ${\mu \models
        \mathsf{Pres}(\alpha)}$ for each ground atom $\alpha$, and
        
        \item ${J \models r}$ if and only if ${\mu \models \mathsf{Pres}(r)}$
        for each semi-ground rule $r$.
    \end{enumerate}
\end{lemma}

\begin{proof}
(Claim 1) We consider all possible forms of $\alpha$.
\begin{itemize}
    \item $\alpha$ is a comparison atom. Then, the truth of $\alpha$ is
    independent from $J$ so the claim is immediate.

    \item ${\alpha = A(\mathbf a)}$ is an object fact. Then,
    ${\mathsf{Pres}(\alpha) = \defined{A}{\mathbf a}}$, and
    ${\pval{\defined{A}{\mathbf a}} = \true}$ if and only if
    ${\defined{A}{\mathbf a} \in J}$, so the claim holds.

    \item ${\alpha = B(\mathbf a,k)}$ is an ordinary numeric fact. The proof is
    analogous to the case of object facts.

    \item ${\alpha = C(\mathbf a,k)}$ is a limit fact. If ${J \models \alpha}$,
    then either ${C(\mathbf a,\infty) \in J}$ or an integer $\ell$ exists such
    that ${C(\mathbf a,\ell) \in J}$ and ${k \preceq_C \ell}$; either way,
    ${\pval{\defined{C}{\mathbf a}} = \true}$ holds; moreover,
    ${\pval{\fin{C}{\mathbf a}} = \false}$ holds in the former and
    ${\pval{\val{C}{\mathbf a}} = \ell}$ holds in the latter case; thus, ${\mu
    \models \mathsf{Pres}(r)}$ clearly holds. The converse direction is
    analogous so we omit it for the sake of brevity.
\end{itemize}

\smallskip

(Claim 2) Let $r$ be an arbitrary semi-ground rule, and let $I$ be the
limit-closed interpretation corresponding to $J$. By definition, ${J \models
r}$ if and only if ${I \models r}$, and the latter is equivalent to ${I \models
r'}$ for each ground instance $r'$ of $r$ by the semantics of universal
quantification in first-order logic; but then, the latter claim is equivalent
to ${J \models r'}$ for each ground instance $r'$ of $r$.

Now note that, by construction, we have ${\mathsf{Pres}(\beta\sigma) =
\mathsf{Pres}(\beta)\sigma}$ for each semi-ground atom $\beta$ and each
grounding $\sigma$, and thus ${\mathsf{Pres}(r\sigma) =
\mathsf{Pres}(r)\sigma}$. Finally, groundings of $r$ can be equivalently seen
as variable assignments to universally quantified numeric variables in
$\mathsf{Pres}(r)$, so Claim~2 follows immediately from Claim 1.
\end{proof}

\presburgerencoding*

\begin{proof}
Lemma~\ref{lemma:pseudo-presburger-correspondence} immediately implies that
${\Prog \models \alpha}$ if and only if the sentence ${\varphi_0 = \forall
\mathbf x.\,\mathsf{Pres}(\alpha) \lor \neg\mathsf{Pres}(\Prog)}$ is valid,
where $\mathbf x$ contains all variables $\defined{A}{\mathbf a}$,
$\defined{B}{\mathbf a k}$, $\defined{C}{\mathbf a}$, $\fin{C}{\mathbf a}$, and
$\val{C}{\mathbf a}$ occurring in $\mathsf{Pres}(\Prog)$ or
$\mathsf{Pres}(\alpha)$. Clearly, $|\mathbf x|$ is polynomially bounded by
${\ssize{\Prog} + \ssize{\alpha}}$, and the magnitude of each integer in
$\varphi_0$ is bounded by the maximum magnitude of an integer in $\Prog$ and
$\alpha$. Let $\varphi_1$ be the sentence obtained from $\varphi_0$ by
converting each top-level conjunct of $\mathsf{Pres}(\Prog)$ into form $\forall
\mathbf y_i. \chi_i$ where $\chi_i$ is in CNF. Formulae $\varphi_0$ and
$\varphi_1$ are equivalent, and $\varphi_1$ is of the form
\begin{displaymath}
    \textstyle \varphi_1 = \forall \mathbf x.\,\mathsf{Pres}(\alpha) \lor \neg \bigwedge_{i=1}^n \forall \mathbf y_i.\bigwedge_{j=1}^{\ell_i} \chi_i^j,
\end{displaymath}
where ${n = |\Prog|}$ and, for each rule ${r_i \in \Prog}$, integer $\ell_i$ is
exponentially bounded by $\ssize{r_i}$, and $\ssize{\chi_i^j}$ and $|\mathbf
y_i|$ are linearly bounded by $\ssize{r_i}$. By moving all quantifiers to the
front of the formula and pushing negations inwards, we finally obtain formula
\begin{displaymath}
    \textstyle \varphi_2 = \forall \mathbf x \exists \mathbf y. \mathsf{Pres}(\alpha) \lor \bigvee_{i=1}^n \bigvee_{j=1}^{\ell_i} \psi_i^j,
\end{displaymath}
where ${\mathbf y = \bigcup_{i=1}^n \mathbf y_i}$ and each $\psi_i^j$ is the
negation-normal form of $\neg\chi_i^j$. Formula $\varphi_2$ is of the required
form, $|\mathbf y|$ is bounded polynomially by $\ssize{\Prog}$, number
${\sum_{i=1}^n \ell_i}$ is bounded polynomially by ${n = |\Prog|}$ and
exponentially by ${\max_{r \in \Prog} \ssize{r}}$, and $\ssize{\psi_i^j}$ is
bounded linearly by $\ssize{\Prog}$.
\end{proof}

\presburgervalidity*

\begin{proof}
Let ${m' = |\mathbf y|}$. Each $\psi_i$ can be seen as a system of linear
inequalities ${S_i = (A_i \mathbf x \le \mathbf c_i)}$ such that ${|\mathbf x|
\le \ell}$, ${|\mathbf c_i| \le k}$, and where the maximal magnitude of all
numbers in $A$ and ${\mathbf c_i}$ is bounded by $a^k$. By Proposition~3
of~\citeA{DBLP:conf/icalp/ChistikovH16} (adapted from the work by
\citeA{GathenSieveking78}), the set of solutions to $S_i$ can be represented by
a semi-linear set $\bigcup_{i'\in N_i} L(\mathbf b_{i'},P_{i'})$ where
${\mathbf b_{i'} \in \mathbb Z^\ell}$, ${P_{i'} \subseteq \mathbb Z^\ell}$,
${|N_i| \leq 2^\ell}$, and the magnitude of all integers in $\mathbf b_{i'}$
and $P_{i'}$ is bounded by ${2^{O(\ell\log\ell)} \cdot a^{k\ell}}$.
Consequently, disjunction ${\bigvee_{i=1}^n \psi_i}$ corresponds to a
semi-linear set ${\bigcup_{j\in M} L(\mathbf b_j,P_j)}$ where ${\mathbf b_j \in
\mathbb Z^{m+m'}}$, ${P_j \subseteq \mathbb Z^{m+m'}}$, ${|M| \leq n2^\ell}$,
and the magnitude of each integer in $\mathbf b_j$ and $P_j$ is still bounded
by $2^{O(\ell\log\ell)}\cdot a^{k\ell}$. Formula ${\exists \mathbf
y.\bigvee_{i=1}^n \psi_i}$ then corresponds to the projection of
${\bigcup_{j\in M} L(\mathbf b_j,P_j)}$ on the variables in $\mathbf x$, which
is a semi-linear set of the form ${\bigcup_{j\in M} L(\mathbf b'_j,P'_j)}$
where each ${\mathbf b'_j \in \mathbb Z^m}$ is a projection of $\mathbf b_j$ on
$\mathbf x$, and each ${P'_j \subseteq \mathbb Z^m}$ is a projection of $P_j$
on $\mathbf x$. Now, Theorem~21 by~\citeA{DBLP:conf/icalp/ChistikovH16} implies
that the satisfying assignments to the formula ${\varphi' = \neg \exists
\mathbf y.\bigvee_{i=1}^n\psi_i}$ can be represented as a semi-linear set
${\bigcup_{j'\in M'} L(\mathbf c_{j'},Q_{j'})}$ where the magnitude of each
integer in each $\mathbf c_{j'}$ and $Q_{j'}$ is bounded by
${b=(2^{O(\ell\log\ell)} \cdot a^{k\ell})^{n2^\ell\cdot O(m^4)}}$. Since
$\varphi'$ has a satisfying assignment if and only if it has a satisfying
assignment involving only numbers from some $\mathbf c_{j'}$, it follows that
$\varphi'$ is satisfiable if and only if it is satisfiable over models where
the absolute value of every integer variable is bounded by $b$. This implies
the claim of this lemma since $\varphi$ is valid if and only if $\varphi'$ is
unsatisfiable.
\end{proof}

\counterinterpretationsbounded*

\begin{proof}
The $\Leftarrow$ direction is trivial. For the $\Rightarrow$ direction assume
that ${\Prog\cup\Dat \not\models \alpha}$ holds, and let $\Dat'$ be obtained
from $\Dat$ by removing each fact that does not unify with an atom in $\Prog$
or $\alpha$; clearly, we have ${\Prog\cup\Dat' \not\models \alpha}$. Let
${\varphi = \forall \mathbf x \exists \mathbf y.\bigvee_{i=1}^n\psi_i}$ be the
Presburger sentence from Lemma~\ref{lem:presburger-encoding} for ${\Prog \cup
\Dat'}$ and $\alpha$. Sentence $\varphi$ is not valid and it satisfies the
following conditions.
\begin{itemize}
    \item Number $m = |\mathbf x|$ is polynomial in $\ssize{\Prog\cup\Dat'}$,
    which, in turn, is bounded by ${|\Prog\cup\Dat'| \cdot \max_{r \in
    \Prog\cup\Dat'} \ssize{r}}$. Moreover, $\Dat'$ contains only facts that
    unify with atoms in $\Prog$ and $\alpha$, so $m$ can be bounded further,
    namely linearly in the product $c s$, for $c = |\Prog|$ and $s=\max_{r \in
    \Prog} \ssize{\Prog}$.

    \item Number $n$ is linear in the product of $c$ and $2^s$.

    \item The size, and hence the number $\ell$ of variables in each $\psi_i$,
    are linear in $s$.
\end{itemize}
Let $a$ be the maximal magnitude of an integer in $\Prog\cup\Dat'$ (and thus in
$\varphi$ as well). By Lemma~\ref{lem:validity-presburger}, an assignment
$\pass$ exists such that ${\pass \not\models \varphi}$ and the magnitude of
each integer variable is bounded by $b = (2^{O(s\log s)} \cdot
a^{O(s^2)})^{O(c\cdot 2^s)\cdot 2^{O(s)}\cdot O((c s)^4)}$. Clearly, $b$ is
polynomial in $a$, exponential in $c$, and doubly-exponential in $s$, as
required. Moreover, clearly ${\pass \models \mathsf{Pres}(\Prog\cup\Dat')}$ and
${\pass \not\models \mathsf{Pres}(\alpha)}$. Now let $J'$ be the pseudo-model
corresponding to $\pass$; by
Lemma~\ref{lemma:pseudo-presburger-correspondence}, we have ${J' \models
\Prog\cup\Dat'}$ and ${J' \not\models \alpha}$. By construction, the magnitude
of each integer in $J'$ is bounded by $b$. Furthermore, let $J$ be the
restriction of $J'$ to the facts that unify with the head of at least one rule
in ${\Prog\cup\Dat'}$; clearly, we still have ${J \models \Prog\cup\Dat'}$ and
${J \not\models \alpha}$. Finally, ${|J| \leq |\Prog\cup\Dat'|}$ holds by our
construction, which implies our claim.
\end{proof}

\begin{lemma}\label{lem:ilfp-step-polynomial}
    For each semi-ground, limit-linear program $\Prog$, pseudo-interpretation
    $J$, and dataset $\Dat$, there exists a polynomial $p$ such that
    $\ILFPStep{\Prog\cup\Dat}{J}$ can be computed in nondeterministic
    polynomial time in ${\ssize{\Prog} + \ssize{\Dat} + \ssize{J}}$, and in
    deterministic polynomial time in ${\ssize{\Prog} + \ssize{\Dat} +
    \ssize{J}}^{p(\max_{r\in\Prog}\ssize r)}$.
\end{lemma}

\begin{proof}
Let ${S = \qset{\mathsf{hd}(r,J)}{\text{rule } r \in \Prog\cup\Dat \text{ is
applicable to } J }}$. Program $\Prog$ is semi-ground, and therefore ${\Prog
\cup \Dat}$ is semi-ground as well; thus, each rule of ${\Prog \cup \Dat}$ can
contribute at most one fact to $S$, so we have ${|S| \leq |\Prog|+|\Dat|}$. By
Definition~\ref{def:rule-app}, $\ILFPStep{\Prog\cup\Dat}{J}$ is the smallest
(w.r.t.\ $\sqsubseteq$) pseudo-interpretation such that
${\ILFPStep{\Prog\cup\Dat}{J} \models S}$, so we can compute
$\ILFPStep{\Prog\cup\Dat}{J}$ as the set containing each object and ordinary
numeric fact in $S$, each fact ${A(\mathbf a,\infty) \in S}$ for $A$ a limit
predicate, each fact ${A(\mathbf a,\ell) \in S}$ such that $A$ is a $\tmin$
(resp.\ $\tmax$) predicate and ${A(\mathbf a,k) \in S}$ implies ${k \neq
\infty}$ and ${k \geq \ell}$ (resp.\ ${k \leq \ell}$). To complete the proof of
this lemma, we next argue that set $S$ can be computed within the required time
bounds.

Consider an arbitrary rule ${r \in \Prog\cup\Dat}$, and let $J'$ be the subset
of $J$ containing all facts that unify with a body atom in $r$; note that
${|J'| \le \ssize{r}}$. Rule $r$ is applicable to $J$ if and only if
conjunction $\constraints{r}{J'}$ has an integer solution. By construction,
$\ssize{\constraints{r}{J'}}$ is linear in ${\ssize{r} + \ssize{J'}}$, the
number of variables in $\constraints{r}{J'}$ and $r$ is the same,
$|\constraints{r}{J'}|$ is linear in $\ssize{r}$, and the magnitude of each
integer in $\constraints{r}{J'}$ is exponentially bounded in ${\ssize{r} +
\ssize{J}}$. But then, checking whether $\constraints{r}{J}$ has an integer
solution is in $\textsc{NP}$ w.r.t.\ ${\ssize{r} + \ssize{J}}$, and in
$\textsc{PTime}$ w.r.t.\ $\ssize{J}^{p(\ssize{r})}$ for some polynomial $p$, as
we argue next.
\begin{itemize}
    \item We first consider the former claim. Let $a$ be the maximal magnitude
    of an integer in $\constraints{r}{J'}$. Conjunction $\constraints{r}{J'}$
    contains only the numbers from $r$ and $J'$, whose magnitude is at most
    $2^{\ssize{r}}$ and $2^{\ssize{J}}$, respectively; thus, we have ${a \leq
    2^{\ssize{r} + \ssize{J}}}$. Moreover, the results by
    \citeA{DBLP:journals/jacm/Papadimitriou81} show that there exists a
    polynomial $p_1$ such that the magnitude of an integer in a solution to
    $\constraints{r}{J'}$ can be bounded by ${b = a^{p_1(\ssize{r})}}$, and so
    there exists a polynomial $p_2$ such that ${b \leq
    2^{p_2(\ssize{r}+\ssize{J})}}$. The binary representation of $b$ thus
    requires at most ${\ssize{r} + \ssize{J}}$ bits, and so we can guess it in
    polynomial time.

    \item We next consider the latter claim. By Theorem 5.4
    of~\citeA{DBLP:journals/mor/Kannan87}, checking satisfiability of
    $\constraints{r}{J}$ over $\mathbb Z$ is fixed-parameter tractable in the
    number $n$ of variables in $r$--that is, there exists a polynomial $p_3$
    such that a solution to $\constraints{r}{J}$ can be computed in time
    ${O((\ssize{r}+\ssize{J}) \cdot 2^{p_3(n)})}$. Since ${n \le \ssize{r}}$
    clearly holds, there exists a polynomial $p_4$ such that the satisfiability
    of $\constraints{r}{J}$ can be checked in time that is thus
    $\ssize{J}^{p_4(\ssize{r})}$.
\end{itemize}

Now assume that $r$ is applicable to $J$. Then ${\mathsf{hd}(r,J) = \head{r}}$
if $\head{r}$ is an object atom, so we assume that ${\head{r} = A(\mathbf
a,s)}$ is a limit atom and argue that $\mathsf{opt}(r,J)$ can be computed within
the required time bounds using the following two steps.
\begin{enumerate}
    \item Depending on whether $A$ is a $\tmin$ or a $\tmax$ predicate, we
    check whether there is a smallest/largest value for $s$ in all solutions to
    $\constraints{r}{J}$---that is, we check whether the integer linear program
    `minimise/maximise $s$ subject to $\constraints{r}{J}$' is bounded.
    \citeA{ByrdGH87} showed that this amounts to checking boundedness of the
    corresponding linear relaxation, which in turn can be reduced to checking
    linear feasibility and can be solved in deterministic polynomial time in
    ${\ssize{r}+\ssize{J}}$.

    \item If the above problem is bounded, we compute its optimal solution,
    which can be reduced to polynomially many (in $\ssize{r} + \ssize{J}$)
    feasibility checks, as shown by \citeA{DBLP:journals/jacm/Papadimitriou81}
    (Corollary~2 with binary search). Each such feasibility check is in
    \textsc{NP} w.r.t.\ $\ssize{r}+\ssize{J}$, and in \textsc{PTime} w.r.t.\
    $\ssize{J}^{p(\ssize{r})}$.
\end{enumerate}
Thus, $\mathsf{hd}(r,J)$ can be computed in nondeterministic polynomial time in
${\ssize{r}+\ssize{J}}$, and in deterministic polynomial time in
$\ssize{J}^{p(\ssize{r})}$, which implies our claim.
\end{proof}

\begin{lemma}\label{lem:adm-fact-ent-conp}
    Deciding ${\Prog \models \alpha}$ is \textsc{coNP}-hard in data complexity
    for $\Prog$ a limit-linear program and $\alpha$ a fact.
\end{lemma}

\begin{proof}
An instance $\mathcal{T}$ of the \emph{square tiling} problem is given by an
integer $N$ coded in unary, a set ${T = \set{t_0,\dots,t_{M-1}}}$ of $M$ tiles,
and two compatibility relations ${H \subseteq T \times T}$ and ${V \subseteq T
\times T}$. The problem is to determine whether there exists a tiling ${\tau :
\{ 0, \dots, N-1 \}^2 \rightarrow T}$ of an ${N \times N}$ square such that
${\langle \tau(i,j),\tau(i+1,j) \rangle \in H}$ holds for all ${0 \leq i <
N-1}$ and ${0 \leq j < N}$, and ${\langle \tau(i,j),\tau(i,j+1) \rangle \in V}$
holds for all ${0 \leq i < N}$ and ${0 \leq j < N-1}$, which is known to be
\textsc{NP}-complete. Thus, to prove the claim of this lemma, we reduce the
complement of the problem by presenting a fixed program $\Prog_\mathit{tiling}$
and a dataset $\Dat_\mathcal{T}$ (that depends on $\mathcal{T}$), and showing
that $\mathcal{T}$ has \emph{no} solution if and only if
${\Prog_\mathit{tiling} \cup \Dat_\mathcal{T} \models \mathit{noSolution}}$.

Our encoding uses object EDB predicates $\mathit{succ}$,
$\mathit{incompatibleH}$, and $\mathit{incompatibleV}$; ordinary numeric EDB
predicates $\mathit{shift}$, $\mathit{tileNo}$, $\mathit{numTiles}$, and
$\mathit{maxTiling}$; nullary object IDB predicate $\mathit{noSolution}$; unary
$\tmin$ IDB predicate $I$; and unary $\tmax$ IDB predicate
$\mathit{tiling}$. Program $\Prog_\mathit{tiling}$ contains rules
\eqref{eq:dhard:number}--\eqref{eq:dhard:ass:noSol}, where ${(s \doteq t)}$
abbreviates ${(s \le t) \land (t \le s)}$.
\begin{align}
                                                                                                                        & \to I(0)                  \label{eq:dhard:number}\\
                                                                                                                        & \to \mathit{tiling}(0)    \label{eq:dhard:ass:init}\\
    \begin{array}{@{}r@{}}
        \mathit{tiling}(n) \land \mathit{numTiles}(\mathit{nt}) \\
        \mathit{shift}(x,y,s) \land \mathit{tileNo}(u,t) \land I(m_1) \land I(m_2) \land (n \doteq m_1 \times \mathit{nt} \times s + t \times s + m_2) \land (m_2 < s) \\
        \mathit{succ}(x,x') \\
        \mathit{shift}(x',y,s') \land \mathit{tileNo}(u',t') \land I(m_1') \land I(m_2') \land (n \doteq m_1' \times \mathit{nt} \times s' + t' \times s' + m_2') \land (m_2' < s') \\
        \mathit{incompatibleH}(u,u') \\
    \end{array}                                                                                                         &
    \begin{array}{@{\;}l@{}}
        \land\\
        \land\\
        \land\\
        \land\\
        \to \mathit{tiling}(n+1)\\
    \end{array}                                                                                                                                     \label{eq:dhard:ass:H}\\
    \begin{array}{@{}r@{}}
        \mathit{tiling}(n) \land \mathit{numTiles}(\mathit{nt}) \\
        \mathit{shift}(x,y,s) \land \mathit{tileNo}(u,t) \land I(m_1) \land I(m_2) \land (n \doteq m_1 \times \mathit{nt} \times s + t \times s + m_2) \land (m_2 < s) \\
        \mathit{succ}(y,y') \\
        \mathit{shift}(x,y',s') \land \mathit{tileNo}(u',t') \land I(m_1') \land I(m_2') \land (n \doteq m_1' \times \mathit{nt} \times s' + t' \times s' + m_2') \land (m_2' < s') \\
        \mathit{incompatibleV}(u,u') \\
    \end{array}                                                                                                         &
    \begin{array}{@{\;}l@{}}
        \land\\
        \land\\
        \land\\
        \land\\
        \to \mathit{tiling}(n+1)\\
    \end{array}                                                                                                                                     \label{eq:dhard:ass:V}\\
    \mathit{tiling}(n) \land \mathit{maxTiling}(m) \land (m<n)                                                         & \to \mathit{noSolution}    \label{eq:dhard:ass:noSol}
\end{align}
Dataset $\Dat_\mathcal{T}$ contains facts
\eqref{eq:dhard:numTiles}--\eqref{eq:dhard:shift}, where ${g_0, \dots,
g_{N-1}}$ are fresh objects, and $t_i$ for ${0 \leq i < M}$ are distinct
objects corresponding to the tiles in $T$. Since $N$ is coded in unary,
although numbers $M^{N^2}-1$ and $M^{i + N j}$ are exponential in $N$, they can
be computed in polynomial time and represented using polynomially many bits.
\begin{align}
    & \to\mathit{numTiles}(M)                                                                                           \label{eq:dhard:numTiles}\\
    & \to\mathit{tileNo}(t_i,i)                 && \text{for each } 0 \le i < M                                         \label{eq:dhard:tileNo}\\
    & \to\mathit{incompatibleH}(t_i,t_j)        && \text{for each } 0 \le i,j < M \text{ such that } (t_i,t_j) \notin H \label{eq:dhard:incH}\\
    & \to\mathit{incompatibleV}(t_i,t_j)        && \text{for each } 0 \le i,j < M \text{ such that } (t_i,t_j) \notin V \label{eq:dhard:incV}\\
    & \to\mathit{maxTiling}(M^{N^2}-1)                                                                                  \label{eq:dhard:maxTiling}\\
    & \to\mathit{succ}(g_i,g_{i+1})             && \text{for each } 0 \le i < N-1                                       \label{eq:dhard:succ}\\
    & \to\mathit{shift}(g_i,g_j,M^{i + N j})    && \text{for each } 0 \le i,j < N                                       \label{eq:dhard:shift}
\end{align}

Our reduction uses the following idea. Facts \eqref{eq:dhard:tileNo} associate
each tile $t_i$ with an integer $i$ where ${0 \le i < M}$; hence, in the rest
of this discussion, we do not distinguish a tile from its number. This allows
us to represent each tiling $\tau$ using a number ${\sum_{0 \le i,j < N} \,
\tau(i,j) \times M^{i+Nj}}$. Thus, given a number $n$ that encodes a tiling,
number $t$ with ${0 \leq t < M}$ corresponds to the tile assigned to position
$(i,j)$ if ${n = m_1 \times M \times M^{i+Nj} + t \times M^{i+Nj} + m_2}$ for
some integers $m_1$ and $m_2$ where ${0 \leq m_2 < M^{i+Nj}}$. Thus, if numeric
variable $n$ is assigned such an encoding of a tiling and numeric variable $s$
is assigned the factor $M^{i+Nj}$ corresponding to a position $(i,j)$, then
conjunction
\begin{displaymath}
    \mathit{tileNo}(u,t) \land I(m_1) \land I(m_2) \land (n \doteq m_1 \times \mathit{nt} \times s + t \times s + m_2) \land (m_2 < s)
\end{displaymath}
is true if and only if $u$ is assigned the tile object corresponding to
position $(i,j)$ in the tiling encoded by $n$. To complete the construction, we
represent each position $(i,j)$ by a pair of objects $(g_i,g_j)$, each of which
is associated with the corresponding factor $M^{i + N j}$ using facts
\eqref{eq:dhard:shift}. Facts \eqref{eq:dhard:succ} provide an ordering on
$g_i$, which allows us to identify adjacent positions. Finally, fact
\eqref{eq:dhard:maxTiling} records the maximal number that encodes a tiling as
we outlined earlier. Program $\Prog_\mathit{tiling}$ then simply checks through
all tilings: rule \eqref{eq:dhard:ass:init} ensures that the tiling encoded as
$0$ is checked; moreover, for each $n$ such that $\mathit{tiling}(n)$ holds,
rules \eqref{eq:dhard:ass:H} and \eqref{eq:dhard:ass:V} derive
$\mathit{tiling}(n+1)$ if either the horizontal or the vertical compatibility
requirement is violated for the tiling encoded by $n$. Finally, rule
\eqref{eq:dhard:ass:noSol} detects that no solution exists if
$\mathit{tiling}(M^{N^2})$ is derived.
\end{proof}

\begin{lemma}\label{lem:adm-fact-ent-conexptime}
    Deciding ${\Prog \models \alpha}$ is \textsc{coNExpTime}-hard for $\Prog$ a
    limit-linear program and $\alpha$ a fact.
\end{lemma}

\begin{proof}
We present a reduction from the \emph{succinct square tiling} problem. An
instance $\mathcal{T}$ of the problem is given by an integer $N$ coded in
unary, a set $T$ containing $M$ tiles, and horizontal and vertical
compatibility relations $H$ and $V$, respectively, as in the proof of
Lemma~\ref{lem:adm-fact-ent-conp}; however, the objective is to tile a square
of ${2^N \times 2^N}$ positions, which is known to be
\textsc{NExpTime}-complete. Thus, to prove the claim of this lemma, we reduce
the complement of the problem by presenting a program $\Prog_\mathcal{T}$ and
showing that $\mathcal{T}$ has \emph{no} solution if and only if
${\Prog_{\mathcal{T}} \models \mathit{noSolution}}$.

The main idea behind our reduction is similar to
Lemma~\ref{lem:adm-fact-ent-conp}. Program $\Prog_{\mathcal{T}}$ contains rules
\eqref{eq:chard:tileNo}--\eqref{eq:chard:incV} that associate each tile with a
number using an ordinary numeric predicate $\mathit{tileNo}$, and encode the
horizontal and vertical incompatibility relations using the object predicates
$\mathit{incompatibleH}$ and $\mathit{incompatibleV}$.
\begin{align}
    & \to\mathit{tileNo}(t_i,i)             && \text{for each } 0 \le i < M                                         \label{eq:chard:tileNo}\\
    & \to\mathit{incompatibleH}(t_i,t_j)    && \text{for each } 0 \le i,j < M \text{ such that } (t_i,t_j) \notin H \label{eq:chard:incH}\\
    & \to\mathit{incompatibleV}(t_i,t_j)    && \text{for each } 0 \le i,j < M \text{ such that } (t_i,t_j) \notin V \label{eq:chard:incV}
\end{align}

The main difference to Lemma~\ref{lem:adm-fact-ent-conp} is that, in order to
obtain a polynomial encoding, we cannot represent a position $(i,j)$ in the
grid explicitly using a pair of objects. Instead, we encode each position using
a pair $(\mathbf i,\mathbf j)$ where ${\mathbf i}$ and ${\mathbf j}$ are
$N$-tuples of objects $\Zero$ and $\One$. If we read $\Zero$ and $\One$ as
representing numbers $0$ and $1$, respectively, then each ${\mathbf i}$ and
${\mathbf j}$ can be seen as a binary number in $[0,2^N-1]$. By a slight abuse
of notation, we often identify a tuple over $\Zero$ and $\One$ with the number
it encodes and use tuples in arithmetic expressions. While positions can be
encoded using $N$ bits, we will also need to ensure distance between positions,
which requires $N+1$ bits. In the rest of this proof, $\mathbf\Zero$ and
$\mathbf\One$ stand for tuples ${\Zero,\dots,\Zero}$ and ${\One,\dots,\One}$,
respectively, whose length is often implicit from the context where these
tuples occur. Similarly, $\mathbf x$, $\mathbf x'$, $\mathbf y$, and $\mathbf
y'$ are tuples of distinct variables whose length will also be clear from the
context.

To axiomatise an ordering on numbers with $N$ bits, program
$\Prog_{\mathcal{T}}$ contains rules
\eqref{eq:chard:bool:0}--\eqref{eq:chard:succprime}, where $B$ is a unary
object predicate, $\mathit{succ}$ is a $2N$-ary object predicate, and
$\mathit{succ'}$ is a $(2N+2)$-ary object predicate. Rules
\eqref{eq:chard:bool:0}--\eqref{eq:chard:succ} ensure ${\Prog_{\mathcal{T}}
\models \mathit{succ}(\mathbf i,\mathbf j})$ where ${\mathbf i}$ and ${\mathbf
j}$ encode numbers with $N$ bits such that ${\mathbf j = 1 + \mathbf i}$; in
particular, rule \eqref{eq:chard:succ} encodes binary incrementation, where
${{\mathbf x}\One{\mathbf \Zero} = 1 + {\mathbf x}\One{\mathbf \Zero}}$ holds
for each position $k$ and each $k$-tuple of zeros and ones ${\mathbf x}$. Rules
\eqref{eq:chard:bool:0}--\eqref{eq:chard:bool:1} and \eqref{eq:chard:succprime}
ensure an analogous property for $\mathit{succ}'$, but for numbers with $N+1$
bits.
\begin{align}
                                                & \to \mathit{B}(\Zero) \label{eq:chard:bool:0}\\
                                                & \to \mathit{B}(\One)  \label{eq:chard:bool:1}\\
    \textstyle\bigwedge_{i=1}^k \mathit{B}(x_i) & \to \mathit{succ}(\mathbf x,\Zero,\mathbf \One,\mathbf x,\One,\mathbf \Zero)    && \text{for each } 0 \le k < N \text{ where }    |\mathbf x| = k \text{ and } |\mathbf \One| = |\mathbf \Zero| = N-k-1   \label{eq:chard:succ}\\
    \textstyle\bigwedge_{i=1}^k \mathit{B}(x_i) & \to \mathit{succ'}(\mathbf x,\Zero,\mathbf \One,\mathbf x,\One,\mathbf \Zero)   && \text{for each } 0 \le k < N+1 \text{ where }  |\mathbf x| = k \text{ and } |\mathbf \One| = |\mathbf \Zero| = N-k     \label{eq:chard:succprime}
\end{align}

Analogously to the proof of Lemma~\ref{lem:adm-fact-ent-conp}, we encoded
tilings using numbers in $[0,M^{2^{2N}}-1]$. To compute the maximum number
encoding a tiling, program $\Prog_{\mathcal{T}}$ contains rules
\eqref{eq:chard:mT:1}--\eqref{eq:chard:maxTiling}, where $\mathit{maxTiling}$
is a unary $\tmin$ predicate, and $\mathit{auxT}$ is a $(2N+1)$-ary $\tmin$
predicate. Auxiliary rules \eqref{eq:chard:mT:1}--\eqref{eq:chard:mT:3}
multiply $M$ with itself as many times as there are grid positions, so we have
$\Prog_{\mathcal{T}}\models{auxT}(\mathbf i,\mathbf j,M^{1+\mathbf i+{2^N}
\cdot \mathbf j})$ for each position ${(\mathbf i,\mathbf j)}$. Consequently,
rule \eqref{eq:chard:maxTiling} ensures that, for all $s$, we have
${\Prog_{\mathcal{T}} \models \mathit{maxTiling}(s)}$ if and only if ${s \ge
M^{2^{2N}}-1}$.
\begin{align}
                                                                                        & \to \mathit{auxT}(\mathbf \Zero,\mathbf \Zero,M)          \label{eq:chard:mT:1} \\
    \mathit{auxT}(\mathbf x,\mathbf y,n) \land \mathit{succ}(\mathbf x,\mathbf x')      & \to \mathit{auxT}(\mathbf x',\mathbf y,M \times n)        \label{eq:chard:mT:2} \\
    \mathit{auxT}(\mathbf \One,\mathbf y,n) \land \mathit{succ}(\mathbf y,\mathbf y')   & \to \mathit{auxT}(\mathbf \Zero,\mathbf y',M \times n)    \label{eq:chard:mT:3} \\
    \mathit{auxT}(\mathbf \One,\mathbf \One,n)                                          & \to \mathit{maxTiling}(n-1)                               \label{eq:chard:maxTiling}
\end{align}

Unlike in the proof of Lemma~\ref{lem:adm-fact-ent-conp}, we cannot include
shift factors explicitly into $\Prog_{\mathcal{T}}$ since this would make the
encoding exponential; moreover, we could precompute shift factors using rules
similar to \eqref{eq:chard:mT:1}--\eqref{eq:chard:mT:3}, but then we would need
to use values from limit predicates in multiplication, which would not produce
a limit-linear program. Therefore, we check tilings using a different approach.
As in the proof of Lemma~\ref{lem:adm-fact-ent-conp}, our construction ensures
that, for all $s$, we have ${\Prog_{\mathcal{T}} \models \mathit{tiling}(s)}$
if and only if each tiling $n$ with ${0 \leq n \leq s}$ does not satisfy the
compatibility relations. Given a tiling encoded by $n$ and a position
${(\mathbf i,\mathbf j)}$, let
\begin{displaymath}
    s_{n,\mathbf i,\mathbf j} = \left\lfloor \frac{n}{M^{{\mathbf i} + {\mathbf j} \cdot 2^N}} \right\rfloor.
\end{displaymath}
Program $\Prog_{\mathcal{T}}$ contains rules
\eqref{eq:chard:I}--\eqref{eq:chard:shiftedTiling:3} where
$\mathit{shiftedTiling}$ is a $\tmax$ predicate of arity $2N+1$ and $I$ is a
unary $\tmin$ predicate. These rules ensure that, for each $\mathbf i$,
$\mathbf j$, and tiling $n$ such that ${\Prog_{\mathcal{T}} \models
\mathit{tiling}(n)}$, we have ${\Prog_{\mathcal{T}} \models
\mathit{shiftedTiling}(\mathbf i,\mathbf j,s_{n,\mathbf i,\mathbf j})}$. To
understand how this is achieved, we order the grid positions as follows:
\begin{displaymath}
    (0,0), (1,0), \dots, (2^N-1,0), (0,1), (1,1), \dots, (2^N-1,1), \dots\dots, (0,2^N-1), (1,2^N-1), \dots, (2^N-1,2^N-1)
\end{displaymath}
Now consider an arbitrary position ${(\mathbf i, \mathbf j)}$ and its successor
${(\mathbf i', \mathbf j')}$ in the ordering. The encoding of a tiling using an
integer $n$ ensures ${s_{n,\mathbf i,\mathbf j} = M \times s_{n,\mathbf
i',\mathbf j'} + t}$ holds, where ${0 \leq t < M}$ is the number of the tile
that $n$ assigns to position ${(\mathbf i, \mathbf j)}$. Thus, rule
\eqref{eq:chard:shiftedTiling:1} ensures that position ${(\mathbf \Zero,\mathbf
\Zero)}$ satisfies the mentioned property, rule
\eqref{eq:chard:shiftedTiling:2} handles adjacent positions of the form
${(\mathbf i,\mathbf j)}$ and ${(\mathbf i+1,\mathbf j)}$, and rule
\eqref{eq:chard:shiftedTiling:3} handles adjacent positions of the form
${(\mathbf \One,\mathbf j)}$ and ${(\mathbf \Zero,\mathbf j+1)}$.
\begin{align}
                                                                                                                                                                                                & \to \mathit{I}(0)                                         \label{eq:chard:I} \\
    \mathit{tiling}(n)                                                                                                                                                                          & \to \mathit{shiftedTiling}(\mathbf \Zero,\mathbf \Zero,n) \label{eq:chard:shiftedTiling:1} \\
    \mathit{shiftedTiling}(\mathbf x,\mathbf y,n) \land \mathit{succ}(\mathbf x,\mathbf x') \land \mathit{I}(\ell) \land \mathit{I}(m) \land (\ell < M) \land (n \doteq M \times m + \ell)      & \to \mathit{shiftedTiling}(\mathbf x',\mathbf y,m)        \label{eq:chard:shiftedTiling:2} \\
    \mathit{shiftedTiling}(\mathbf\One,\mathbf y,n) \land \mathit{succ}(\mathbf y,\mathbf y') \land \mathit{I}(\ell) \land \mathit{I}(m) \land (\ell < M) \land (n \doteq M \times m + \ell)    & \to \mathit{shiftedTiling}(\mathbf \Zero,\mathbf y',m)    \label{eq:chard:shiftedTiling:3}
\end{align}
Note that, for all $n$ and $n'$ with $n < n'$ and each position ${(\mathbf
i,\mathbf j)}$, we have $s_{n,\mathbf i,\mathbf j} < s_{n',\mathbf i,\mathbf
j}$. Thus, since $\mathit{shiftedTiling}$ is a $\tmax$ predicate, the limit
value for $s$ in ${\mathit{shiftedTiling}(\mathbf i,\mathbf j,s)}$ will always
correspond to the limit value for $n$ in $\mathit{tiling}(n)$.

Checking horizontal compatibility is now easy, but checking vertical
compatibility requires dividing $s_{n,\mathbf i,\mathbf j}$ by $M^{2^N}$, which
would make the reduction exponential. Hence, $\Prog_{\mathcal{T}}$ checks
compatibility using rules
\eqref{eq:chard:conflict:H}--\eqref{eq:chard:conflict:propV}, where
$\mathit{conflict}$ is a $\tmax$ predicate of arity $3N+3$. These rules ensure
that, for each $\mathbf i$, $\mathbf j$, $\mathbf d$, $u$, and tiling $n$ such
that ${\Prog_{\mathcal{T}} \models \mathit{tiling}(n)}$ and the position that
precedes ${(\mathbf i,\mathbf j)}$ by distance $\mathbf d$ in the ordering
cannot be labelled in $n$ with tile $u$, we have ${\Prog_{\mathcal{T}} \models
\mathit{conflict}(\mathbf i,\mathbf j,\mathbf d,u,s_{n,\mathbf i,\mathbf j})}$.
To this end, assume that $(\mathbf x,\mathbf y)$ is labelled with tile $u'$;
now if ${(u,u') \not\in H}$ and ${\mathbf x \neq \mathbf \Zero}$ (i.e., the
predecessor $\mathbf x'$ of $\mathbf x$ exists), then rule
\eqref{eq:chard:conflict:H} says that the position preceding $(\mathbf
x,\mathbf y)$ by $\mathbf \Zero\One$ (i.e., the position to the left) cannot be
labelled with $u$; moreover, if ${(u,u') \not\in V}$ and ${\mathbf y \neq
\mathbf \Zero}$ (i.e., the predecessor $\mathbf y'$ of $\mathbf y$ exists),
then rule \eqref{eq:chard:conflict:V} says that the position preceding
$(\mathbf x,\mathbf y)$ by $\mathbf \One\Zero = 2^N$ (i.e., the position above)
cannot be labelled with $u$. Moreover, rule \eqref{eq:chard:conflict:propH}
propagates such constraints from position $(\mathbf i,\mathbf j)$ to $(\mathbf
i-1,\mathbf j)$ while reducing the distance by one, and rule
\eqref{eq:chard:conflict:propV} does so for positions $(\mathbf \Zero,\mathbf
j)$ and $(\mathbf \One,\mathbf j-1)$.
\begin{align}    
    \begin{array}{@{}r@{}}
        \mathit{shiftedTiling}(\mathbf x,\mathbf y,m) \land \mathit{succ}(\mathbf x',\mathbf x) \land \mathit{incompatibleH}(u,u') \land \mathit{tileNo}(u',t') \\
        \mathit{I}(\ell) \land (m \doteq M \times \ell + t') \\
    \end{array}                                                                                                                                                                                                 &
    \begin{array}{@{\;}lr@{}}
        \land \\
        \to \mathit{conflict}(\mathbf x,\mathbf y,\mathbf \Zero\One,u,m) \\
    \end{array}                                                                                                                                                                                                 \label{eq:chard:conflict:H}\\
    \begin{array}{@{}r@{}}
        \mathit{shiftedTiling}(\mathbf x,\mathbf y,m)  \land \mathit{succ}(\mathbf y',\mathbf y) \land \mathit{incompatibleV}(u,u') \land \mathit{tileNo}(u',t') \\
        \mathit{I}(\ell) \land (m \doteq M \times \ell + t') \\
    \end{array}                                                                                                                                                                                                 &
    \begin{array}{@{\;}lr@{}}
        \land \\
        \to \mathit{conflict}(\mathbf x,\mathbf y,\One\mathbf \Zero,u,m)
    \end{array}                                                                                                                                                                                                 \label{eq:chard:conflict:V}\\
    \begin{array}{@{}r@{}}
        \mathit{shiftedTiling}(\mathbf x,\mathbf y,m) \land \mathit{succ}(\mathbf x,\mathbf x') \land \mathit{conflict}(\mathbf x',\mathbf y,\mathbf z',u,m') \land \mathit{succ'}(\mathbf z,\mathbf z') \\
        \mathit{I}(\ell) \land (\ell < M) \land (m \doteq M \times m' + \ell) \\
    \end{array}                                                                                                                                                                                                 &
    \begin{array}{@{\;}lr@{}}
        \land \\
        \to \mathit{conflict}(\mathbf x,\mathbf y,\mathbf z,u,m) \\
    \end{array}                                                                                                                                                                                                 \label{eq:chard:conflict:propH}\\
    \begin{array}{@{}r@{}}
        \mathit{shiftedTiling}(\mathbf \One,\mathbf y,m) \land \mathit{succ}(\mathbf y,\mathbf y') \land \mathit{conflict}(\mathbf \Zero,\mathbf y',\mathbf z',u,m') \land \mathit{succ'}(\mathbf z,\mathbf z') \\
        \mathit{I}(\ell) \land (\ell < M) \land (m \doteq M \times m' + \ell) \\
    \end{array}                                                                                                                                                                                                 &
    \begin{array}{@{\;}lr@{}}
        \land \\
        \to \mathit{conflict}(\mathbf \One,\mathbf y,\mathbf z,u,m) \\
    \end{array}                                                                                                                                                                                                 \label{eq:chard:conflict:propV}
\end{align}

Program $\Prog_{\mathcal{T}}$ also contains rules
\eqref{eq:chard:invalid:detect}--\eqref{eq:chard:invalid:propV} where
$\mathit{invalid}$ is a $(2N+1)$-ary $\tmax$ predicate. These rules ensure
that, for each $\mathbf i$, $\mathbf j$, and each tiling $n$ such that
${\Prog_{\mathcal{T}} \models \mathit{tiling}(n)}$ and there exists a position
$(\mathbf i',\mathbf j')$ that comes after $(\mathbf i,\mathbf j)$ in the
position order such that $n$ does not satisfy the compatibility relations
between $(\mathbf i',\mathbf j')$ and its horizontal or vertical successor, we
have ${\Prog_{\mathcal{T}} \models \mathit{invalid}(\mathbf i,\mathbf j,
s_{n,\mathbf i,\mathbf j})}$. Rule \eqref{eq:chard:invalid:detect} determines
invalidity at position $(\mathbf x,\mathbf y)$ for conflicts with zero
distance, and rules \eqref{eq:chard:invalid:propH} and
\eqref{eq:chard:invalid:propV} propagate this information to preceding
positions analogously to rules \eqref{eq:chard:conflict:propH} and
\eqref{eq:chard:conflict:propV}.
\begin{align}
    \mathit{conflict}(\mathbf x,\mathbf y,\mathbf \Zero,u,m) \land \mathit{tileNo}(u,t) \land \mathit{I}(\ell) \land (m \doteq M \times \ell + t)       & \to \mathit{invalid}(\mathbf x,\mathbf y,m) \label{eq:chard:invalid:detect} \\
    \begin{array}{@{}r@{}}
        \mathit{shiftedTiling}(\mathbf x,\mathbf y,m) \land \mathit{succ}(\mathbf x,\mathbf x') \land \mathit{invalid}(\mathbf x',\mathbf y,m') \\
        \mathit{I}(\ell) \land (\ell < M) \land (m \doteq M \times m' + \ell) \\
    \end{array}                                                                                                                                         &
    \begin{array}{@{\;}lr@{}}
        \land \\
        \to \mathit{invalid}(\mathbf x,\mathbf y,m) \\
    \end{array}                                                                                                                                                                                         \label{eq:chard:invalid:propH} \\
    \begin{array}{@{}r@{}}
        \mathit{shiftedTiling}(\mathbf \One,\mathbf y,m) \land \mathit{succ}(\mathbf y,\mathbf y') \land \mathit{invalid}(\mathbf \Zero,\mathbf y',m') \\
        \mathit{I}(\ell) \land (\ell < M)\land (m \doteq M \times m' + \ell) \\
    \end{array}                                                                                                                                         &
    \begin{array}{@{\;}lr@{}}
        \land \\
        \to \mathit{invalid}(\mathbf \One,\mathbf y,m) \\
    \end{array}                                                                                                                                                                                         \label{eq:chard:invalid:propV}
\end{align}

Finally, program $\Prog_{\mathcal{T}}$ contains rules
\eqref{eq:chard:tiling:0}--\eqref{eq:chard:noSol} where $\mathit{tiling}$ is a
unary $\tmax$ predicate and $\mathit{noSolution}$ is a nullary predicate. Rule
\eqref{eq:chard:tiling:0} ensures that tiling encoded by $0$ is checked. Based
on our discussion from the previous paragraph, for each invalid tiling $n$ such
that ${\Prog_{\mathcal{T}} \models \mathit{tiling}(n)}$, we have
${\Prog_{\mathcal{T}} \models \mathit{invalid}(\mathbf \Zero,\mathbf \Zero,
s_{\mathbf \Zero,\mathbf \Zero,n})}$; moreover, ${s_{\mathbf \Zero,\mathbf
\Zero,n} = n}$ so, if ${\Prog_{\mathcal{T}} \models \mathit{invalid}(\mathbf
\Zero,\mathbf \Zero, n)}$ holds, then rule \eqref{eq:chard:tiling:inc} ensures
that tiling encoded by $n+1$ is considered; atom ${(m \doteq n)}$ is needed in
the rule since no numeric variable is allowed to occur in more than one
standard body atom. If we exhaust all available tilings, rule
\eqref{eq:chard:noSol} determines that no solution exists, just as in the proof
of Lemma~\ref{lem:adm-fact-ent-conp}.
\begin{align}
                                                                                                & \to \mathit{tiling}(0)    \label{eq:chard:tiling:0} \\
    \mathit{invalid}(\mathbf \Zero,\mathbf \Zero,m) \land \mathit{tiling}(n) \land (m \doteq n) & \to \mathit{tiling}(n+1)  \label{eq:chard:tiling:inc} \\
    \mathit{tiling}(n) \land \mathit{maxTiling}(m) \land (m < n)                                & \to \mathit{noSolution}   \label{eq:chard:noSol}
\end{align}

Based on our discussion of the consequences of $\Prog_{\mathcal{T}}$, we
conclude that instance $\mathcal{T}$ of the succinct tiling problem does not
have a solution if and only if ${\Prog_{\mathcal{T}} \models
\mathit{noSolution}}$.
\end{proof}

\begin{proposition}\label{prop:int-lim-isomorphic}
    For $J$ a pseudo-interpretation and $\alpha$ a fact, ${J \models \alpha}$
    if and only if ${\set{\alpha} \sqsubseteq J}$.
\end{proposition}

\begin{proof}
Consider an arbitrary pseudo-interpretation $J$ and the corresponding
limit-closed interpretation $I$. If ${J \models \alpha}$, then ${\alpha \in
I}$, so there exists a fact ${\alpha' \in J}$ such that such that
${\set{\alpha} \sqsubseteq \set{\alpha'}}$, which implies ${\set{\alpha}
\sqsubseteq J}$. Moreover, if ${\set{\alpha} \sqsubseteq J}$, then there exists
a fact ${\alpha' \in J}$ such that ${\set{\alpha} \sqsubseteq \set{\alpha'}}$
and, since $I$ is limit-closed, we have ${\alpha \in I}$, which implies ${J
\models \alpha}$.
\end{proof}

\admissiblefactentailmentconpconexptime*

\begin{proof}
Lemmas~\ref{lem:adm-fact-ent-conp} and~\ref{lem:adm-fact-ent-conexptime} prove
hardness. Moreover, the following nondeterministic algorithm decides
${\Prog\cup\Dat \not\models \alpha}$ in time polynomial in ${\ssize{\Dat} +
\ssize{\alpha}}$, and exponential in ${\ssize{\Prog} + \ssize{\Dat} +
\ssize{\alpha}}$.
\begin{enumerate}
    \item \emph{Compute the semi-grounding $\Prog'$ of $\Prog$}.

    \item \label{item:complexity-pf-step2} \emph{Guess a pseudo-interpretation
    $J$ over the signature of $\Prog'\cup\Dat$ such that the number of facts in
    $J$ and the absolute values of all integers in $J$ are bounded as in
    Theorem~\ref{th:pseudo}.}

    \item \emph{Check that $J$ is a pseudo-model of $\Prog\cup\Dat$; if not,
    return $\false$.}

    \item \emph{Return $\false$ if ${J \models \alpha}$ and $\true$ otherwise.}
\end{enumerate}
Correctness of the algorithm follows from Theorem~\ref{th:pseudo}, so we next
argue about its complexity. The mentioned data complexity holds by the
following observations.
\begin{itemize}
    \item In step~1, $\ssize{\Prog'}$, $|\Prog'|$, and the time required to
    compute $\Prog'$ are all polynomial in $\ssize{\Dat}$ and constant in
    $\ssize{\alpha}$.

    \item Since $|\Prog'|$ is polynomial in $\ssize{\Dat}$ and constant in
    $\ssize{\alpha}$, and ${\max_{r \in \Prog'} \ssize{r}}$ is constant in
    $\ssize{\Dat}$ and $\ssize{\alpha}$, the magnitude of the integers in $J$
    is exponentially bounded in ${\ssize{\alpha} + \ssize{\Dat}}$ by
    Theorem~\ref{th:pseudo}; thus, the number of bits needed to represent each
    integer in $J$ is polynomial in ${\ssize{\alpha} + \ssize{\Dat}}$.
    Furthermore, we have ${|J| \leq |\Dat|+|\Prog'|}$, and ${|\Dat|+|\Prog'|}$
    is polynomial in $\ssize{\Dat}$ and constant in $\ssize{\alpha}$; thus, $J$
    can be guessed in step~2 in nondeterministic polynomial time in
    ${\ssize{\alpha} + \ssize{\Dat}}$.

    \item By Lemma~\ref{lem:imm-cons-op-props}, checking that $J$ is a
    pseudo-model of $\Prog\cup\Dat$ amounts to checking ${\ILFPStep{\Prog' \cup
    \Dat}{J} = J}$. By Lemma~\ref{lem:ilfp-step-polynomial}, $\ILFPStep{\Prog'
    \cup \Dat}{J}$ can be computed in deterministic polynomial time in
    ${\ssize{J} + \ssize{\Dat}}$ and hence in $\ssize{\Dat}$ (as $\ssize{J}$ is
    polynomial in $\ssize{\Dat}$). Hence, step~3 requires deterministic
    polynomial time in ${\ssize{\alpha} + \ssize{\Dat}}$.

    \item By Proposition~\ref{prop:int-lim-isomorphic}, step~4 amounts to
    checking ${\set{\alpha} \sqsubseteq J}$, which can be done in time
    polynomial in $\ssize{J}$ and $\ssize{\alpha}$, and hence polynomial in
    ${\ssize{\Dat} + \ssize{\alpha}}$ as well.
\end{itemize}
Finally, the mentioned combined complexity holds by the following observations.
\begin{itemize}
    \item In step~1, $\ssize{\Prog'}$, $|\Prog'|$, and time required to compute
    $\Prog'$ are all exponential in ${\ssize{\Prog} + \ssize{\Dat}}$ and
    constant in $\ssize{\alpha}$.

    \item Since $|\Prog'|$ is exponential in ${\ssize{\Prog} + \ssize{\Dat}}$
    and constant in $\ssize{\alpha}$, and ${\max_{r \in \Prog'} \ssize{r}}$ is
    linear in $\ssize{\Prog}$ and constant in ${\ssize{\Dat} +
    \ssize{\alpha}}$, the magnitude of the integers in $J$ is doubly
    exponentially bounded in ${\ssize{\Prog} + \ssize{\Dat} + \ssize{\alpha}}$
    by Theorem~\ref{th:pseudo}; thus, the number of bits needed to represent
    each integer in $J$ is exponential in ${\ssize{\Prog} + \ssize{\Dat} +
    \ssize{\alpha}}$. Furthermore, we have ${|J| < |\Dat| + |\Prog'|}$, and
    ${|\Dat| + |\Prog'|}$ is exponential in ${\ssize{\Prog} + \ssize{\Dat}}$
    and constant in $\ssize{\alpha}$; thus, $J$ can be guessed in step~2 in
    nondeterministic exponential time in ${\ssize{\Prog} + \ssize{\Dat} +
    \ssize{\alpha}}$.

    \item By Lemma~\ref{lem:imm-cons-op-props}, checking that $J$ is a
    pseudo-model of $\Prog\cup\Dat$ amounts to checking ${\ILFPStep{\Prog' \cup
    \Dat}{J} = J}$. By Lemma~\ref{lem:ilfp-step-polynomial}, polynomial $p$
    exists such that $\ILFPStep{\Prog' \cup \Dat}{J}$ can be computed in
    deterministic polynomial time in ${\ssize{\Prog'} + \ssize{\Dat} +
    \ssize{J}^{p(\max_{r \in \Prog'} \ssize{r})}}$, which, in turn, is bounded
    by ${O(2^{\ssize{\Prog} + \ssize{\Dat}}) + \ssize{\Dat} +
    O(2^{(\ssize{\Prog}+\ssize{\Dat}) p(\max_{r\in\Prog}\ssize r)})}$. Hence,
    step~3 requires deterministic exponential time in ${\ssize{\Prog} +
    \ssize{\Dat} + \ssize{\alpha}}$.

    \item By Proposition~\ref{prop:int-lim-isomorphic}, step~4 amounts to
    checking ${\set{\alpha} \sqsubseteq J}$, which can be done in time
    polynomial in $\ssize{J}$ and $\ssize{\alpha}$, and hence in time
    exponential in ${\ssize{\Prog} + \ssize{\Dat} + \ssize{\alpha}}$. \qedhere
\end{itemize}
\end{proof}

\section{Proofs for Section~\ref{sec:tractability}}\label{sec:proofs-stable}

For arbitrary value propagation graph ${G^J_\Prog = (\pgNodes,\pgEdges,\mu)}$,
a \emph{path} in $G^J_\Prog$ is a nonempty sequence ${\pi = \pgNodeRaw{1},
\dots, \pgNodeRaw{n}}$ of nodes from $\pgNodes$ such that
${\pgEdgeRaw{v_i}{v_{i+1}} \in \pgEdges}$ holds for each ${0 \leq i < n}$; such
$\pi$ \emph{starts} in $\pgNodeRaw{1}$ and \emph{ends} in $\pgNodeRaw{n}$. We
define ${|\pi| = n}$; moreover, by a slight abuse of notation, we sometimes
write ${\pi \cap X = \emptyset}$ or ${\pgNodeRaw{i} \in \pi}$, where we
identify $\pi$ with the set of its nodes. A path $\pi$ is \emph{simple} if all
of its nodes are pair-wise distinct. A path $\pi$ is a \emph{cycle} if
${\pgNodeRaw{n} = \pgNodeRaw{1}}$.

\begin{definition}
    Given a semi-ground, limit linear program $\Prog$, a pseudo-interpretation
    $J$, value propagation graph ${G^J_\Prog = (\pgNodes,\pgEdges,\mu)}$, and a
    path ${\pi = \pgNodeRaw{1}, \dots, \pgNodeRaw{n}}$ in $G^J_\Prog$, the
    \emph{weight} $\mu(\pi)$ of $\pi$ is defined as
    \begin{displaymath}
        \mu(\pi) = \sum_{i=1}^{n-1} \mu(\pgEdgeRaw{\pgNodeRaw{i}}{\pgNodeRaw{i+1}}).
    \end{displaymath}
\end{definition}

\begin{lemma}\label{lem:cycle-weight-propagates}
    Let $\Prog$ be a semi-ground and stable limit-linear program, let $J$ be a
    pseudo-model of\/ $\Prog$, let ${G^J_\Prog = (\pgNodes,\pgEdges,\mu)}$, and
    let ${\pgNode{A}{\mathbf a},\pgNode{B}{\mathbf b} \in \pgNodes}$ be nodes
    such that $\pgNode{A}{\mathbf b}$ is reachable from $\pgNode{B}{\mathbf a}$
    by a path $\pi$. Then, for each ${k \in \mathbb{Z}}$ such that ${J \models
    B(\mathbf b,k)}$,
    \begin{itemize}
        \item ${J \models A(\mathbf a,k+\mu(\pi))}$ if $A$ and $B$ are both
        $\tmax$ predicates;

        \item ${J \models A(\mathbf a,-k-\mu(\pi))}$ if $A$ is a
        $\tmin$ predicate and $B$ is a $\tmax$ predicate;

        \item ${J \models A(\mathbf a,k-\mu(\pi))}$ if $A$ and $B$ are
        both $\tmin$ predicates; and

        \item ${J \models A(\mathbf a,-k+\mu(\pi))}$ if $A$ is a
        $\tmax$ predicate and $B$ is a $\tmin$ predicate.
    \end{itemize}
\end{lemma}

\begin{proof}
We consider the case when $A$ and $B$ are both $\tmax$ predicates; the
remaining cases are analogous. We proceed by induction on the length of $\pi$.
The base case ($\pi$ is empty) is immediate. For the inductive step, assume
that ${\pi = \pi',\pgNode{A}{\mathbf a}}$ where $\pi'$ is a path starting at
$\pgNode{B}{\mathbf b}$ and ending in node $\pgNode{C}{\mathbf c}$. Then, there
exists an edge ${e = \pgEdge{C}{\mathbf c}{A}{\mathbf a}
\in \pgEdges}$, and $e$ is produced by a rule ${r = C(\mathbf c,n) \land
\varphi \to A(\mathbf a,s) \in \Prog}$ such that $n$ is a variable occurring in
$s$, and $\sigma$ is a grounding of $r$ such that
\begin{enumerate}
    \item ${J \models(C(\mathbf c,n) \land \varphi)\sigma}$, and

    \item ${\delta_r^e(J) = \mu(e) = \mu(\pgEdge{C}{\mathbf c}{A}{\mathbf a})}$.
\end{enumerate}
We next consider the case when $C$ is a $\tmax$ predicate; the case when $C$ is
a $\tmin$ predicate is analogous. Let $\ell$ be such that ${C(\mathbf c,\ell)
\in J}$. We have the following possibilities.
\begin{itemize}
    \item If $\mathsf{opt}(r,J)$ and $\ell$ are both integers (i.e., they are
    not $\infty$), we have ${\mu(e) = \mathsf{opt}(r,J) - \ell}$.
    
    \item If ${\mathsf{opt}(r,J) = \infty}$, then ${\mu(e) = \infty}$ by
    Definition~\ref{def:vp-graph}.
    
    \item If ${\ell = \infty}$, then ${\mu(e) = \infty}$ by
    Definition~\ref{def:stable-programs} and the fact that $\Prog$ is stable,
    and moreover ${\mathsf{opt}(r,J) = \infty}$ by
    Definition~\ref{def:vp-graph}.
\end{itemize}
Now for an arbitrary ${k \in \mathbb{Z}}$ such that ${J \models B(\mathbf
b,k)}$, we consider the following two cases.
\begin{itemize}
    \item ${\mu(e) \neq \infty}$. The inductive hypothesis holds for $\pi'$, so
    ${J \models C(\mathbf c,k+\mu(\pi'))}$ and thus ${k+\mu(\pi') \leq \ell}$.
    Consequently, we have ${\mu(\pi) = \mu(\pi') +\mu(e) = \mu(\pi') +
    \mathsf{opt}(r,J) - \ell \leq \mathsf{opt}(r,J)-k}$, and so ${k + \mu(\pi)
    \leq \mathsf{opt}(r,J)}$ holds. Moreover, ${J \models \Prog}$ implies
    ${\ILFPStep{\Prog}{J} = J}$ by Lemma~\ref{lem:imm-cons-op-props}; thus,
    Proposition~\ref{prop:int-lim-isomorphic} and the definition of
    $\ILFPStepOp{\Prog}{}$ imply ${J \models A(\mathbf a,k + \mu(\pi))}$.

    \item ${\mu(e) = \infty}$. Clearly, ${\mu(\pi) = \infty}$. Moreover, ${J
    \models \Prog}$ implies ${\ILFPStep{\Prog}{J} = J}$ by
    Lemma~\ref{lem:imm-cons-op-props}; thus, ${\mathsf{opt}(r,J) = \infty}$,
    Proposition~\ref{prop:int-lim-isomorphic}, and the definition of
    $\ILFPStepOp{\Prog}{}$ imply ${J \models A(\mathbf a,\infty)}$. \qedhere
\end{itemize}
\end{proof}

\soundness*

\begin{proof}
Let ${G^J_\Prog = (\pgNodes,\pgEdges,\mu)}$, let ${J' =
\ILFPStepOp{\Prog}{\infty}}$, and let ${G^{J'}_\Prog =
(\pgNodes',\pgEdges',\mu')}$. Now assume for the sake of a contradiction that
there exist a cycle $\pi$ in $G_{\Prog}^J$ and a node ${\pgNode{A}{\mathbf a}
\in \pi}$ such that ${\mu(\pi) > 0}$ and ${J' \not\models A(\mathbf
a,\infty)}$. Rule applicability is monotonic w.r.t.\ $\sqsubseteq$, so $\pi$ is
still a cycle in $G_\Prog^{J'}$, and, since $\Prog$ is stable, we have
${\mu'(\pi) \geq \mu(\pi) > 0}$. We consider the case when $A$ is a $\tmax$
predicate; the remaining case is analogous. Now ${\pgNode{A}{\mathbf a} \in
\pgNodes \subseteq \pgNodes'}$ implies that ${A(\mathbf a,k) \in J'}$ for some
$k$; moreover, ${J' \not\models A(\mathbf a,\infty)}$ implies ${k \neq
\infty}$. But then, Lemma~\ref{lem:cycle-weight-propagates} implies ${J'
\models A(\mathbf a,k + \mu'(\pi))}$; moreover, ${\mu'(\pi) > 0}$ implies that
${k + \mu'(\pi)}$ is either $\infty$ or it is an integer larger than $k$;
either way, this contradicts our assumption that ${A(\mathbf a,k) \in J'}$.
\end{proof}

\termination*

\begin{proof}
For $J$ a pseudo-interpretation, $A$ an $(n+1)$-ary limit predicate, and
${\mathbf a}$ an $n$-tuple of objects such such that ${A(\mathbf a,\ell) \in
J}$, let ${\Val{J}{A}{\mathbf a} = \ell}$ if ${\ell = \infty}$ or $A$ is a
$\tmax$ predicate and ${\ell \in \mathbb{Z}}$, and ${\Val{J}{A}{\mathbf a} =
-\ell}$ if $A$ is a $\tmin$ predicate and ${\ell \in \mathbb{Z}}$; moreover,
let $\ActRules{J}{A}{\mathbf a}$ be the set containing each rule ${r \in
\Prog}$ that is applicable to $J$ and where $\head{r}$ is of the form
$A(\mathbf a,s)$. By monotonicity of $\DLog$, we have ${\ActRules{J}{A}{\mathbf
a} \subseteq \ActRules{J'}{A}{\mathbf a}}$ for each $J$ and $J'$ such that ${J
\sqsubseteq J'}$. Moreover, for each edge ${e = \pgEdge{B}{\mathbf
b}{A}{\mathbf a} \in \pgEdges}$ generated by a rule ${r \in
\ActRules{J}{A}{\mathbf a}}$, Definition~\ref{def:vp-graph} ensures that the
following property holds:
\begin{align}
    \Val{J}{B}{\mathbf b} + \mu(e) \geq \Val{\ILFPStep{\set{r}}{J}}{A}{\mathbf a} \tag{$\ast$}
\end{align}

To prove this lemma, we first show the following auxiliary claim.

\medskip

\noindent\textbf{Claim ($\diamondsuit$).\ }
{\em
    For each ${n \ge 0}$ determining the pseudo-interpretation ${J =
    \ILFPStepOp{\Prog}{n}}$ and the value propagation graph ${G^J_\Prog =
    (\pgNodes,\pgEdges,\mu)}$, each ${n' \geq 1}$ determining the
    pseudo-interpretation ${J' = \ILFPStepOp{\Prog}{n+n'}}$ and the value
    propagation graph ${G^{J'}_\Prog = (\pgNodes',\pgEdges',\mu')}$, each set of
    nodes ${X \subseteq \pgNodes}$, and each node ${\pgNode{A}{\mathbf a} \in
    \pgNodes}$ of such that
    \begin{enumerate}
        \item\label{item:diamond1}
        ${\pgEdges = \pgEdges'}$;
    
        \item\label{item:diamond2}
        ${\Val{J'}{B}{\mathbf b} = \infty}$ holds for each node
        ${\pgNode{B}{\mathbf b} \in \pgNodes'}$ that occurs in $G^{J'}_\Prog$
        in a positive-weight cycle;
    
        \item\label{item:diamond3}
        $\pgNode{A}{\mathbf a} \not\in X$;

        \item\label{item:diamond4}
        $\Val{J'}{A}{\mathbf a} < \Val{\ILFPStep{\Prog}{J'}}{A}{\mathbf a}$;
    
        \item\label{item:diamond5}
        ${|\pi| \leq n'}$ holds for each simple path $\pi$ in ${G_{\Prog}^J}$
        that ends in $\pgNode{A}{\mathbf a}$ and satisfies ${\pi \cap X =
        \emptyset}$;

        \item\label{item:diamond6}
        for each node ${\pgNode{B}{\mathbf b} \in X}$, there exists a path
        $\pi$ in $G_{\Prog}^J$ that starts in $\pgNode{A}{\mathbf a}$ and ends
        in $\pgNode{B}{\mathbf b}$;
    \end{enumerate}
    one of the following holds:
    \begin{enumerate}[(i)]
        \item\label{item:diamondC1}
        ${\Val{J'}{C}{\mathbf c} + \mu'(\pi) \ge
        \Val{\ILFPStep{\Prog}{J'}}{A}{\mathbf a}}$ for some node
        ${\pgNode{C}{\mathbf c} \in X}$ and path $\pi$ in $G_\Prog^J$ starting
        in $\pgNode{C}{\mathbf c}$ and ending in $\pgNode{A}{\mathbf a}$;

        \item\label{item:diamondC2}
        ${\ActRules{J}{C}{\mathbf c} \subsetneq \ActRules{J'}{C}{\mathbf c}}$
        for some node ${\pgNode{C}{\mathbf c} \in \pgNodes}$.
    \end{enumerate}
}

\begin{proof}
For arbitrary $n$, we prove the claim by induction on $n'$. For the base case
${n' = 1}$, consider an arbitrary set ${X \subseteq \pgNodes}$ and vertex
${\pgNode{A}{\mathbf a} \in \pgNodes}$ that satisfy
properties~\ref{item:diamond1}--\ref{item:diamond6} of $(\diamondsuit)$. We
distinguish two cases.
\begin{itemize}
    \item There exists an edge ${e = \pgEdge{B}{\mathbf b}{A}{\mathbf a} \in
    \pgEdges}$ such that $\Val{J'}{B}{\mathbf
    b}+\mu'(e)=\Val{\ILFPStep{\Prog}{J'}}{A}{\mathbf a}$. Now either
    ${\pgNode{B}{\mathbf b} \in X}$ or ${\pgNode{B}{\mathbf b} =
    \pgNode{A}{\mathbf a}}$ holds: if that were not the case, path ${\pi =
    \pgNode{B}{\mathbf a},\pgNode{A}{\mathbf a}}$ would be a simple path in
    $G^J_\Prog$ such that ${|\pi| = 2}$, which would contradict
    property~\ref{item:diamond5}.

    We next show that ${\pgNode{B}{\mathbf b} = \pgNode{A}{\mathbf a}}$ is
    impossible. For the sake of a contradiction, assume that
    ${\pgNode{B}{\mathbf b} = \pgNode{A}{\mathbf a}}$ holds, and thus we have
    ${\Val{J'}{A}{\mathbf a} + \mu'(e) = \Val{\ILFPStep{\Prog}{J'}}{A}{\mathbf
    a}}$. By property~\ref{item:diamond4}, this implies ${\Val{J'}{A}{\mathbf
    a} + \mu'(e) > \Val{J'}{A}{\mathbf a}}$, and hence ${\mu'(e) > 0}$.
    Consequently, path $\pi$ is a positive-weight cycle in $G_\Prog^{J'}$, and
    so, by property~\ref{item:diamond2}, we have ${\Val{J}{A}{\mathbf a} =
    \infty}$, which, in turn, contradicts property~\ref{item:diamond4}.

    Consequently, we have ${\pgNode{B}{\mathbf b} \in X}$. But then, since, by
    assumption, ${\Val{J'}{B}{\mathbf b} + \mu'(e) =
    \Val{\ILFPStep{\Prog}{J'}}{A}{\mathbf a}}$, part~\eqref{item:diamondC1} of
    the claim holds for ${\pgNode{C}{\mathbf c} = \pgNode{B}{\mathbf b}}$.

    \item For each edge ${\pgEdge{B}{\mathbf b}{A}{\mathbf a} \in \pgEdges}$ we
    have ${\Val{J'}{B}{\mathbf b} + \mu'(e) <
    \Val{\ILFPStep{\Prog}{J'}}{A}{\mathbf a}}$. Then, for each rule ${r \in
    \Prog}$ that generates an edge ${e = \pgEdge{B}{\mathbf b}{A}{\mathbf a}}$,
    property~($\ast$) ensures ${\Val{\ILFPStep{\set r}{J'}}{A}{\mathbf a} <
    \Val{\ILFPStep\Prog{J'}}{A}{\mathbf a}}$. Since
    \begin{displaymath}
        \Val{\ILFPStep\Prog{J'}}{A}{\mathbf a} = \max_{r \in \Prog,\head{r} = A(\mathbf a,s)} \Val{\ILFPStep{\set{r}}{J'}}{A}{\mathbf a},
    \end{displaymath}
    a rule ${r \in \Prog}$ exists that satisfies
    ${\Val{\ILFPStep\Prog{J'}}{A}{\mathbf a} =
    \Val{\ILFPStep{\set{r}}{J'}}{A}{\mathbf a}}$ but does not generate an edge
    in $\pgEdges$ ending in $\pgNode{A}{\mathbf a}$. Clearly, $\head{r}$ is of
    the form $A(\mathbf a,s)$ and $r$ is applicable to $J'$, so ${r \in
    \ActRules{J'}{A}{\mathbf a}}$ holds. Moreover, $r$ is semi-ground; hence,
    if $s$ were to contain a variable, this variable would occur in a limit
    body atom of $r$, and so $r$ would generate an edge in $\pgEdges$;
    consequently, $s$ is ground. Finally, if $r$ were applicable to $J$, then
    ${\set{ A(\mathbf a,s) } \sqsubseteq J'}$ and so ${\Val{J'}{A}{\mathbf a}
    \ge \Val{\set{ A(\mathbf a,s) }}{A}{\mathbf a} =
    \Val{\ILFPStep\Prog{J'}}{A}{\mathbf a}}$, which contradicts
    property~\ref{item:diamond4}. Consequently, we have ${r \notin
    \ActRules{J}{A}{\mathbf a}}$, and so part~\eqref{item:diamondC2} of the
    claim holds for ${\pgNode{C}{\mathbf c} = \pgNode{A}{\mathbf a}}$.
\end{itemize}

\medskip

For the inductive step, we assume that ($\diamondsuit$) holds for ${n'-1 \geq
1}$, each set ${X \subseteq \pgNodes}$, and each node ${\pgNode{A}{\mathbf a}
\in \pgNodes}$; and we consider an arbitrary set ${X \subseteq \pgNodes}$ and
vertex ${\pgNode{A}{\mathbf a} \in \pgNodes}$ that satisfy
properties~\ref{item:diamond1}--\ref{item:diamond6} of $(\diamondsuit)$. By
property~\ref{item:diamond4}, there exists a rule ${r \in \Prog}$ such that
${\Val{J'}{A}{\mathbf a} < \Val{\ILFPStep{\Prog}{J'}}{A}{\mathbf a} =
\Val{\ILFPStep{\set r}{J'}}{A}{\mathbf a}}$. Now if $r$ does not generate an
edge in $\pgEdges$, then in exactly the same way as in the base case we
conclude that part~\eqref{item:diamondC2} of claim ($\diamondsuit$) holds for
${\pgNode{C}{\mathbf c} = \pgNode{A}{\mathbf a}}$; consequently, in the rest of
this proof we assume that $r$ generates at least one edge in $\pgEdges$. Let
${J''= \ILFPStepOp{\Prog}{n+n'-1}}$ and let ${G_\Prog^{J''} =
(\pgNodes'',\pgEdges'',\mu'')}$. Then, ${\pgEdges = \pgEdges'' = \pgEdges'}$ by
property~\ref{item:diamond1}, and ${\Val{\ILFPStep{\set r}{J''}}{A}{\mathbf a}
\le \Val{J'}{A}{\mathbf a} < \Val{\ILFPStep{\set r}{J'}}{A}{\mathbf a}}$, so
there exists an edge ${e = \pgEdge{B}{\mathbf b}{A}{\mathbf a} \in \pgEdges}$
such that ${\Val{J''}{B}{\mathbf b} < \Val{J'}{B}{\mathbf b}}$. Furthermore,
since $\Val{J'}{A}{\mathbf a}<\Val{\ILFPStep{\set r}{J'}}{A}{\mathbf a}$, if
$\pgNode{B}{\mathbf a}$ were equal to $\pgNode{A}{\mathbf a}$, then path
$\pgNode{A}{\mathbf a},\pgNode{A}{\mathbf a}$ would be a positive-weight cycle
containing $\pgNode{A}{\mathbf a}$, which contradicts
property~\ref{item:diamond2}; hence, we have ${\pgNode{B}{\mathbf b} \neq
\pgNode{A}{\mathbf a}}$ and so path $\pgNode{B}{\mathbf b},\pgNode{A}{\mathbf
a}$ is simple. Now if ${\pgNode{B}{\mathbf b} \in X}$ holds, then, since $r$
generates $e$ and and ${\Val{\ILFPStep{\set r}{J'}}{A}{\mathbf a} =
\Val{\ILFPStep{\Prog}{J'}}{A}{\mathbf a}}$, by property~($\ast$), we have
${\Val{J'}{B}{\mathbf b} + \mu'(e) \ge \Val{\ILFPStep{\Prog}{J'}}{A}{\mathbf
a}}$---that is, part~\eqref{item:diamondC1} of the claim holds for
$\pgNode{C}{\mathbf c}=\pgNode{B}{\mathbf b}$. Therefore, in the rest of this
proof we assume ${\pgNode{B}{\mathbf b} \notin X}$. We now distinguish two
cases.
\begin{itemize}
    \item $\pgNode{B}{\mathbf b}$ is reachable from $\pgNode{A}{\mathbf a}$ in
    $G_\Prog^J$. We next show that the set $X\cup\set{\pgNode{A}{\mathbf a}}$
    and node $\pgNode{B}{\mathbf b}$ satisfy properties~\ref{item:diamond5}
    and~\ref{item:diamond6} of the inductive hypothesis for $n'-1$.

    For property~\ref{item:diamond5}, note that, since $\pgNode B{\mathbf b}$
    is the direct predecessor of $\pgNode A{\mathbf a}$ in $G_\Prog^J$, each
    simple path $\pi$ in $G_\Prog^J$ that ends in $\pgNode{B}{\mathbf b}$ and
    does not involve $\pgNode{A}{\mathbf a}$ can be extended to the simple path
    $\pi,\pgNode{A}{\mathbf a}$ that ends in $\pgNode{A}{\mathbf a}$. Thus, we
    have
    \begin{displaymath}
    \begin{array}{@{}l@{}}
        \max \qset{|\pi|}{\pi \text{ is a simple path in } G_\Prog^J \text{ ending in } \pgNode{B}{\mathbf b} \text{ and } \pi \cap (X\cup\set{\pgNode A{\mathbf a}}) = \emptyset} <{}\\
        \hspace{6cm} \max \qset{|\pi|}{\pi \text{ is a simple path in } G_\Prog^J \text{ ending in } \pgNode{A}{\mathbf a} \text{ and } \pi \cap X = \emptyset}. \\
    \end{array}
    \end{displaymath}
    Property~\ref{item:diamond5} for $X$, $\pgNode{A}{\mathbf a}$, and $n'$
    ensures
    \begin{displaymath}
        \max \qset{|\pi|}{\pi \text{ is a simple path in } G_\Prog^J \text{ ending in } \pgNode{A}{\mathbf a} \text{ and } \pi \cap X = \emptyset} \le n',
    \end{displaymath}
    which in turn implies
    \begin{displaymath}
        \max \qset{|\pi|}{\pi \text{ is a simple path in } G_\Prog^J \text{ ending in } \pgNode{B}{\mathbf b} \text{ and } \pi \cap (X\cup\set{\pgNode A{\mathbf a}}) = \emptyset} \le n'-1.
    \end{displaymath}

    Property~\ref{item:diamond6} holds for $X$ and $\pgNode{A}{\mathbf a}$;
    moreover, there exists a path from $\pgNode{B}{\mathbf b}$ to
    $\pgNode{A}{\mathbf a}$ via the edge $e$, so the property also holds for
    the set ${X \cup \set{\pgNode{A}{\mathbf a}}}$ and node $\pgNode{B}{\mathbf
    b}$.
    
    Property~\ref{item:diamond3} ($\pgNode{B}{\mathbf b}\notin
    X\cup\set{\pgNode{A}{\mathbf a}}$) and property~\ref{item:diamond4}
    ($\Val{J''}{B}{\mathbf b} < \Val{J'}{B}{\mathbf b}$) have already been
    established for $X\cup\set{\pgNode{A}{\mathbf a}}$, $\pgNode{B}{\mathbf
    b}$, and $n'-1$; moreover, properties~\ref{item:diamond1}
    and~\ref{item:diamond2} do not depend on $X$, $\pgNode{A}{\mathbf a}$, and
    $n'$. Thus, we can apply the inductive hypothesis and conclude that one of
    the following holds:
    \begin{enumerate}[(i)]
        \item\label{item:diamondInd1}
        ${\Val{J''}{C}{\mathbf c} + \mu''(\pi) \ge \Val{J'}{B}{\mathbf b}}$
        holds for some node ${\pgNode{C}{\mathbf c} \in X \cup
        \set{\pgNode{A}{\mathbf a}}}$ and path $\pi$ in $G_\Prog^J$ that starts
        in $\pgNode{C}{\mathbf c}$ and ends in $\pgNode{B}{\mathbf b}$;

        \item\label{item:diamondInd2}
        ${\ActRules{J}{C}{\mathbf c} \subsetneq \ActRules{J''}{C}{\mathbf c}}$
        holds for some node ${\pgNode{C}{\mathbf c} \in \pgNodes}$.
    \end{enumerate}
    If~\eqref{item:diamondInd2} is true, then case~\eqref{item:diamondC2} of
    claim ($\diamondsuit$) holds since ${\ActRules{J''}{C}{\mathbf c} \subseteq
    \ActRules{J'}{C}{\mathbf c}}$. Thus, we next assume that
    case~\eqref{item:diamondInd1} holds, and we show that then
    part~\eqref{item:diamondC1} of claim ($\diamondsuit$) holds for
    $\pgNode{C}{\mathbf c}$, $X$, and $\pgNode{A}{\mathbf a}$.

    We first show that ${\pgNode{C}{\mathbf c} \ne \pgNode{A}{\mathbf a}}$. For
    contradiction, assume ${\pgNode{C}{\mathbf c} =
    \pgNode{A}{\mathbf a}}$. Then ${\Val{J''}{A}{\mathbf a} + \mu''(\pi) \ge
    \Val{J'}{B}{\mathbf b}}$. Moreover, since $r$ generates $e$, by
    property~($\ast$) and property~\ref{item:diamond4}, we have
    \begin{displaymath}
        \Val{J'}{B}{\mathbf b} + \mu'(e) \ge \Val{\ILFPStep{\set r}{J'}}{A}{\mathbf a} = \Val{\ILFPStep{\Prog}{J'}}{A}{\mathbf a} > \Val{J'}{A}{\mathbf a}.
    \end{displaymath}
    Consequently, ${\Val{J''}{A}{\mathbf a} + \mu''(\pi) +\mu'(e) >
    \Val{J'}{A}{\mathbf a}}$. Moreover, ${\Val{J'}{A}{\mathbf a} \ge
    \Val{J''}{A}{\mathbf a}}$ holds since $\ILFPStepOp\Prog{}$ is monotonic,
    and ${\mu'(\pi) > \mu''(\pi)}$ holds since $\Prog$ is stable. By these observations, we have
    \begin{displaymath}
        \Val{J'}{A}{\mathbf a} + \mu'(\pi) +\mu'(e) > \Val{J'}{A}{\mathbf a};
    \end{displaymath}
    that is, ${\mu'(\pi) + \mu'(e) > 0}$. But then $\pi,\pgNode{A}{\mathbf a}$
    is a positive-weight cycle in $G_\Prog^{J'}$, and so we have
    ${\Val{J'}{A}{\mathbf a} = \infty}$, which contradicts
    property~\ref{item:diamond4}.

    Thus, we have ${\pgNode{C}{\mathbf c} \in X}$. Then, from
    \begin{align*}
        \Val{J''}{C}{\mathbf c} + \mu''(\pi) \ge \Val{J'}{B}{\mathbf b}                                                             & \text{ and} \\
        \Val{J'}{B}{\mathbf b} + \mu'(e) \ge \Val{\ILFPStep{\set{r}}{J'}}{A}{\mathbf a} = \Val{\ILFPStep{\Prog}{J'}}{A}{\mathbf a}  & \text{ we conclude} \\
        \Val{J'}{C}{\mathbf c} + \mu'(\pi) + \mu'(e) \ge \Val{\ILFPStep\Prog{J'}}{A}{\mathbf a}
    \end{align*}
    as in the case for ${\pgNode{C}{\mathbf c} = \pgNode{A}{\mathbf a}}$. Since
    ${\mu'(\pi) + \mu'(e) = \mu'(\pi,\pgNode{A}{\mathbf a})}$,
    part~\eqref{item:diamondC1} of claim ($\diamondsuit$) holds for
    $\pgNode{C}{\mathbf c}$, $X$, and $\pgNode{A}{\mathbf a}$.

    \item $\pgNode{B}{\mathbf b}$ is not reachable from $\pgNode{A}{\mathbf a}$
    in $G_\Prog^J$. Then, by property~\ref{item:diamond6}, $\pgNode{B}{\mathbf
    b}$ is not reachable in $G_\Prog^J$ from any node in $X$; otherwise,
    $\pgNode{B}{\mathbf b}$ would also be reachable in $G_\Prog^J$ from
    $\pgNode{A}{\mathbf a}$ via some node in $X$. Thus, no simple path in
    $G_\Prog^J$ ending in $\pgNode{B}{\mathbf b}$ involves $\pgNode{A}{\mathbf
    a}$ or a node in $X$---that is, each such path can be extended to a
    simple path ending in $\pgNode{A}{\mathbf a}$. Now
    property~\ref{item:diamond5} ensures
    \begin{align*}
         \max\qset{|\pi|}{\pi \text{ is a path in } G_\Prog^J \text{ ending in } \pgNode{A}{\mathbf a} \text{ and }\pi \cap X=\emptyset} \le n', & \text{ which implies} \\
         \max\qset{|\pi|}{\pi \text{ is a path in } G_\Prog^J \text{ ending in } \pgNode{B}{\mathbf b}} \le n'-1.
    \end{align*}
    Thus, property~\ref{item:diamond5} of the inductive hypothesis for $n'-1$
    holds for the set $\emptyset$ and node $\pgNode{B}{\mathbf b}$. Moreover,
    property~\ref{item:diamond6} holds vacuously for $\emptyset$,
    properties~\ref{item:diamond3} and~\ref{item:diamond4} have already been
    established for $\pgNode{B}{\mathbf b}$, and properties~\ref{item:diamond1}
    and~\ref{item:diamond2} hold by assumption. Thus, we can apply the
    inductive hypothesis for $n'-1$ to $\emptyset$ and $\pgNode{B}{\mathbf b}$,
    and so one of the following holds:
    \begin{enumerate}[(i)]
        \item\label{item:diamondInd3}
        ${\Val{J''}{C}{\mathbf c} + \mu''(\pi) \ge \Val{J'}{B}{\mathbf b}}$ for
        some node ${\pgNode{C}{\mathbf c} \in \emptyset}$ and path $\pi$ in
        $G_\Prog^J$ that starts in $\pgNode{C}{\mathbf c}$ and ends in
        $\pgNode{B}{\mathbf b}$;

        \item\label{item:diamondInd4}
        ${\ActRules{J}{C}{\mathbf c} \subsetneq \ActRules{J''}{C}{\mathbf c}}$
        for some node $\pgNode{C}{\mathbf c}$ in $G_{\Prog}^J$.
    \end{enumerate}
    Clearly,~\eqref{item:diamondInd3} is trivially false,
    so~\eqref{item:diamondInd4} holds. But then, case~\eqref{item:diamondC2} of
    claim ($\diamondsuit$) holds since ${\ActRules{J''}{C}{\mathbf c} \subseteq
    \ActRules{J'}{C}{\mathbf c}}$. \qedhere
\end{itemize}
\end{proof}

Note that, for each ${n \ge 0}$ and each simple path $\pi$ in
$G_\Prog^{\ILFPStepOp{\Prog}{n}}$, $|\pi|$ is bounded by the number of nodes in
$G_\Prog^{\ILFPStepOp{\Prog}{n}}$, which is in turn bounded by ${m = |\Prog|}$.
Therefore, claim~($\diamondsuit$) for ${n'=m}$ and ${X = \emptyset}$ ensures
that, for each ${n \ge 0}$ such that ${\ILFPStepOp{\Prog}{n+m} \ne
\ILFPStepOp{\Prog}{n+m+1}}$, one of the following holds:
\begin{enumerate}
    \item ${\Val{\ILFPStepOp{\Prog}{n+m}}{C}{\mathbf c} \ne \infty}$ for some
    node $\pgNode{C}{\mathbf c}$ that occurs in
    $G_\Prog^{\ILFPStepOp{\Prog}{n+m}}$ in a positive weight cycle (so the
    value of the fact corresponding to $\pgNode{C}{\mathbf c}$ is set to
    $\infty$ in the next iteration of the main loop of the algorithm),

    \item $G_\Prog^{\ILFPStepOp{\Prog}{n+m}}$ contains at least one edge that
    does not occur in $G_\Prog^{\ILFPStepOp{\Prog}{n}}$, or

    \item ${\ActRules{\ILFPStepOp{\Prog}{n}}{C}{\mathbf c} \subsetneq
    \ActRules{\ILFPStepOp{\Prog}{n+m}}{C}{\mathbf c}}$ for some node
    $\pgNode{C}{\mathbf c}$ in $G_\Prog^{\ILFPStepOp{\Prog}{n}}$.
\end{enumerate}
For each ${n \ge 0}$, the size of the set
${\ActRules{\ILFPStepOp{\Prog}{n}}{C}{\mathbf c}}$ for each node
$\pgNode{C}{\mathbf c}$ and the number of nodes in
$G_\Prog^{\ILFPStepOp{\Prog}{n}}$ are both bounded by $m$, and the number of
edges in $G_\Prog^{\ILFPStepOp{\Prog}{n}}$ is bounded by $m^2$. Thus, the
number of iterations of the main loop is bounded by ${m \cdot(m+1) \cdot(m^2+1)
\cdot(m^2+1) \leq 8 m^6}$, where the first factor is given by
Claim~($\diamondsuit$), the second factor comes from the first case above, the
third factor comes from second case, and the fourth factor comes from the third
case . Hence, Algorithm~\ref{alg:lim-stab-fp} reaches a fixpoint after at most
$8 m^6$ iterations of the main loop.
\end{proof}

\algorithmcorrect*

\begin{proof}
Partial correctness follows by Lemma~\ref{lem:soundness}, while termination
follows by Lemma~\ref{lem:termination}. Moreover, the number of iterations of
the main loop of Algorithm~\ref{alg:lim-stab-fp} is polynomially bounded in
$|\Prog\cup\Dat|$, and hence $\ssize{J}$ in each such iteration is bounded by
$\ssize{\Prog\cup\Dat}$. Consequently, lines~\ref{step:compute-imm-cons}
and~\ref{step:compute-dep-graph} of Algorithm~\ref{alg:lim-stab-fp} require
time that is worst-case exponential in ${\max_{r \in \Prog} \ssize{r}}$ and
polynomial in $\ssize{\Prog \cup \Dat}$ by
Lemma~\ref{lem:ilfp-step-polynomial}. Moreover, lines~\ref{step:copy-j},
\ref{step:infty-propagation}, and the check in line~\ref{step:loop1-end}
require time polynomial in $\ssize{J}$ and hence in $\ssize{\Prog\cup\Dat}$.
Finally, we argue that the check for positive-weight cycles in
line~\ref{step:loop2-begin2} is feasible in time polynomial in
$\ssize{\Prog\cup\Dat}$. Let $G'$ be the graph obtained from $G_\Prog^J$ by
negating all weights. Then, a maximal-weight path from $\pgNodeRaw{v_1}$ to
$\pgNodeRaw{v_2}$ in $G^J_\Prog$ corresponds to the least-weight path from
$\pgNodeRaw{v_1}$ to $\pgNodeRaw{v_2}$ in $G'$. Thus, detecting whether a node
occurs in $G^J_\Prog$ in at least one positive-weight cycle reduces to
detecting whether the node occurs in $G'$ on a negative cycle (i.e., on a cycle
with a negative sum of weights), which can be solved in polynomial time using,
for example, a variant of the Floyd-Warshall
algorithm~\cite{DBLP:journals/ipl/Hougardy10}.
\end{proof}

\stablecomplexity*

\begin{proof}
The \textsc{ExpTime} lower bound in combined complexity and the \textsc{PTime}
lower bound in data complexity are inherited from plain
Datalog~\cite{DBLP:journals/csur/DantsinEGV01}. The \textsc{PTime} upper bound
in data is immediate by Theorem~\ref{thm:correctness}. For the \textsc{ExpTime}
upper bound in combined complexity, note that, for $\Prog'=\Prog'_0\cup\Dat$
the semi-grounding of $\Prog=\Prog_0\cup\Dat$ over constants in $\Prog$, we
have that $\ssize{\Prog'}$ is exponentially bounded in $\ssize{\Prog}$, whereas
${\max_{r \in \Prog'_0} \ssize{r} = \max_{r\in\Prog_0} \ssize{r}}$. Hence, by
Theorem~\ref{thm:correctness}, running Algorithm~\ref{alg:lim-stab-fp} on
$\Prog'$ gives us an exponential-time decision procedure for ${\Prog \models
\alpha}$.
\end{proof}

\stabilityundecidable*

\begin{proof}
We present a reduction from Hilbert's tenth problem, which is to determine
whether, given a polynomial ${P(x_1, \dots, x_n)}$ over variables ${x_1, \dots,
x_n}$, equation ${P(x_1, \dots, x_n) = 0}$ has integer solutions. For each such
polynomial $P$, we can assume without loss of generality that $P$ is of the
form $\sum_j(c_j \times \prod_{i=1}^n x_i^{k_{j,i}})$ for ${c_j \in
\mathbb{Z}}$ and ${k_{j,i} \geq 0}$. Now let $\Prog_P$ be the program
containing the following rule, where $B$ is a unary $\tmax$ predicate, and
${A_1, \dots, A_n}$ are distinct unary ordinary numeric predicates:
\begin{align}
    A_1(x_1) \land \dots \land A_n(x_n) \land (P(x_1,\dots,x_n) \leq 0) \land (P(x_1,\dots,x_n) \geq 0) \land B(m) \land (m \leq 0) \to B(m+1)  \label{eq:undec:stable}
\end{align}
Note that rule \eqref{eq:undec:stable} is limit-linear since variables ${x_1,
\dots, x_n}$ do not occurs in a limit atom in the rule. We show that $\Prog_P$
is stable if and only if ${P(x_1, \dots, x_n) = 0}$ has no integer solutions.

Assume that ${P(x_1, \dots, x_n) = 0}$ has no integer solutions. Then, for each
grounding $\sigma$ of \eqref{eq:undec:stable}, at least one of the first two
comparison atoms in the rule is not satisfied, and so $\Prog_P$ is trivially
stable since, for each pseudo-interpretation $J$, the value propagation graph
$G_\Prog^J$ does not contain any edges.

Assume that substitution $\sigma$ exists such that ${P(\sigma(x_1), \dots,
\sigma(x_n)) = 0}$ holds, and let $J_1$ and $J_2$ be the following
pseudo-interpretations with the corresponding value propagation graphs:
\begin{align}
    J_1 & = \set{ A_1(\sigma(x_1)), \dots, A_1(\sigma(x_n))} \cup \set{ B(0) }  & G_{\Prog_P}^{J_1} & = (\pgNodes_1,\pgEdges_1,\mu_1) \\
    J_2 & = \set{ A_1(\sigma(x_1)), \dots, A_1(\sigma(x_n))} \cup \set{ B(1) }  & G_{\Prog_P}^{J_2} & = (\pgNodes_2,\pgEdges_2,\mu_2)
\end{align}
Then, we clearly have ${J_1 \sqsubseteq J_2}$ and ${e = \pgEdge{B}{}{B}{} \in
\pgEdges_1 = \pgEdges_2}$; however, ${\mu_1(e) = 1}$ and ${\mu_2(e) = 0}$.
Consequently, program $\Prog_P$ is not stable.
\end{proof}

\begin{lemma}\label{lem:infinity-propagates}
    For each pseudo-interpretation $J$ and each semi-ground type-consistent
    rule $r$, if $r$ is applicable to $J$ and it contains a limit body atom
    ${B(\mathbf b,n) \in \sbody{r}}$ such that ${B(\mathbf b,\infty) \in J}$
    and variable $n$ occurs in $\head{r}$, then ${\mathsf{opt}(r,J) = \infty}$.
\end{lemma}

\begin{proof}
Consider an arbitrary pseudo-interpretation $J$ and rule $r$ applicable to $J$
that contains a limit body atom ${B(\mathbf b,n) \in \sbody{r}}$ with
${B(\mathbf b,\infty) \in J}$ and $n$ occurring in $\head{r}$. We consider the
case when ${\head{r} = A(\mathbf a,s)}$ for $A$ a $\tmax$ predicate and
variable $n$ occurs in $s$ with a negative coefficient; the cases when $A$ is a
$\tmin$ predicate and/or $n$ occurs in $s$ with a positive coefficient are
analogous. Then, term $s$ has the form ${\ell \times n + t}$ for some negative
integer $\ell$ and term $t$ not containing $n$. Moreover, $r$ is
type-consistent, so $B$ is a $\tmin$ predicate. Since $r$ is applicable to $J$,
conjunction $\constraints{r}{J}$ has a solution $\sigma$. We next show that
${\mathsf{opt}(r,J) = \infty}$ holds, for which it suffices to argue that, for
each ${k \in \mathbb{Z}}$, conjunction $\constraints{r}{J}$ has a solution
$\sigma'$ such that ${s \sigma' \ge k}$. Let ${k_0 = s\sigma}$ and let
$\sigma'$ be the grounding of $\constraints{r}{J}$ such that ${\sigma'(n) =
\sigma(n) - |k-k_0|}$ and ${\sigma'(m) = \sigma(m)}$ for each variable ${m \ne
n}$. Since ${B(\mathbf b,\infty) \in J}$, we have ${J \models B(\mathbf
b,\sigma'(n))}$. Moreover, $r$ is type-consistent, $\sigma$ satisfies all
comparison atoms in the body of $r$, $B$ is a $\tmin$ predicate, and
${\sigma'(n) \le\sigma(n)}$, so $\sigma'$ also satisfies all comparison atoms
in the body of $r$. Hence, $\sigma'$ is a solution to $\constraints{r}{J}$.
Then, the following calculation implies the claim of this lemma:
\begin{displaymath}
    s\sigma' = \ell \times(\sigma(n) - |k-k_0|) + t\sigma = (\ell \times \sigma(n) + t\sigma) - \ell \times|k-k_0| = k_0 - \ell \times |k-k_0| \ge k_0 + |k-k_0| \ge k. \qedhere
\end{displaymath}
\end{proof}

\begin{lemma}\label{lem:delta-nondecreasing}
    For each type-consistent, limit-linear rule $r$ with ${\head{r} = A(\mathbf
    t,s)}$, each limit body atom ${B(\mathbf t',n) \in \sbody{r}}$ such that
    $n$ occurs in $s$, each semi-grounding $\sigma$ of $r$ such that ${n
    \not\in \dom{\sigma}}$, and all pseudo-interpretations $J_1$ and $J_2$ such
    that ${J_1 \sqsubseteq J_2}$ and $r\sigma$ is applicable to $J_1$, we have
    ${\delta_{r\sigma}^e(J_1) \le \delta_{r\sigma}^e(J_2)}$ where ${e =
    \pgEdge{B}{\mathbf t'\sigma}{A}{\mathbf t\sigma}}$.
\end{lemma}

\begin{proof}
Consider arbitrary $r$, ${B(\mathbf t',n)}$, $\sigma$, $J_1$, $J_2$, and $e$ as
stated in the lemma. We consider the case when $A$ is a $\tmax$ and $B$ is a
$\tmin$ predicate; the remaining cases are analogous. Let $\varphi$ be the body
of $r$ and let $\sigma$ be a semi-grounding of $r$. Moreover, the claim is
trivial if ${\delta_{r\sigma}^e(J_2) = \infty}$, so we next assume
${\delta_{r\sigma}^e(J_2) \ne \infty}$. Due to ${J_1 \sqsubseteq J_2}$, each
solution to $\constraints{r\sigma}{J_1}$ is a solution to
$\constraints{r\sigma}{J_2}$ as well, so therefore ${\delta_{r\sigma}^e(J_1)
\ne \infty}$ holds. By Definition~\ref{def:vp-graph}, we then have
${\mathsf{opt}(r\sigma,J_1) \neq \infty}$ and ${\mathsf{opt}(r\sigma,J_2) \neq
\infty}$. But then, by Lemma~\ref{lem:infinity-propagates}, there exist
${k_1,k_2 \in \mathbb{Z}}$ such that ${B(\mathbf t'\sigma,k_1) \in J_1}$ and
${B(\mathbf t'\sigma,k_2) \in J_2}$; since $B$ is a $\tmin$ predicate, we have
${k_2 \le k_1}$; moreover, by Definition~\ref{def:vp-graph} we have
${\delta_{r\sigma}^e(J_1) = \mathsf{opt}(r,J_1) + k_1}$ and
${\delta_{r\sigma}^e(J_2) = \mathsf{opt}(r,J_2) + k_2}$. Rule $r$ is type-consistent, so variable $n$
occurs negatively in $s$; thus, $s\sigma$ is of the form ${s' \times n + s''}$
where $s'$ is a ground product evaluating to a negative integer and $s''$ does
not mention $n$. Moreover,
${\mathsf{opt}(r\sigma,{J_1}) \ne \infty}$, so there exists a grounding $\sigma_1$
of $r\sigma$ such that ${J_1 \models \varphi\sigma\sigma_1}$ and
${\mathsf{opt}(r\sigma,J_1) = s\sigma\sigma_1 = s' \times k_1 + s''\sigma_1 = \delta_{r\sigma}^e(J_1) - k_1}$. 
Let $\sigma_2$ be the
substitution such that ${\sigma_2(n) = k_2}$ and ${\sigma_2(m) = \sigma_1(m)}$
for ${m \neq n}$. Clearly, $J_2$ satisfies all object and numeric atoms in
$\varphi\sigma\sigma_2$.  Then, we have the following:
\begin{displaymath}
    \delta_{r\sigma}^e(J_2) = \mathsf{opt}(r\sigma,J_2) + k_2 \geq s\sigma\sigma_2 + k_2 = s' \times n\sigma_2 + s''\sigma_2 + k_2 = s' \times k_2 + s''\sigma_1 + k_2
\end{displaymath}
Furthermore, we have already established ${s' \times k_1 + s''\sigma_1 = \delta_{r\sigma}^e(J_1) - k_1}$, which implies the following:
\begin{displaymath}
    \delta_{r\sigma}^e(J_2) \geq s' \times k_2 - s' \times k_1 + \delta_{r\sigma}^e(J_1) - k_1 + k_2 = (s' + 1) \times (k_2 - k_1) + \delta_{r\sigma}^e(J_1)
\end{displaymath}
But then, $s' < 0$ and ${k_2 \leq k_1}$ clearly imply ${\delta_{r\sigma}^e(J_2)
\geq \delta_{r\sigma}^e(J_1)}$, as required.
\end{proof}

\typeconsistentprogramsstable*

\begin{proof}
For $\Prog$ a type-consistent program and $\sigma$ a semi-grounding of $\Prog$,
condition 1 of Definition~\ref{def:stable-programs} follows by
Lemma~\ref{lem:delta-nondecreasing}, and condition 2 of
Definition~\ref{def:stable-programs} follows by
Lemma~\ref{lem:infinity-propagates}.
\end{proof}

\typeconsistentprogramschecking*

\begin{proof}
Let $\Prog$ be a limit-linear program. We can check whether $\Prog$ is
type-consistent by considering each rule ${r \in \Prog}$ independently. Note
that the first type consistency condition is satisfied for every semi-ground
limit-linear rule where all numeric terms are simplified as much as possible;
thus, no semi-grounding of $r$ (with constants from $\Prog$) where all numeric
terms are simplified as much as possible can violate the first condition of
Definition~\ref{def:type-consistent}. Thus, it suffices to check whether a
semi-grounding of $r$ (with constants from $\Prog$) can violate the second or
the third condition. In both cases, it suffices to consider at most one atom
$\alpha$ at a time (a limit head atom ${A(\mathbf a,s)}$ for the second
condition or a comparison atom ${s_1 < s_2}$ or ${s_1 \leq s_2}$ for the third
condition). In $\alpha$, we consider at most one numeric term $s$ at a time
(${s \in \set{s_1,s_2}}$ for the third condition), where $s$ is of the form
${t_0 + \sum_{i=1}^n t_i \times m_i}$ and $t_i$, for ${i \geq 1}$, are terms
constructed from integers, variables not occurring in limit atoms, and
multiplication. Moreover, for each such $s$, we consider each variable $m$
occurring in $s$.

By assumption, $m$ occurs in $s$, so we have ${m_i = m}$ for some $i$. For the
second condition of Definition~\ref{def:type-consistent}, we need to check
that, if the limit body atom ${B(\mathbf s,m_i)}$ introducing $m_i$ has the
same (different) type as the head atom, then term $t_i$ can only be grounded to
positive (negative) integers or zero. For the third condition, we need to check
that, if ${s = s_1}$ and the limit body atom ${B(\mathbf s,m_i)}$ introducing
$m_i$ is $\tmin$ ($\tmax$), then term $t_i$ can only be grounded to positive
(negative) integers or 0, and dually for the case ${s = s_2}$. Hence, in either
case, it suffices to check whether term $t_i$ can be semi-grounded so that it
evaluates to a positive integer, a negative integer, or zero. We next discuss
how this can be checked in logarithmic space. Let ${t_i = t_i^1 \times \dots
\times t_i^k}$, where each $t_i^j$ is an integer or a variable not occurring in
a limit atom, and assume without loss of generality that we want to check
whether $t_i$ can be grounded to a positive integer; this is the case if and
only if one of the following holds:
\begin{itemize}
    \item all $t_i^j$ are integers whose product is positive;

    \item the product of all integers in $t_i$ is positive and $\Prog$ contains
    a positive integer;

    \item the product of all integers in $t_i$ is positive, $\Prog$ contains a
    negative integer, and the total number of variable occurrences in $t_i$ is
    even;

    \item the product of all integers in $t_i$ is negative, $\Prog$ contains a
    negative integer, and the total number of variable occurrences in $t_i$ is
    odd; or

    \item the product of all integers in $t_i$ is negative, $\Prog$ contains
    both positive and negative integers, and some variable $t_i^j$ has an odd
    number of occurrences in $t_i$.
\end{itemize}
Each of these conditions can be verified using a constant number of pointers
into $\Prog$ and binary variables. This clearly requires logarithmic space, and
it implies our claim.
\end{proof}

%%% Local Variables:
%%% mode: latex
%%% TeX-master: "paper"
%%% End:

}{}

\end{document}

%%% Local Variables:
%%% mode: latex
%%% TeX-master: t
%%% End: